\newif\ifpreprint
\preprinttrue
  \documentclass[letterpaper, 10pt, conference]{ieeeconf}
  \IEEEoverridecommandlockouts
\usepackage{graphicx}
\usepackage[dvipsnames]{xcolor}
\usepackage{colortbl}

\definecolor{oursmint}{HTML}{E2F1EA}
\definecolor{losssalmon}{HTML}{F9E8E4}
\usepackage{nicematrix}
\usepackage{booktabs}
\usepackage{tabularx}
\usepackage{rotating}
\makeatletter
\def\endtable{\end@float}
\def\endfigure{\end@float}
\makeatother
\usepackage[caption=false]{subfig}
\usepackage{mathrsfs}
\usepackage[mode=buildmissing]{standalone}
\usepackage{amsfonts, amssymb, amsmath}
\usepackage{soul}

\makeatletter
\let\NAT@parse\undefined
\makeatother
\usepackage[colorlinks=true, linkcolor=Blue, citecolor=Blue, urlcolor=Blue]{hyperref}
\usepackage{nicefrac}
\usepackage{multirow}
\usepackage{cleveref}
\usepackage{tcolorbox}

\crefname{algocf}{algorithm}{algorithms}
\Crefname{algocf}{Algorithm}{Algorithms}
\crefname{section}{Sec.}{Secs.}

\usepackage{tikz}
\usetikzlibrary{arrows.meta,calc,positioning}
\usepackage{contour}

\usepackage[ruled,vlined]{algorithm2e}

  \makeatletter
  \def\@begintheorem#1#2{\topsep 0pt\rmfamily\trivlist\itemindent\z@%
      \item[]\leavevmode{\bfseries #1\ #2:}\ \ignorespaces}
  \def\@opargbegintheorem#1#2#3{\topsep 0pt\rmfamily\trivlist\itemindent\z@%
      \item[]\leavevmode{\bfseries #1\ #2}\ (#3):\ \ignorespaces}
  \makeatother
\newtheorem{definition}{Definition}
\newtheorem{theorem}{Theorem}
\newtheorem{problem}{Problem}
\newtheorem{remark}{Remark}
\newtheorem{proposition}{Proposition}
\newtheorem{lemma}{Lemma}
\newtheorem{corollary}{Corollary}

\newcommand{\ith}{i^{\mathrm{th}}}
\newcommand{\Qb}{Q_{\beta}}
\newcommand{\eps}{\varepsilon}
\newcommand{\epsb}{\varepsilon^{(\beta)}}
\newcommand{\CVaR}{\mathrm{CVaR}_{\beta}}
\newcommand{\VaR}{\mathrm{VaR}_{\beta}}

  \title{\LARGE {\scshape Firm}Grasp: A \ul{\bf F}riction-\ul{\bf I}nformed \ul{\bf R}isk \ul{\bf M}argin for Robust \ul{\bf Grasp} Synthesis}

\ifpreprint
\author{Clinton Enwerem\(^{1}\), John S. Baras\(^{1}\), and Calin Belta\(^{1}\)%
\thanks{\(^{1}\)The authors are with the Department of Electrical \& Computer Engineering and Institute for Systems Research, University of Maryland, College Park, MD, USA. Emails:
{\ttfamily\small \char`\{enwerem, baras, calin\char`\}@umd.edu}.}%
}
\else
\author{Anonymous Authors [Submission No.~TBD]}
\fi

\begin{document}
\maketitle
\thispagestyle{empty}
\pagestyle{empty}

\begin{abstract}
	Classical grasp quality metrics assume one deterministic friction coefficient and therefore cannot assess whether a grasp maintains force closure across plausible friction values. We present FIRMGrasp, a family of grasp quality metrics that incorporates friction uncertainty through Conditional Value-at-Risk (CVaR). At confidence level $\beta$, we evaluate the force-closure margin at the mean of the adverse friction tail. This evaluation defines the risk-adjusted margin $\varepsilon^{(\beta)}$, the inscribed-ball radius of the corresponding grasp wrench space. We prove that $\varepsilon^{(\beta)}$ varies monotonically with $\beta$, remains differentiable in the grasp parameters, and certifies that any grasp with $\varepsilon^{(\beta)} > 0$ achieves force closure with probability at least $\beta$. Across 1,599 LEAP Hand and Allegro Hand grasps, $\varepsilon^{(\beta)}$ identifies friction-sensitive grasps that receive high nominal Ferrari-Canny scores, and 53\% of the nominally force-closed grasps lose closure in the adverse friction tail. The nominal margin ranks a successful grasp above a failed grasp with probabilities of only 0.53 in the shake test and 0.67 in the pick test, whereas $\varepsilon^{(\beta)}$ achieves 0.63 and 0.78. At an adverse friction coefficient of 0.2, 70\% of grasps with positive $\varepsilon^{(\beta)}$ withstand a simulated lift and lateral pull, compared with 25\% of grasps with positive nominal margin and nonpositive $\varepsilon^{(\beta)}$. We also synthesize grasps with positive $\varepsilon^{(\beta)}$ for the RealHand L6 and LEAP Hand, both of which retain the object during adverse-friction lifts. In MuJoCo trials with the RealHand L6, 95\% of grasps that establish contact and have positive $\varepsilon^{(\beta)}$ retain the object at the same adverse friction coefficient.
\end{abstract}

\section{Introduction}\label{sec:intro}
\begin{figure}[t]
    \centering
    \begin{tikzpicture}[
        img/.style={inner sep=0, outer sep=0},
        lbl/.style={font=\scriptsize\sffamily, align=center},
        mlbl/.style={font=\scriptsize\sffamily, align=center, fill=white,
                     fill opacity=0.88, text opacity=1, inner sep=1.5pt,
                     rounded corners=1pt},
        hdr/.style={font=\scriptsize\sffamily, text=gray!35!black, align=center},
        rowlbl/.style={font=\scriptsize\sffamily, text=gray!25!black, align=center,
                       rotate=90, anchor=south},
        arr/.style={-{Stealth[length=6pt,width=6pt]}, line width=1.6pt},
    ]
        \node[img, xshift=-2pt] (r1a) {\includegraphics[scale=0.12]{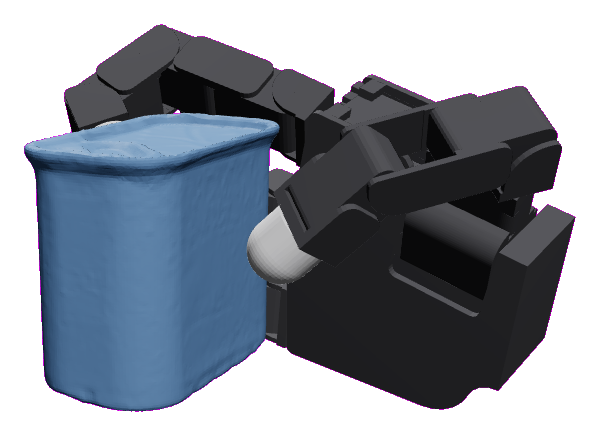}};
        \node[img, right=1.35cm of r1a] (r1b) {\includegraphics[scale=0.12]{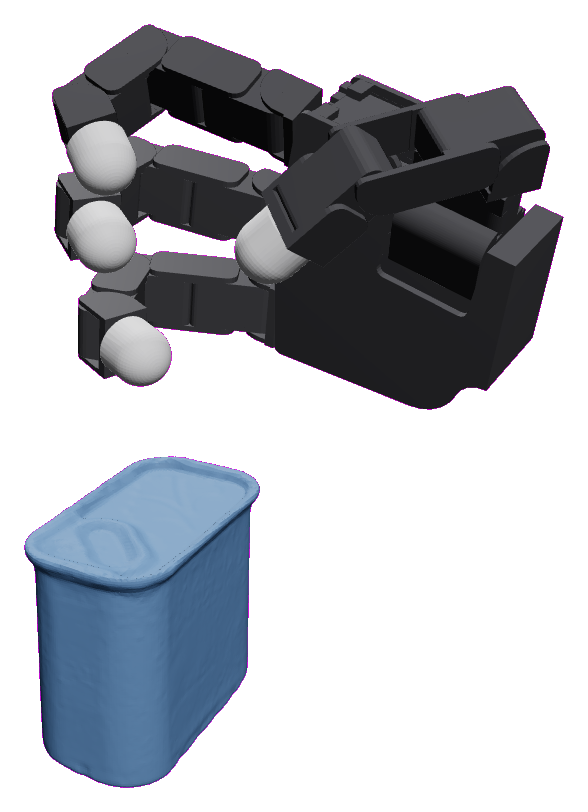}};
        \node[hdr, above=3pt] at (r1a.center |- r1b.north) {grasp at execution};
        \node[hdr, above=3pt] at (r1b.north) {after adverse-friction lift};
        \draw[arr, red!75!black]
            (r1a.east) -- node[above=3pt, mlbl]
                          {$\eps_\text{nom}{=}{+}0.008$\\[1pt]$\epsb{=}{-}0.073$}
                          node[below=3pt, mlbl, text=red!75!black] {drops} (r1b.west);
        \node[rowlbl] at ([xshift=-3pt]r1a.west) {friction-sensitive};
        \node[below=1pt of r1a, lbl] {(a)};
        \node[below=1pt of r1b, xshift=8pt, yshift=24pt, lbl] {(b)};

        \node[img, below=0.8cm of r1a] (r2a) {\includegraphics[scale=0.12]{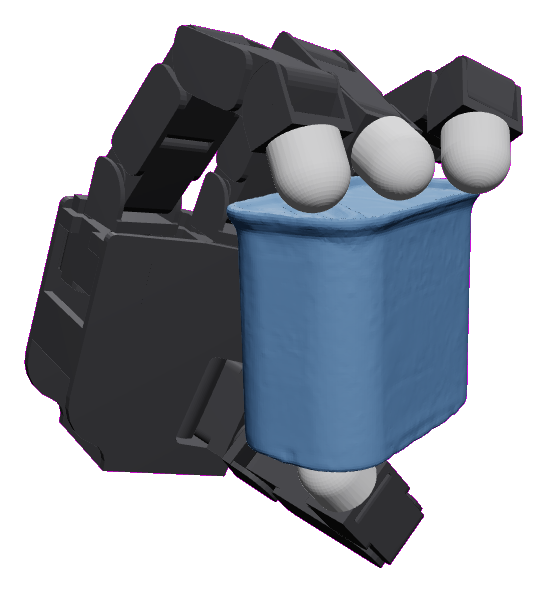}};
        \node[img, right=1.35cm of r2a] (r2b) {\includegraphics[scale=0.105]{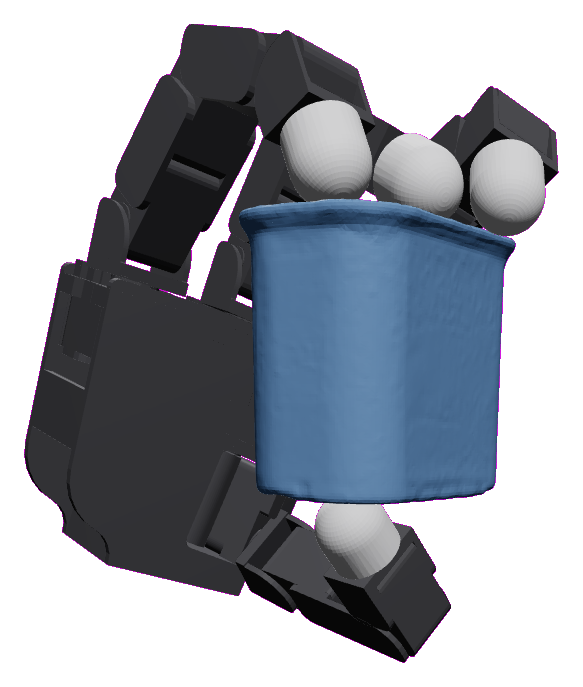}};
        \draw[arr, green!45!black]
            (r2a.east) -- node[above=3pt, mlbl]
                          {$\eps_\text{nom}{=}{+}0.013$\\[1pt]$\epsb{=}{+}0.002$}
                          node[below=3pt, mlbl, text=green!45!black] {holds} (r2b.west);
        \node[rowlbl] at ([xshift=-3pt]r2a.west) {certified-robust};
        \node[below=1pt of r2a, lbl] {(c)};
        \node[below=1pt of r2b, xshift=15pt, lbl] {(d)};
    \end{tikzpicture}
	\vspace{-8pt}
    \caption{\textbf{Predicting Lift Success under Adverse Friction.} The nominal Ferrari-Canny margin assigns nearly identical scores to two Allegro Hand grasps of the same object ($\eps_\text{nom} = {+}0.008$ and ${+}0.013$), but they produce different outcomes under friction uncertainty. At confidence $\beta = 0.9$, $\epsb$ (\Cref{def:rsepsmet}) rejects the first grasp and certifies the second under the adverse prior $0.5\,\mathcal{U}(0.7, 1.0) + 0.5\,\mathcal{U}(0.1, 0.3)$. In a gravity-on lift at $\mu{=}0.3$, the first grasp loses the object, whereas the second completes the lift and retains the object under lateral pull.}
    \label{fig:teaser_pick_contrast}
\end{figure}
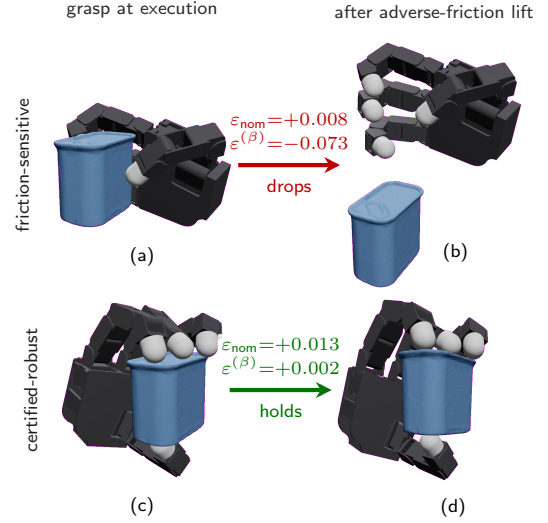
Grasping a rigid object with a rigid-link robotic hand entails choosing contacts, a wrist frame, and hand configurations that stabilize the object against arbitrary exogenous wrenches including gravity, and one may further require that the grasp be dexterous, in equilibrium, or stable. Since many grasps meet these desiderata, \emph{quality measures} rank them so that synthesis reduces to an optimization problem. The most widely used such measure is the Ferrari-Canny, or epsilon, metric, the radius of the largest ball inscribed in the grasp wrench space~\cite{ferrariPlanningOptimalGrasps1992}, alongside a range of alternatives that instead evaluate the grasp matrix, finger force limits, or hand-object geometry~\cite{liSastryTaskOriented1988,salisburyArticulatedHands1982,kimOptimalGrasping2001,rimonMechanicsRobotGrasping2019,li_frogger_2023,liAnalyticTheoryIntrinsic2024a}.

However, these metrics score a grasp at the single instant of synthesis, under a friction coefficient, object pose, and contact configuration all assumed fixed and known, so none of them bounds what happens once execution departs from that instant. A grasp that is force-closed and positive-epsilon at synthesis can destabilize within seconds under lateral exogenous forces and an execution-time friction coefficient distinct from the one assumed at synthesis~\cite{enweremVariationalNeuralBeliefParameterizations2026a}. \Cref{fig:teaser_pick_contrast} previews this failure on a toy grasp-and-lift example, two Allegro Hand grasps that the nominal epsilon margin certifies as force-closed at synthesis time. Only the grasp our risk-sensitive certificate also certifies retains the object upon lift under adverse friction, while the other slips. The pattern is not specific to friction. The same epsilon metric is a poor predictor of robustness to pose error, with its force-closure verdict flipping under pose perturbations the metric does not model~\cite{weiszPoseErrorRobust2012}. We therefore treat grasp quality as the probability that a grasp keeps force closure across the distribution of friction coefficients the physical contact could plausibly realize.

Given this probabilistic setting, and motivated by the demonstrated effectiveness of risk-sensitive formulations built on the Conditional Value-at-Risk (CVaR) and related coherent risk measures for safe decision-making under distributional uncertainty in motion planning~\cite{enwerem2024robust}, reinforcement learning~\cite{enwerem2025safety}, and belief-space control under uncertainty~\cite{enweremRiskConstrainedBeliefSpaceOptimization2026a}, we ground grasp quality in the same CVaR construction by formulating it as a risk-sensitive functional over the friction distribution. Thus, our metric enables a direct quantification of robustness, capturing failure modes that arise only after interaction with the environment, a capability standard quality metrics do not provide.

In practice, evaluating force closure through CVaR at confidence level $\beta$ gives a margin, tunable by $\beta$, that quantifies the likelihood a grasp keeps force closure under the adverse tail of the friction prior instead of at a single nominal coefficient. This margin reduces to the classical epsilon margin under a degenerate, point-mass prior, so it generalizes the classical desiderata instead of replacing them. A grasp with a positive margin under the friction tail also has a positive margin at the nominal coefficient, a monotonicity property we prove in \Cref{thm:monotone}. The resulting metric complements recent grasp-synthesis methods~\cite{li_frogger_2023,wang_dexgraspnet_2023,suh_dexterous_2026}, since any of them can supply candidate grasps for our metric to score and rank by friction sensitivity, a ranking our simulated execution studies show predicts lift success and post-execution stability (\Cref{tab:sim_exec,ssec:simexec,fig:grasp_gallery}).
\\[1em]
\noindent\textbf{Contributions:}
\begin{enumerate}
\item \emph{New Risk-Sensitive Quality Metrics.} We define two new grasp
		quality margins, the CVaR margin $\Qb$ (\Cref{eq:Qbeta}) and its closed-form
		analytic counterpart $\epsb$ (\Cref{def:rsepsmet}). Each margin scores a grasp by its
		friction-tail behavior rather than by a single nominal friction
		coefficient, and $\epsb$ reduces to the classical Ferrari-Canny margin
		under a point-mass friction prior, so both generalize the nominal
		metric rather than restating it. We prove their monotonicity,
		differentiability, and probabilistic-closure properties
		(Theorems~\ref{thm:monotone} to~\ref{thm:certificate},
		Proposition~\ref{prop:ordering}).

\item \emph{Probabilistic Closure Certificate.} We provide a formal guarantee that
		a positive risk-sensitive margin implies force closure with probability at
		least $\beta$, for $\Qb$ through the CVaR tail bound
		(Theorem~\ref{thm:certificate}) and for $\epsb$ through friction
		monotonicity (Corollary~\ref{cor:epsb_certificate}).

\item \emph{Risk-Aware Grasp Synthesis and Evaluation.} We implement a sample-based
		grasp synthesis pipeline that uses the CVaR metric to select robust grasps from a
		candidate pool, and we demonstrate on seven objects that the CVaR metric surfaces a
		friction sensitivity that the classical epsilon metric and the min-weight metric of
		FRoGGeR~\cite{li_frogger_2023} do not model, since neither treats friction as
		uncertain.

\end{enumerate}

\section{Related Work}\label{sec:relwork}
\subsection{Classical Analytical Grasp Metrics}
Analytical grasp quality assessment dates to Ferrari and Canny's epsilon metric~\cite{ferrariPlanningOptimalGrasps1992}, alongside the Grasp Wrench Space volume~\cite{liSastryTaskOriented1988}, the grasp isotropy index~\cite{salisburyArticulatedHands1982,kimOptimalGrasping2001}, and the largest minimum resisted wrench~\cite{rimonMechanicsRobotGrasping2019}, each quantifying a different algebraic property of the same grasp map. Uncertainty-aware treatments of these metrics have, until recently, stayed largely empirical, scoring robustness with Monte Carlo statistics over the nominal metric~\cite{weiszPoseErrorRobust2012}. Such statistics predict grasp success reasonably well, but the sampling is expensive, and both the friction coefficient and the pose distribution stay fixed at their nominal values throughout.

By contrast, a more recent line of work departs from the epsilon metric analytically. The min-weight metric of FRoGGeR~\cite{li_frogger_2023} reports the minimum among a set of scaled friction generators at the linearized friction cones. Min-weight is almost-everywhere differentiable, and its linear-program formulation, differentiability included, drives grasp synthesis in under a second, fast enough for dynamic manipulation~\cite{li_drop_2025,suh_dexterous_2026}. It shares with the epsilon metric it approximates, however, a fixed Coulomb friction coefficient, so a grasp it scores highly is precision-optimal only under that nominal coefficient. Our metric extends the same line of work to a friction coefficient treated as uncertain rather than fixed.

\subsection{Uncertainty Representations for Dexterous Grasping}
Other recent work addresses uncertainty at a different stage of grasp planning or execution. PONG bounds the probability of force closure analytically under uncertain contact normals~\cite{liPONG2023}, Gaussian Process Implicit Surfaces represent the object surface itself as a distribution rather than a fixed mesh~\cite{mahlerGPGPISOPT2015,kumarConstrainingGaussianProcess2024}, and a learned variational belief over contact parameters certifies an already-synthesized grasp under multimodal uncertainty~\cite{enweremVariationalNeuralBeliefParameterizations2026a}. These methods certify an already synthesized grasp against execution-time or post-execution perturbations under an assumed uncertainty model. We instead incorporate uncertainty \emph{before} and \emph{during} synthesis. We retain Coulomb's friction condition but model the friction coefficient as a random variable. Consequently, each contact force lies within a \emph{set} of possible friction cones determined by the coefficient distribution. Our metric complements existing force-closure margins, in particular FRoGGeR's min-weight \(\ell^*\) and the robustness measures of Li \emph{et al.}~\cite{liAnalyticTheoryIntrinsic2024a}, and each member of the resulting metric family answers a different robustness question about a grasp.

\subsection{Robust Grasp Synthesis}
A separate line of work targets wrench capability directly. \cite{liAnalyticTheoryIntrinsic2024a} prove that a force-closed grasp synthesized under a nominal basis-wrench hull withstands any disturbance smaller than the residual between that hull and the hull the perturbed basis wrenches sweep out. Grasp stability tracks wrench resistance closely enough that this residual, which we term the \emph{wrench capability residual}, is a useful quantity in its own right. The residual measures the distance a friction perturbation can move the primitive wrenches before closure fails.

In addition to analytic synthesis, a substantial body of learning-based work targets dexterous grasp generation directly. Differentiable formulations of force closure treat contact selection as a smooth optimization objective, propagating gradients through a grasp planner~\cite{liu2020deep} or through a differentiable simulator that scores candidate grasps by how well they withstand contact-rich perturbation~\cite{turpin2022grasp,turpin2023fast}, a differentiable force-closure estimator built on the classical epsilon metric extends the same treatment to arbitrary hand structures~\cite{liu2021synthesizing}, and a bilevel formulation pairs a conditional generative model with a differentiable force-closure constraint~\cite{wu2023learning}. A separate line of work learns generative models that sample grasp poses conditioned on object geometry or point clouds, spanning variational~\cite{mousavian20196}, adversarial~\cite{corona2020ganhand,lundell2021multi,lundell2021ddgc}, real-time regression~\cite{mayer2022ffhnet}, and diffusion-based~\cite{lu2024ugg,weng2024dexdiffuser,urain2023se} formulations, alongside SE(3)-equivariant dexterous grasp flows~\cite{enwerem_equidexflow_2026}. Earlier discriminative approaches instead score or synthesize grasps directly from depth images or contact maps~\cite{varley2015generating,schmidt2018grasping,brahmbhatt2019contactgrasp}, while hand-agnostic frameworks learn grasping policies that transfer across gripper morphologies~\cite{shao2020unigrasp,li2022efficientgrasp,xu2021adagrasp}. These generators are efficient at producing candidate grasps but do not certify or rank robustness under friction uncertainty, so our metric is agnostic to the generator, scoring grasps from any such source.

\section{Mathematical Preliminaries}\label{sec:bkgd}
\begin{figure}[t]
	\centering
	\begin{tikzpicture}[
			framearrow/.style={->,>=Stealth,line width=1pt},
			xlen/.store in=\xlen, xlen=0.09,
			ylen/.store in=\ylen, ylen=0.13,
			zlen/.store in=\zlen, zlen=0.09
		]
		\node[anchor=south west, inner sep=0] (img) at (0,0) {
			\includegraphics[width=0.65\linewidth,trim=220pt 240pt 380pt 255pt,clip]{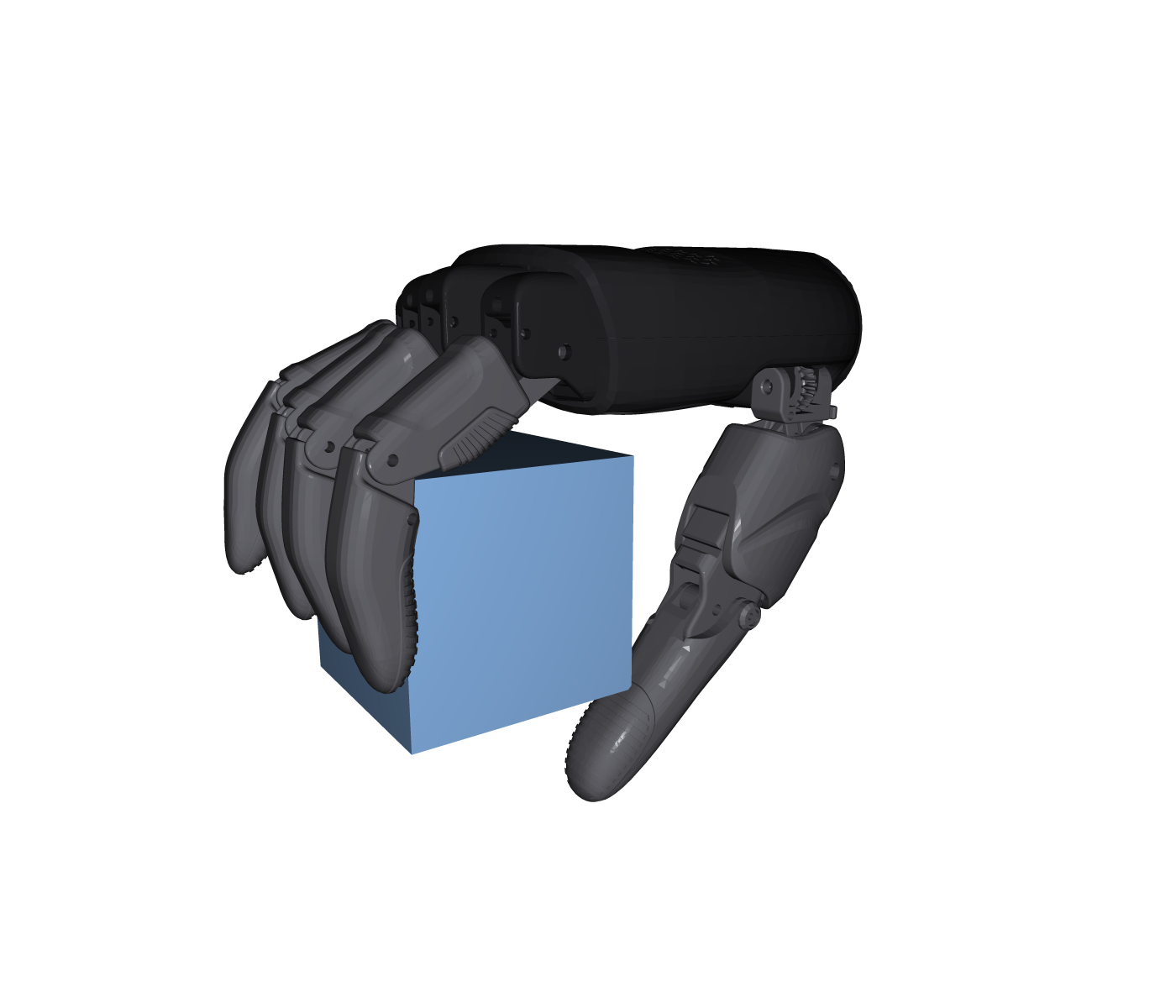}
		};
		\begin{scope}[x={(img.south east)},y={(img.north west)}]
			\coordinate (wrist) at (1.005,0.83);
			\draw[framearrow, red!80!black] (wrist) - ++(0,-\xlen-0.08)
			node[below, xshift=4pt] {\small \(\boldsymbol{x_H}\)};
			\draw[framearrow, green!45!black] (wrist) - ++(\ylen, -0.04)
			node[below right] {\small \(\boldsymbol{y_H}\)};
			\draw[framearrow, blue!75] (wrist) - ++(-\zlen-0.04,0.00)
			node[above left] {\small \(\boldsymbol{z_H}\)};
			\node[above right] at (wrist) {\small \(\{\mathcal{F}_H\}\)};
			\begin{scope}[transform canvas={xshift=-0.4cm, yshift=-0.5cm}]
				\coordinate (obj) at (0.58,0.45);
				\draw[framearrow, red!80!black] (obj) - ++(\zlen,-0.04)
				node[right, xshift=-4pt, yshift=-2pt] {\small \(\boldsymbol{x_O}\)};
				\draw[framearrow, green!45!black] (obj) - ++(0,\ylen)
				node[above] {\small \(\boldsymbol{y_O}\)};
				\draw[framearrow, blue!75] (obj) - ++(-\zlen+0.012,-\ylen+0.09)
				node[below left, xshift=4pt, yshift=5pt] {\small \(\boldsymbol{z_O}\)};
				\node[above right] at (obj) {\small \(\{\mathcal{F}_{O}\}\)};
				\coordinate (joint_base) at (0.325,0.62);
				\coordinate (p1) at (0.41,0.320);
				\draw[framearrow, black] (p1) - ++(0.1,0.0325)
				node[above right, xshift=-3pt, yshift=-8pt] {\small \(n^{(i)}\)};
				\draw[framearrow, black] (p1) - ++(0.1,-0.045);
				\draw[framearrow, black] (p1) - ++(0.00,-0.15)
				node[right, xshift=-2pt, yshift=-2pt] {\small \(\{\mathcal{F}^{C_i}\}\)};
				\node[below left, yshift=-6pt, xshift=-6pt] at (p1) {\small \({p^{(i)}}\)};
				\draw[fill=white, draw=black, line width=0.4pt] (p1) circle (2pt);
			\end{scope}
		\end{scope}
	\end{tikzpicture}
	\caption{\textbf{Grasping Frames.} Contact, object, and hand coordinate frames
		used in our uncertainty-aware grasp planning framework. These include the hand base
		frame, \(\{\mathcal{F}_H\}\), the object frame, \(\{\mathcal{F}_O\}\), and the
		\(\ith\) contact frame, \(\{\mathcal{F}^{C_i}\}\) attached to contact point \(i\)
		at \(p^{(i)} \in \mathbb{R}^3\). Admissible contact forces at a contact point are
		in \(\{\mathcal{F}^{C_i}\}\) and positions
		in \(\{\mathcal{F}_O\}\), and we assume knowledge of the contact-to-object frame
		force rotation matrix (\(R_{OC}\)) and the object-to-hand homogeneous
		transformation (\(T_{OH}\)).}\label{fig:grspframes}
\end{figure}
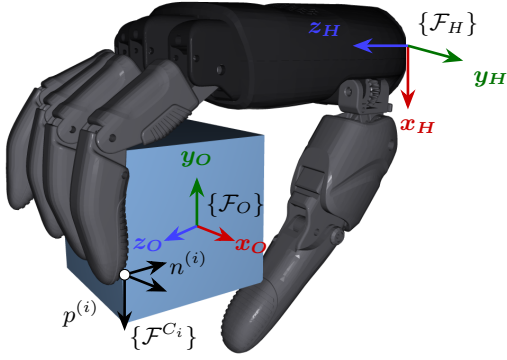

We consider a multi-fingered hand with configuration \(q_h \in \mathbb{R}^{n_h}\)
grasping a rigid object whose pose in the hand frame \(\{\mathcal{F}_{H}\}\) is
\(X^o \in SE(3)\). The hand forward kinematics \(\mathrm{FK}(\cdot)\) generates a
contact set \(\mathcal{C}:=\{(p^{(i)}, \{\mathcal{F}^{C_i}\}) \mid p^{(i)}:=
\mathrm{FK}(q_h, X^o) \in \mathbb{R}^3,\; i = 1, \ldots, n_c,\;
n_c \le N_{\max}\}\), where \(n_c\) is the contact count and \(N_{\max}\) the
fingertip count of the hand. Each contact comprises a point \(p^{(i)}\) and a frame
\(\{\mathcal{F}^{C_i}\} := \{t^{(i)}_x, t^{(i)}_y, n^{(i)}\}\) of
orthonormal tangent directions and the unit surface normal
\(n^{(i)} \in \mathbb{R}^3\).

Concretely, our risk-sensitive margins evaluate the adverse tail of a friction-dependent quantity
through the lower-tail Conditional
Value-at-Risk (CVaR)~\cite{rockafellarOptimizationConditionalValueatrisk2000}. For a scalar
random variable \(Z\) with quantile function \(F_Z^{-1}\) and
confidence \(\beta \in (0,1)\),
\begin{equation}\label{eq:cvardef}
	\CVaR(Z) := \frac{1}{1-\beta}\int_0^{1-\beta} F_Z^{-1}(u)\,du,
\end{equation}
the mean of the worst \((1-\beta)\)-fraction of outcomes, so \(\beta \to 1\) is the
risk-averse limit and \(\VaR(Z) := F_Z^{-1}(1-\beta)\) is the matching lower quantile,
with \(\CVaR(Z) \le \VaR(Z)\).

\subsection{The Geometry of Parametric Grasps}\label{ssec:graspgeom}
We assume hard-finger contacts with friction, with contact forces satisfying Coulomb's
condition~\cite{murrayMathematicalIntroductionRobotic2017a} and lying in a linearized
polyhedral friction cone of a predetermined number of edge rays or
\emph{generators}, where each generator is a primitive contact force in
\(\mathbb{R}^3\).

\begin{definition}[Grasp Parameterization]\label{def:graspparam}
	We denote a grasp by the tuple
	\(g:=(\mathcal{C}, q_h, X^o)\) comprising the contact
	set, hand configuration, and object pose. Let \(\mathcal{G}\)
	denote the set of all such grasps.
\end{definition}

Given this parameterization, we measure the quality of a grasp \(g\)
(\Cref{def:graspparam}) by a signed scalar whose sign reports force closure. Both the classical Ferrari-Canny
margin and the risk-adjusted margin we develop in \Cref{sec:rsgqual} instantiate
the underlying quality map, and every closure certificate below turns on its sign.

\begin{definition}[Grasp Quality Function]\label{def:graspqual}
	A grasp quality function is a scalar map \(Q: \mathcal{G} \rightarrow \mathbb{R}\),
	positive when the grasp attains the encoded property
	(here, force closure) and negative otherwise.
\end{definition}

\begin{definition}[Primitive Contact Forces \& Induced Wrench Cone]\label{def:ctcforce}
	The polyhedral cone approximating admissible contact forces at the \(\ith\) contact
	point is
	\begin{equation}\label{eq:polyfriccone}
		\mathfrak{C}^{(i)} := \Bigl\{F = {\sum^{n_s}_{j=1}\alpha^{(i)}_j f^{(i)}_j}
		\;\Big|\; \alpha^{(i)}_j \ge 0\Bigr\},
	\end{equation}
	where \(j = 1, \ldots, n_s\) indexes the generator set and \(f^{(i)}_{j}\) is
	the primitive force corresponding to contact \(i\) and generator \(j\), expressed in
	the object frame:
	\begin{equation}\label{eq:feq}
		f^{(i)}_{j}=f^{(i)}_{n, j}n^{(i)} + f^{(i)}_{t, j}, \quad
		f^{(i)}_{t, j} \in \operatorname{span}\bigl\{t^{(i)}_{x}, t^{(i)}_{y}\bigr\}.
	\end{equation}
\end{definition}

\noindent The generators satisfy the friction limit specified by Coulomb's condition,
i.e., \(\|f_t\| = \mu f_n\). Thus, from~\eqref{eq:feq}, each generator is an affine
function of the friction coefficient \(\mu\):
\begin{equation}\label{eq:coufcone}
	f^{(i)}_{j} := \phi_{ij}(\mu)
	= f^{(i)}_{n,j} n^{(i)} + \mu f^{(i)}_{n,j}\hat{t}_{ij}.
\end{equation}
The primitive wrench about the object's center-of-mass $o$ is then
\begin{equation}\label{eq:primwrench}
	w_{ij}(\mu)=\left[\begin{array}{c}
			\phi_{ij}(\mu) \\
			(p^{(i)}-o) \times \phi_{ij}(\mu)
		\end{array}\right].
\end{equation}
We assemble the GWS as the convex hull of the primitive wrenches of \Cref{def:ctcforce},
\(\mathcal{W}(\mu) := \operatorname{conv}(\{w_{ij}(\mu)\}_{i,j})\), the
total-force-normalized construction
of~\cite{ferrariPlanningOptimalGrasps1992,liSastryTaskOriented1988}
with a unit normal force per generator (\(f^{(i)}_{n,j} = 1\)) and unscaled
moment coordinates. Because the tangential components of the generators come in
opposing pairs, \(\mathcal{W}(\mu') \subseteq \mathcal{W}(\mu)\) for
\(0 \le \mu' \le \mu\). Every generator at \(\mu'\) is a convex combination of the
corresponding pair at \(\mu\), so the GWS nests in the friction coefficient.
Splitting each generator at \(\mu = 0\) also fixes two auxiliary bodies used
repeatedly below, the \emph{normal wrench body} \(\mathcal{W}_n := \mathcal{W}(0)\)
and the \emph{tangential wrench body}
\(\mathcal{W}_t := \operatorname{conv}(\{(f^{(i)}_{n,j}\hat{t}_{ij},\;
(p^{(i)}-o)\times f^{(i)}_{n,j}\hat{t}_{ij})\}_{i,j})\).

\section{Risk-Adjusted Grasp Wrench Space}\label{sec:radjgws}
The friction coefficient assumed at planning time may deviate from its realized value at execution, so the grasp wrench space is a random polytope rather than a deterministic set. We scale the underlying polytope with friction, discount its adverse tail, and evaluate a force-closure margin on the discounted set. Under friction uncertainty, the polytope scales linearly with the coefficient (Lemma~\ref{lem:linearscaling})\footnote{Under object-pose uncertainty, the GWS polytope rotates instead, and the active uncertainty modality determines which transformation applies.} Let the friction
coefficient be a random scalar
\(\mu \sim p_\mu\). The resulting polyhedral friction cone~\eqref{eq:polyfriccone}
thus becomes a convex set of uncertain generators whose magnitudes depend on the
realization of \(\mu\).

\begin{lemma}[GWS Linear Scaling Under Friction Uncertainty]\label{lem:linearscaling}
	The primitive wrench \(w_{ij}(\mu)\) is affine in \(\mu\). As a result, the GWS
	\(\mathcal{W}(\mu)\) satisfies
	\(\mathcal{W}(\mathbb{E}[\mu]) \subseteq \mathbb{E}[\mathcal{W}(\mu)]\) in the
	Minkowski-expectation sense, and its diameter scales linearly in \(\mu\).
\end{lemma}
\begin{proof}
	From~\eqref{eq:coufcone}, \(\phi_{ij}(\mu) = f^{(i)}_{n,j} n^{(i)} + \mu
	f^{(i)}_{n,j} \hat{t}_{ij}\), where \(\hat{t}_{ij}\) is the unit tangent direction.
	This is affine in \(\mu\), so \(w_{ij}(\mu)\) is affine in \(\mu\) by
	~\eqref{eq:primwrench}. The convex hull of affinely-parameterized vertices is
	Minkowski-affine in \(\mu\), yielding the stated scaling. The diameter
	\(\mathrm{diam}(\mathcal{W}(\mu)) = 2\mu \max_{ij} \|f^{(i)}_{n,j}\| \cdot
	\|[{\hat{t}_{ij}}^\top, (p^{(i)}-o)\times \hat{t}_{ij}]^\top\|\) is linear in
	\(\mu\).
\end{proof}

In particular, for simplicity, we focus on variable friction. Extension to variable
object pose follows analogously. Throughout, \(\theta \in \mathbb{R}^{n_\theta}\)
collects the grasp parameters, i.e., the wrist pose and hand configuration that determine
the grasp \(g\) of \Cref{def:graspparam}, and we write
\(\mathcal{W}(\mu) \equiv \mathcal{W}(\mu;\theta)\), suppressing \(\theta\) where it
remains fixed. Because each primitive wrench is affine in \(\mu\)
(\Cref{lem:linearscaling}), the friction CVaR \(v_\beta\) of \eqref{eq:cvardef}
collapses the underlying friction prior into a single adverse coefficient. The next
definition names this friction value and the wrench body assembled at it.

\begin{definition}[Risk-Adjusted Friction and Wrench Body]\label{def:radjfriction}
	Let \(\mu \sim p_\mu\). The \emph{risk-adjusted friction} at confidence \(\beta\)
	is \(v_\beta := \CVaR(\mu)\), the mean of the adverse friction tail, and the
	\emph{risk-adjusted wrench body} is the GWS assembled at the risk-adjusted friction \(v_\beta\),
	\begin{equation}\label{eq:radjbody}
		\mathcal{W}^{(\beta)}(\theta) := \mathcal{W}(v_\beta;\theta).
	\end{equation}
	Since \(v_\beta \le \mathbb{E}[\mu]\) and the GWS nests in \(\mu\),
	\(\mathcal{W}^{(\beta)}\) shrinks as \(\beta\) rises. \(\mathcal{W}^{(\beta)}\)
	is the wrench capability the grasp retains once the friction coefficient is
	discounted to its adverse tail.
\end{definition}

\begin{definition}[Risk-Sensitive Directional Support]\label{def:riskgraspmets}
	For each direction \(d \in \mathbb{S}^5\), the unit sphere in \(\mathbb{R}^6\),
	the \emph{risk-sensitive directional support} of the uncertain GWS is
	\begin{equation}\label{eq:hbeta}
		h_\beta(d;\theta) := \CVaR\!\left(h_{\mathcal{W}(\mu)}(d)\right),
	\end{equation}
	where \(h_{\mathcal{W}(\mu)}(d) := \sup_{w \in \mathcal{W}(\mu)} d^\top w\) is
	the support function of the uncertain GWS.
\end{definition}

\begin{proposition}[Support Structure of the Risk Adjustment]\label{prop:radjbody}
The map $d \mapsto h_\beta(d;\theta)$ is a valid support function, convex
and positively homogeneous of degree one in $d$, and it is the support function of
the convex body
\begin{equation}
  \bar{\mathcal{W}}^{(\beta)}(\theta)
  :=\bigl\{w\in\mathbb{R}^6 : d^\top w \le h_\beta(d;\theta),\;\forall d\in\mathbb{S}^5\bigr\},
\end{equation}
which contains the risk-adjusted wrench body:
$\mathcal{W}^{(\beta)} \subseteq \bar{\mathcal{W}}^{(\beta)}$, with
$h_{\mathcal{W}^{(\beta)}}(d) = h_\beta(d;\theta)$ in every direction $d$ whose
maximizing primitive wrench does not change over the friction tail.
\end{proposition}
\begin{proof}
Every $h_{\mathcal{W}(\mu)}(d)$ is non-decreasing in $\mu$, since the GWS nests in
the friction coefficient. The tail event realizing the lower-tail CVaR is
therefore $\{\mu \le \VaR(\mu)\}$ for every direction simultaneously, so
$h_\beta(d;\theta) = \mathbb{E}\bigl[h_{\mathcal{W}(\mu)}(d) \mid \mu \le
\VaR(\mu)\bigr]$ is a tail-conditional expectation of support functions and hence
convex and positively homogeneous in $d$. For the inclusion, each vertex support
$d^\top w_{ij}(\mu)$ is affine in $\mu$ by Lemma~\ref{lem:linearscaling}, so
$h_{\mathcal{W}^{(\beta)}}(d) = \max_{ij}\, \mathbb{E}[d^\top w_{ij}(\mu) \mid \mu
\le \VaR(\mu)] \le \mathbb{E}[\max_{ij}\, d^\top w_{ij}(\mu) \mid \mu \le
\VaR(\mu)] = h_\beta(d;\theta)$, with equality when one primitive wrench attains
the maximum throughout the tail.
\end{proof}

Beyond this directional support, assessing force closure in the distributional setting demands a probabilistic definition over the whole friction prior, so we formalize probabilistic (force) closure next as a scalar condition.

\begin{definition}[Probabilistic Closure]\label{def:probclosure}
	A grasp \(g \in \mathcal{G}\) satisfies \emph{probabilistic closure at level
		$\beta$} if \(
		\CVaR(\eps(g,\mu)) > 0,\)
	which implies force closure in all but the worst $(1-\beta)$-fraction of uncertainty
	realizations under the assumed model.
\end{definition}
\begin{remark}[Necessity of a Distribution-Dependent Quality]
	A quality that ignores the friction distribution cannot certify probabilistic
	closure. Take a grasp that is force-closed at the nominal friction $\mu_\text{nom}$ but
	relies on friction, so $\eps(g,\mu_\text{nom}) > 0$ while its friction threshold
	$\mu^* := \inf\{\mu : \eps(g,\mu) > 0\}$ is strictly positive and the grasp loses
	closure, $\eps(g,\mu) \le 0$, for every $\mu < \mu^*$. Fix $\delta \in (0,1)$ and
	$\eta \in (0, 1-\delta)$ and pick any friction value $\mu_- < \mu^*$. The two-point
	distribution $p_\delta$ that places mass $\delta + \eta$ at $\mu_-$ and mass
	$1 - \delta - \eta$ at $\mu_\text{nom}$ has closure-loss probability
	$\Pr_{p_\delta}[\eps(g,\mu) \le 0] = \delta + \eta > \delta$. A quality $Q(g)$ that
	does not depend on $p_\delta$ returns the same value here as under a distribution
	concentrated near $\mu_\text{nom}$ with negligible closure loss, so it cannot separate the
	two and cannot guarantee closure with probability at least $\beta$. The
	risk-sensitive condition of \Cref{def:probclosure} depends on the adverse tail of the
	distribution and avoids the gap described above.
\end{remark}

Given this, the probabilistic closure definition of~\Cref{def:probclosure} yields a force-closure certificate. The following theorem states, formally, the probability with which a grasp is force-closed once its quality margin turns positive.
\begin{theorem}[Probabilistic Closure Certificate]\label{thm:certificate}
	Let $\mu \sim p_\mu$. If $\CVaR(\eps(g,\mu)) > 0$, i.e., if $g$ satisfies
	probabilistic closure at level $\beta$ per \Cref{def:probclosure}, then
	\begin{equation}
		\Pr\bigl[\eps(g,\mu) > 0\bigr] \ge \beta,
	\end{equation}
	i.e., the grasp is force-closed with probability at least $\beta$.
\end{theorem}
\begin{proof}
	By the lower-tail definition of the CVaR risk measure~\cite{rockafellarOptimizationConditionalValueatrisk2000}, $\CVaR(Z) > 0$ implies $\VaR(Z) > 0$, i.e.,
	the $(1-\beta)$-quantile of $Z$ is positive, so $\Pr[Z \le 0] < 1-\beta$ and
	$\Pr[Z > 0] > \beta$. Applying this to $Z = \eps(g,\mu)$ and recalling that
	$\eps(g,\mu) > 0$ certifies force closure gives the result.
\end{proof}

\section{Problem Formulation}\label{prb:form}
Problem~\ref{prob:rsrgs} states the synthesis target our metric family serves.

\begin{problem}[Risk-Sensitive Robust Grasp Synthesis]\label{prob:rsrgs}
Given an object model, a robotic hand model with configuration space \(\mathbb{R}^{n_h}\), a friction prior \(p_\mu\), and a confidence level \(\beta \in (0,1)\), find
\begin{equation}\label{eq:rsrgs}
    g^* \in \arg\max_{g \in \mathcal{G}_{\mathrm{feas}}}\; \CVaR\!\bigl(\eps(g,\mu)\bigr),
    \quad \mu \sim p_\mu,
\end{equation}
where \(\mathcal{G}_{\mathrm{feas}} \subseteq \mathcal{G}\) collects the grasps satisfying joint limits, self-collision avoidance, and contact-formation feasibility for the arm-hand-object system of~\Cref{fig:grspframes}, and \(\eps(g,\mu)\) is the Ferrari-Canny margin of grasp \(g\) under friction realization \(\mu\). By~\Cref{thm:certificate}, any \(g^*\) attaining a positive risk-sensitive margin admits probabilistic closure at level \(\beta\) in the sense of~\Cref{def:probclosure} and therefore secures force closure with probability at least \(\beta\) over the friction realizations of \(p_\mu\).
\end{problem}

Specifically, our objective in~\eqref{eq:rsrgs} couples grasp parameters to the tail of the friction-induced wrench-margin distribution, so the optimizer prefers grasps whose wrench capacity withstands the adverse \((1-\beta)\)-fraction of friction realizations rather than the nominal one. \Cref{sec:rsgqual} develops our metric family that scores candidate grasps, and \Cref{sec:theory} establishes the monotonicity, differentiability, and closure-certificate properties that ground~\eqref{eq:rsrgs}. \Cref{sec:synthesis} then describes the sample, certify, and score procedure that approximates \(g^*\). The studies instantiate two computable quantities, our analytic margin \(\epsb\) for scoring and selection, and, as the synthesis surrogate of \eqref{eq:rsrgs}, a risk-adjusted extension of FRoGGeR's min-weight metric~\cite{li_frogger_2023}, \(\ell^{*(\beta)}\), the min-weight program of \eqref{eq:minweight} evaluated on the risk-adjusted body \(\mathcal{W}^{(\beta)}\) (\Cref{ssec:synth_objective}).

\section{Risk-Sensitive Grasp Quality}\label{sec:rsgqual}
Classical grasp quality metrics such as force closure and $\eps$-metrics evaluate
feasibility at a single instant under nominal assumptions, and provide no guarantee
once the realized contact friction departs from the value assumed at synthesis. We
therefore treat grasp quality as a property of a distribution of friction
realizations rather than a single value. Published friction tables
\cite{CoefficientFrictionReference} report a range of values for a given material
pair, motivating a treatment via intervals or distributions rather than a point
estimate. This motivates the following definition.

\begin{definition}[Risk-Adjusted Margin]\label{def:rsepsmet}
	The risk-adjusted margin at confidence
	\(\beta \in (0,1)\), denoted \(\epsb\), is the inscribed-ball radius of the
	risk-adjusted wrench body of \Cref{def:radjfriction},
	\begin{equation}\label{eq:rseps}
		\epsb(g) := \min_{\|d\|=1} h_{\mathcal{W}^{(\beta)}}(d)
		= \eps(g, v_\beta),
	\end{equation}
	the Ferrari-Canny margin of the GWS assembled at the risk-adjusted friction
	\(v_\beta\) (\Cref{fig:construction}). We define \(\epsb\) as a signed margin, taking the signed
	origin-to-facet distance when the origin leaves the body, so a negative \(\epsb\)
	measures the margin by which the adverse friction tail violates closure. Our metric
	recovers the classical \(\eps\) under a point-mass prior, approaches the
	mean-friction margin \(\eps(g, \mathbb{E}[\mu])\) as \(\beta \to 0\), and the
	worst-case margin \(\eps(g, \operatorname{ess\,inf} \mu)\) as
	\(\beta \to 1\).
\end{definition}

\begin{figure}[t]
	\centering
	\includegraphics[width=\linewidth]{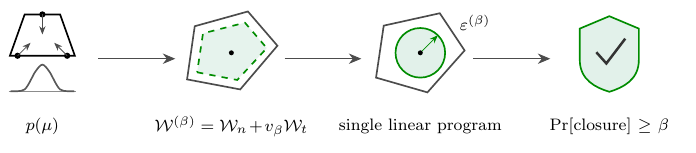}
	\caption{\textbf{Constructing the Risk-Adjusted Margin.} From a contact set and
		friction prior \(p(\mu)\), the risk-adjusted grasp wrench space
		\(\mathcal{W}^{(\beta)} = \mathcal{W}_n + v_\beta\mathcal{W}_t\) scales the
		tangential wrench body by the friction CVaR \(v_\beta\); a single linear
		program inscribes the largest ball of radius \(\epsb\) in
		\(\mathcal{W}^{(\beta)}\) (\Cref{eq:rseps}), whose positivity certifies force
		closure with probability at least \(\beta\) (Theorem~\ref{thm:certificate}).}
	\label{fig:construction}
\end{figure}

\begin{table}[t]
	\caption{Common grasp quality functions and their risk-sensitive adaptations. Adapted from Table~2 in~\cite{roaGraspQualityMeasures2015}. Prob.\ = probabilistic uncertainty model (see \Cref{ssec:friction}). The shaded row is the member the present paper instantiates analytically. The \(\ell^{*(\beta)}\) member drives the synthesis study of \Cref{ssec:synth_objective}.}
	\label{tab:qualfuncs}
	\centering
	\footnotesize
	\setlength{\tabcolsep}{2.5pt}
	\renewcommand{\arraystretch}{1.05}
	\begin{tabular}{@{}l l l l@{}}
		\toprule
		Quality \(Q\)                       & Risk-sens.\ \(\Qb\)   & Unc. variable\ & Unc.\ model \\
		\midrule
		\rowcolor{oursmint}
		\(\eps\) (Ferrari-Canny)           & \(\epsb\)              & \(\mu\)         & Prob.        \\
		Min-weight~\cite{li_frogger_2023}   & \(\ell^{*(\beta)}\)   & \(\mu\)         & Prob.        \\
		\bottomrule
	\end{tabular}
\end{table}

Concretely, we treat grasp robustness as a quasi-static property of the synthesized grasp under
uncertain contact friction. The realized grasp-quality margin is the Ferrari-Canny
metric $\eps(g,\mu)$, positive when the grasp is force-closed and negative
otherwise, and we write $\eps_\text{nom} := \eps(g, \mu_\text{nom})$ for its value
at the nominal friction $\mu_\text{nom}$. Rather than evaluating the realized margin in expectation, we adopt a
risk-sensitive formulation,
\begin{equation}\label{eq:Qbeta}
	\Qb(\theta) = \CVaR\!\left(\eps(g,\mu)\right), \qquad \mu \sim p_\mu,
\end{equation}
the mean of the worst $(1-\beta)$-fraction of margin outcomes under the lower-tail
convention of \eqref{eq:cvardef}. Two risk summaries of the same margin
distribution therefore coexist: $\epsb$ discounts the friction coefficient first
and evaluates the margin once, while $\Qb$ evaluates the margin at every friction
realization and averages its adverse tail. Proposition~\ref{prop:ordering}
relates the two, and each admits the closure certificate of \Cref{sec:theory}.
For $N$ Monte Carlo friction samples
For samples $\{\mu_i\}_{i=1}^N \sim p_\mu$, we compute $\Qb(\theta)$ by sorting
$\{\eps(g,\mu_i)\}$ in ascending order and averaging the worst
$\lceil (1-\beta) N \rceil$ values, which penalizes low-probability, high-impact
failures. The analytic protocol of \Cref{ssec:reeval} avoids the Monte Carlo sampling
entirely by evaluating $\epsb$. A dynamic extension that evaluates
the margin over a finite-horizon execution rollout, rather than at the grasp
instant, is the domain of the execution-time controller and lies outside the
scope of the present study.

\section{Theoretical Properties}\label{sec:theory}
Three properties make $\Qb$ usable in practice. The closure certificate of
Theorem~\ref{thm:certificate}, established in \Cref{sec:radjgws}, already turns
a positive margin into a force-closure guarantee that holds with probability at
least $\beta$. Monotonicity in $\beta$ (Theorem~\ref{thm:monotone}) lets a user
dial conservatism, and differentiability (Theorem~\ref{thm:diff}) exposes a
gradient for optimization.

\begin{theorem}[Monotonicity of $\Qb$ in $\beta$]\label{thm:monotone}
	For fixed grasp parameters $\theta$, $\Qb(\theta)$ is non-increasing in $\beta$.
\end{theorem}
\begin{proof}
	Under the lower-tail convention, $\Qb(\theta) = \mathrm{CVaR}_\beta(Z)$ with
	$Z = \eps(g,\mu)$ admits the quantile-integral representation
	\begin{equation*}
		\mathrm{CVaR}_\beta(Z) = \frac{1}{1-\beta}\int_0^{1-\beta} F_Z^{-1}(u)\,du,
	\end{equation*}
	where $F_Z^{-1}$ is the quantile function of $Z$. The right-hand side is the
	running average of $F_Z^{-1}$ over $[0, 1-\beta]$, and $F_Z^{-1}$ is
	non-decreasing on $[0,1]$. For $\beta' > \beta$, the upper endpoint shrinks from
	$1-\beta$ to $1-\beta'$, so the average moves toward the left end of the
	quantile function, which is smaller. Hence $\mathrm{CVaR}_{\beta'}(Z) \le
	\mathrm{CVaR}_\beta(Z)$.
\end{proof}

\begin{corollary}
	Higher risk aversion ($\beta \nearrow 1$) yields more conservative synthesis
	objectives, at the cost of potentially over-penalizing grasps that are robust under
	typical disturbances.
\end{corollary}

\begin{theorem}[Differentiability of $\Qb$]\label{thm:diff}
	If $\eps(g,\mu)$ is differentiable in $\theta$ for almost every $\mu$, then
	$\Qb(\theta) = \CVaR(\eps(g,\mu))$ is differentiable in $\theta$ almost everywhere,
	with
	\begin{equation}\label{eq:cvargrad}
		\nabla_\theta \Qb(\theta) =
		\mathbb{E}\!\bigl[\nabla_\theta \eps(g,\mu) \;\big|\;
			\eps(g,\mu) \le \VaR(\eps(g,\mu))\bigr].
	\end{equation}
\end{theorem}
\begin{proof}
	The Ferrari-Canny margin $\eps(g,\mu)$ is the value of a linear program and is
	piecewise smooth in $\theta$, differentiable away from a measure-zero set of
	active-vertex changes. For a continuous friction distribution, CVaR is
	differentiable by the envelope (Danskin) theorem applied to the
	Rockafellar-Uryasev infimum
	representation~\cite{rockafellarOptimizationConditionalValueatrisk2000}, with the
	infimizing $v^\star$ equal to
	$\VaR$, which yields~\eqref{eq:cvargrad}. In the sample-average approximation, CVaR
	is piecewise linear in the sorted tail values, with subgradients available via the
	sorting permutation, which is a.e.\ differentiable.
\end{proof}

Taken together, Theorem~\ref{thm:certificate} certifies grasps through $\Qb$, whereas the analytic
protocol of \Cref{ssec:reeval} evaluates $\epsb$. The next two results connect these
margins. The first ties them to the risk-sensitive directional support
of \Cref{def:riskgraspmets}, and the second gives $\epsb$ its own closure
certificate.

\begin{proposition}[Relations Among the Risk-Sensitive Margins]\label{prop:ordering}
For any $\theta$ and $\beta \in (0,1)$,
\begin{equation}\label{eq:ordering}
	\Qb(\theta) \le \min_{\|d\|=1} h_\beta(d;\theta)
	\quad \text{and} \quad
	\epsb(g) \le \min_{\|d\|=1} h_\beta(d;\theta),
\end{equation}
while $\Qb$ and $\epsb$ admit no general ordering between themselves. The minimum
over directions pulls $\Qb$ below $\epsb$, and Jensen's inequality along any
single direction pulls it above.
\end{proposition}
\begin{proof}
For every $d$, $\eps(g,\mu) \le h_{\mathcal{W}(\mu)}(d)$ pointwise in $\mu$;
applying the monotone functional $\CVaR$ and then minimizing over $d$ yields the
first inequality. The second inequality restates
Proposition~\ref{prop:radjbody}. Per direction,
$h_{\mathcal{W}^{(\beta)}}(d) \le h_\beta(d;\theta)$, and minimizing both sides
over $d$ preserves the bound.
\end{proof}

\begin{corollary}[Closure Certificate for $\epsb$]\label{cor:epsb_certificate}
If $\epsb(g) > 0$, then $\Pr[\eps(g,\mu) > 0] \ge \beta$. The guarantee
mirrors Theorem~\ref{thm:certificate} but follows from friction monotonicity
rather than from the CVaR tail bound. The GWS nests in $\mu$, so
$\eps(g,\cdot)$ is non-decreasing, $\epsb > 0$ implies closure for every
$\mu \ge v_\beta$, and $v_\beta \le \VaR(\mu)$ gives
$\Pr[\mu \ge v_\beta] \ge \beta$.
\end{corollary}

\section{Synthesis Algorithm}\label{sec:synthesis}
Theorem~\ref{thm:diff} provides a basis for gradient-based grasp
optimization. In the present work, however, we approximate
the maximizer of Problem~\ref{prob:rsrgs} with a sample, certify, and score
strategy layered on an established force-closure synthesizer.
Algorithm~\ref{alg:synth} summarizes the procedure, and we describe its stages
next.

\begin{algorithm}[t]
	\caption{Risk-Sensitive Grasp Synthesis}\label{alg:synth}
	\DontPrintSemicolon
	\KwIn{Object model, hand model, friction prior $p_\mu$, confidence level $\beta$,
		seed budget $S$}
	\KwOut{Scored grasp pool $\mathcal{G}_{\mathrm{pool}}$, best grasp $g^*$}
	$\mathcal{G}_{\mathrm{pool}} \leftarrow \emptyset$;\quad
	$v_\beta \leftarrow \CVaR(\mu)$, $\mu \sim p_\mu$\;
	\For{$s = 1, \ldots, S$}{
		Sample a seeded palm pose and finger preshape\;
		Refine the sample with the force-closure synthesis NLP\;
		\If{the refined grasp certifies force closure via \eqref{eq:minweight}, $\ell^* > 0$}{
			Add $g_s$, with contact set $\mathcal{C}_s$ and primitive wrenches,
			to $\mathcal{G}_{\mathrm{pool}}$\;
		}
	}
	\ForEach{$g_s \in \mathcal{G}_{\mathrm{pool}}$}{
		Assemble $\mathcal{W}^{(\beta)} = \mathcal{W}(v_\beta)$ from the stored
		primitive wrenches\;
		$\epsb(g_s) \leftarrow$ signed inscribed-ball radius of
		$\mathcal{W}^{(\beta)}$\;
	}
	$g^* \leftarrow \arg\max_{g_s \in \mathcal{G}_{\mathrm{pool}}} \epsb(g_s)$\;
	\KwRet{$(\mathcal{G}_{\mathrm{pool}},\, g^*)$}
\end{algorithm}

The procedure begins with a sampling stage that draws seeded palm poses and finger preshapes around the object
and refines each sample through the FRoGGeR min-weight nonlinear
program~\cite{li_frogger_2023}, which certifies force closure ($\ell^* > 0$) at
convergence and serializes the contact set and primitive wrenches of every
accepted grasp. Next, the scoring stage assembles the risk-adjusted body
$\mathcal{W}^{(\beta)}$ once per grasp at the risk-adjusted friction $v_\beta$
and computes $\epsb$ as its signed inscribed radius (\Cref{def:rsepsmet}), so the
risk adjustment scores an entire pool with one convex-hull computation per grasp
and no simulation. The selection stage keeps the scored pool for the aggregate
studies and returns the highest-ranked grasp when the task requires one execution,
as in the per-object analysis of \Cref{tab:main_results}. Because
Algorithm~\ref{alg:synth} scores a sampled pool rather than computing gradients (e.g., $\nabla_\theta \Qb$),
it does not exploit the differentiability established in Theorem~\ref{thm:diff}.
In the synthesis study of \Cref{ssec:synth_objective}, we take one step in the differentiability
direction by optimizing the differentiable risk-adjusted min-weight objective inside the synthesizer itself. Nevertheless, we leave a gradient-based synthesis extension of our pipeline to future work.

\section{Experiments}\label{sec:exp}
The protocol below evaluates our risk-sensitive metrics on stored grasp datasets across three robotic hands.

\subsection{Robotic Hand Platforms and Grasp Datasets}\label{ssec:platform}
We evaluate our metrics on grasp datasets for several anthropomorphic hands, the LEAP~\cite{shaw_leap_2023}
and Allegro hands for the aggregate and synthesis studies, the RealHand L6 for the
per-object analysis, and the Shadow Hand for the cross-dataset study
(\Cref{ssec:crosscorpus}). Grasps come from a Drake-backed pipeline derived from
FRoGGeR~\cite{li_frogger_2023}, which generates candidate contacts, filters them by
opposition geometry, and certifies force closure analytically through the min-weight
linear program of \eqref{eq:minweight} at synthesis time. Each accepted grasp
serializes its grasp map and contact set, and our metrics run
analytically from the stored logs without re-running the simulator. The dynamic
six-axis shake~\cite{li2023gendexgrasp, zhong_gagrasp_2025} in \Cref{ssec:predictive} is a separate predictive-validity study on
these stored grasps, not a synthesis-time verification. We run headless for synthesis and
use Drake's offscreen rendering for visualization. The same pipeline also retargets
to other arm-hand systems. The appendix documents a ZArm-mounted LEAP variant with
arm-reachability and fixture-collision gates, together with a qualitative gallery of representative certified-robust grasps
(\Cref{fig:leap_gallery}).

\subsection{Object Set}\label{ssec:objects}
Each study uses a different object set. The aggregate study
(\Cref{ssec:divergence}) evaluates a combined set of 1{,}599 force-closed grasps, 800
on LEAP and 799 on Allegro, over eight objects, four geometric primitives (sphere,
cube, box, cylinder) and four YCB objects~\cite{calli2015ycb} (tennis ball, master chef can, potted meat
can, mustard bottle). The broader dexterous grasping literature draws on a wider range of standardized
object sets and benchmarks beyond YCB, including the Columbia Grasp
Database~\cite{goldfeder2009columbia}, the GraspIt! simulator's object
library~\cite{miller2004graspit}, the Yale-CMU-Berkeley manipulation
benchmark~\cite{calli2017yale}, and hand-pose-annotated datasets such as
DexYCB~\cite{chao2021dexycb} and GRAB~\cite{taheri2020grab}, together with the GRASP
taxonomy of human grasp types~\cite{feix2016grasp} and the MANO hand
model~\cite{romero2017embodied} used to parameterize several of these resources.
The predictive-validity study (\Cref{ssec:predictive}) drops
the sphere, since the rotational symmetry of a sphere collapses the six shake
directions to a single representative outcome and stops the test from discriminating
on orientation, and reports 641 LEAP grasps over the remaining seven objects. The
per-object analysis (\Cref{tab:main_results}) reports one representative RealHand L6
grasp for each of seven objects (cube, box, cylinder, potted meat can, mustard
bottle, tomato soup can, tennis ball), a set that swaps the tomato soup can in for
the master chef can to keep every object representable in the L6 aperture. The
primitives provide controlled baselines with known symmetry, while the YCB objects
introduce realistic surface geometry and scale variation.

\subsection{Grasp Synthesis}\label{ssec:synth_protocol}
For each object we draw 100 seeded synthesis attempts. Each attempt samples a
palm pose and finger preshape around the object, refines it through the
FRoGGeR nonlinear program, and keeps the result only when the solver converges
to a force-closed grasp ($\ell^* > 0$ via \eqref{eq:minweight}). The resulting protocol
yields the 800 LEAP and 799 Allegro grasps of the aggregate study and the
RealHand L6 pool behind \Cref{tab:main_results}. Scoring follows the analytic
protocol of \Cref{ssec:reeval}, one risk-adjusted hull per grasp, with no Monte
Carlo sampling at synthesis time.

\subsection{Friction Priors}\label{ssec:friction}
We model friction uncertainty with two priors. The nominal prior is a tight Gaussian,
$\mu \sim \mathcal{N}(0.7, 0.1^2)$, centered at a typical rubber-on-plastic coefficient
with mild spread. The adverse prior is a bimodal mixture,
$0.5\,\mathcal{U}(0.7, 1.0) + 0.5\,\mathcal{U}(0.1, 0.3)$, that places half its mass in
a slippery low-friction tail to model surfaces with mixed-friction populations. The
nominal prior probes the regime where the risk adjustment is mild, and the adverse
prior stresses friction sensitivity.

\subsection{Analytic Evaluation Protocol}\label{ssec:reeval}
We evaluate every grasp in the stored dataset analytically, without Monte Carlo
sampling. By Lemma~\ref{lem:linearscaling} the primitive wrenches are affine in the
friction coefficient, so the risk-adjusted quality $\epsb$ of \Cref{def:rsepsmet}
reduces to one inscribed-radius computation per grasp. We assemble the polytope
$\mathcal{W}^{(\beta)} = \mathcal{W}(v_\beta)$ at the risk-adjusted friction
$v_\beta = \CVaR(\mu)$ of the prior, estimated once per prior from
$2{\times}10^5$ samples, build its convex hull, and take the minimum
origin-to-facet distance. We report $\epsb$
at confidence $\beta = 0.9$ under each prior, taking $\epsb$ as a continuous signed
margin, negative when the risk tail violates closure, so it ranks grasps without a clamp.
The sign of $\epsb$ is the closure certificate of
Corollary~\ref{cor:epsb_certificate}.

\subsection{Baselines}\label{ssec:baselines}
We compare against two baselines. The first is the nominal Ferrari-Canny epsilon
$\eps_\text{nom}$, which evaluates the GWS inscribed-ball radius at the fixed nominal
friction $\mu=0.7$. The aggregate and per-object analyses evaluate
$\eps_\text{nom}$ with each cone edge normalized to a unit contact force, the
convention of~\cite{ferrariPlanningOptimalGrasps1992}, while the predictive and
cross-dataset studies evaluate it on the same unit-normal-force cone that defines
$\epsb$. At a fixed friction the two conventions differ by the constant factor
$(1+\mu_\text{nom}^2)^{-1/2} \approx 0.82$, a global rescaling that leaves every
sign, ranking, correlation, and certified fraction unchanged. The second is the min-weight metric $\ell^*$ from
FRoGGeR~\cite{li_frogger_2023}, computed by solving the linear program
\begin{equation}\label{eq:minweight}
	\max_{\alpha, \ell}\; \ell \quad \text{s.t.}\; W\alpha = 0,\;
	\mathbf{1}^\top{\alpha} = 1,\; \alpha_i \ge \ell \;\forall\, i,
\end{equation}
where $W$ is the $6 \times m$ matrix of primitive wrenches, \({\alpha}=[\alpha_1, \alpha_2, \ldots, \alpha_m]^\top \in \mathbb{R}^m\) is a vector of non-negative weights, one per GWS friction-cone generator, and $m = n_c n_s$ for $n_c$ contacts and $n_s$ cone edges. We report the raw $\ell^*$, which Li et
al.~\cite{li_frogger_2023} normalize to $m \cdot \ell^*$ as a practical quality
indicator for precision grasps.

\subsection{Simulated Execution Studies}\label{ssec:simexec}
Our Drake simulation testbed evaluates every certified grasp using the protocol of
\Cref{apx:drkexec}. For each grasp, we establish contact through a small pad welded
at each synthesized fingertip, close the hand to the certified configuration, and
apply the gravity-off six-axis shake. For a per-object subset, we also perform the
gravity-on lift with the graded lateral pull described in
\Cref{sssec:pickcompanion}. We run two
studies on the execution testbed. The first establishes that the certified configurations
execute, in the sense that the closure acquires the object and the acquired object
stays held at the nominal friction. The second tests whether our risk metric \(\epsb\)
predicts object retention under execution-time friction uncertainty, in particular
under slips arising from low-friction surfaces.

First, we evaluate grasp executability. We execute the seven-object certified pool of
\Cref{ssec:predictive} in simulation. For each grasp, we establish contact through
the certified pads, close the hand to the synthesized configuration, and apply the
six-axis shake at the nominal friction coefficient \(\mu = 0.7\). Note that
the present study measures whether the certified configurations acquire and hold the objects
but does not isolate the friction robustness our metric targets, since every grasp
executes at the friction it was certified against. \Cref{tab:sim_exec} reports the
per-object outcomes. Contact transfer is the dominant attrition stage, 302 of the 641
replayed grasps establish contact, while execution after transfer is reliable, 279 of
the 302 established grasps hold through the nominal shake and all 51 established
grasps in the lift subset hold through the nominal lift.

Second, we test our metric's claim through a controlled friction contrast on the same
executed pool. Rather than synthesizing one grasp per objective, we retain the
established grasps whose nominal margin \(\eps_\text{nom}\) sits at or above its
per-object median, so object identity does not drive the split, and within this
high-nominal pool the per-object median of \(\epsb\) separates a high-\(\epsb\) group, high
nominal margin and high risk margin, from a low-\(\epsb\) group, high nominal margin and
low risk margin. The two groups are the executed analog of the nominal-versus-risk
contrast, both look equally certifiable to \(\eps_\text{nom}\), and only \(\epsb\)
distinguishes them before execution. \Cref{tab:hw_2x2} reports the retention of each
group across the executed friction sweep \(\mu \in \{0.2, 0.3, 0.4, 0.7\}\). At the
nominal friction the arms hold at matching rates, 0.93 against 0.96, which is the
regime where the nominal margin certifies both. As the friction drops the arms
separate, 0.69 against 0.50 at \(\mu = 0.4\), 0.57 against 0.36 at \(\mu = 0.3\), and
0.41 against 0.32 at \(\mu = 0.2\), so among grasps the nominal margin rates equally,
\(\epsb\) identifies before execution which ones fail once the friction assumption
breaks.

\begin{table}[t]
	\centering
	\caption{Simulated Grasp Executability per Object. Attempted: certified
		grasps evaluated in simulation; Grip: the certified pads establish
		contact; Shake Hold: the grasp retains the object during the
		gravity-off six-axis shake at \(\mu = 0.7\); Lift Grip and Lift Held:
		the hand establishes contact and retains the object at \(\mu = 0.7\)
		in the 15-grasp-per-object gravity-on lift subset. MCC: Master Chef
		Can. Salmon marks the lowest grip rate and mint the highest; bold
		marks Lift Held, which equals Lift Grip in every row.}\label{tab:sim_exec}
	\small
	\setlength{\tabcolsep}{4.5pt}
	\begin{tabular}{@{}p{1.5cm} c c c c c@{}}
		\toprule
		Object & Attempted & Grip & Shake Hold & Lift Grip & Lift Held \\
		\midrule
		Cube               & 100 & 36 & 31 & 4  & \textbf{4} \\
		Box      & 98  & 45 & 43 & 7  & \textbf{7} \\
		Cylinder & 100 & 57 & 52 & 10 & \textbf{10} \\
		MCC    & 96  & 52 & 49 & 7  & \textbf{7} \\
		Mustard     & 82  & 33 & 27 & 5  & \textbf{5} \\
		Meat Can    & 98  & \cellcolor{losssalmon}28 & 26 & 4  & \textbf{4} \\
		Tennis Ball        & 67  & \cellcolor{oursmint}51 & 51 & 14 & \textbf{14} \\
		\bottomrule
	\end{tabular}
\end{table}

\begin{table}[t]
	\centering
	\caption{Friction-Contrast Retention on Executed Grasps. Retention rate of
		the two high-\(\eps_\text{nom}\) groups, split at the per-object
		\(\epsb\) median into a high-\(\epsb\) and a low-\(\epsb\) group, at each executed
		friction of the shake sweep. The two groups tie at the nominal friction and
		separate as the friction drops. Mint marks the risk-favored group, salmon
		the group the risk margin excludes.}\label{tab:hw_2x2}
	\small
	\setlength{\tabcolsep}{4pt}
	\begin{tabular}{@{}l c c c c c@{}}
		\toprule
		Group & \(n\) & \(\mu = 0.2\) & \(\mu = 0.3\) & \(\mu = 0.4\) & \(\mu = 0.7\) \\
		\midrule
		\rowcolor{oursmint}
		High \(\eps_\text{nom}\), high \(\epsb\) & 109 & 0.41 & 0.57 & 0.69 & 0.93 \\
		\rowcolor{losssalmon}
		High \(\eps_\text{nom}\), low \(\epsb\)  & 44  & 0.32 & 0.36 & 0.50 & 0.96 \\
		\bottomrule
	\end{tabular}
\end{table}
\section{Results and Discussion}\label{sec:results}
The aggregate analysis spans a combined dataset of
1{,}599 force-closed grasps, 800 on the LEAP hand and 799 on the Allegro hand over the eight core test objects. Across the grasp data, we compare our risk metric \(\epsb\) to the nominal Ferrari-Canny margin \(\eps_\text{nom}\) and to the FRoGGeR min-weight baseline
\(\ell^*\) under two friction priors, a tight Gaussian \(\mathcal{N}(0.7, 0.1^2)\)
and an adverse mixture with equal weights \(0.5\,\mathcal{U}(0.7, 1.0) + 0.5\,\mathcal{U}(0.1, 0.3)\),
at confidence \(\beta = 0.9\). Every number traces to a stored evaluation log that
records the convention, the priors with their risk-adjusted frictions, and the
method.

\begin{table}[t]
	\centering
	\caption{Aggregate Metric Relationships (Pearson \(r\), 1{,}599 Grasps). The shaded rows show the adverse prior, where our risk metric differs from both fixed-friction baselines.}\label{tab:aggregate}
	\small
	\setlength{\tabcolsep}{5pt}
	\begin{tabular}{@{}l r r r@{}}
		\toprule
		Metric Pair & LEAP & Allegro & Combined \\
		\midrule
		\(\eps_\text{nom}\) vs \(\ell^*\)                       & 0.62 & 0.62 & 0.60 \\
		\(\epsb\) (nominal prior) vs \(\eps_\text{nom}\)        & 0.95 & 0.95 & 0.95 \\
		\(\epsb\) (nominal prior) vs \(\ell^*\)                 & 0.72 & 0.71 & 0.69 \\
		\rowcolor{oursmint}
		\(\epsb\) (adverse prior) vs \(\eps_\text{nom}\)        & 0.38 & 0.30 & \textbf{0.32} \\
		\rowcolor{oursmint}
		\(\epsb\) (adverse prior) vs \(\ell^*\)                 & 0.57 & 0.52 & 0.52 \\
		\bottomrule
	\end{tabular}
\end{table}

\begin{figure}[t]
	\centering
	\includegraphics[width=\linewidth]{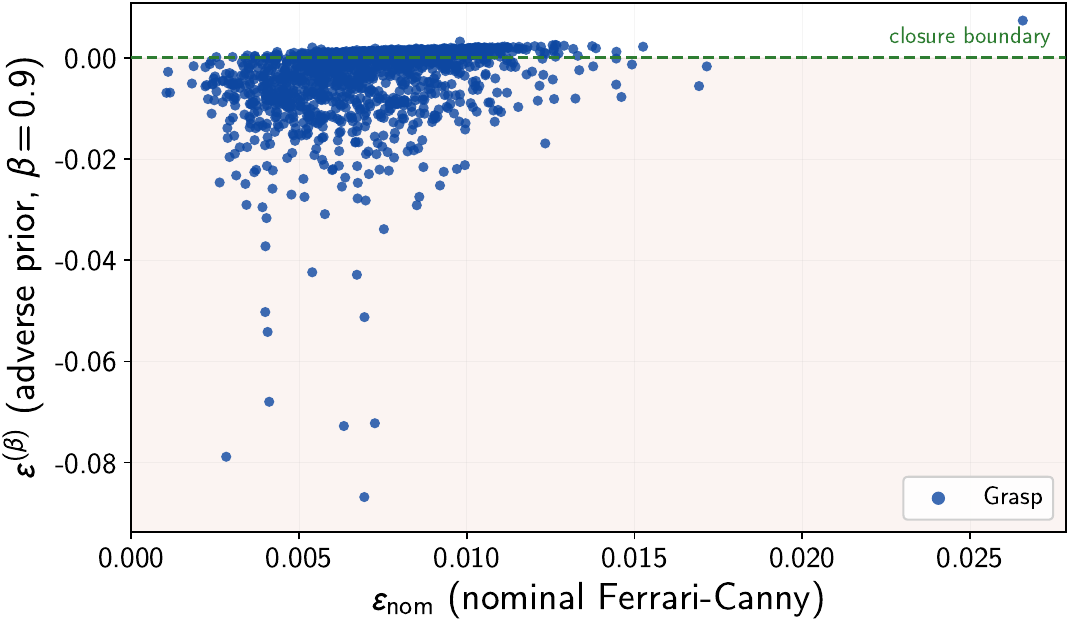}
	\caption{\textbf{Aggregate Divergence Under the Adverse Prior.} Risk metric
		\(\epsb\) under the adverse friction prior against the nominal Ferrari-Canny
		\(\eps_\text{nom}\), one point per grasp over the combined 1{,}599-grasp dataset.
		Every grasp is force-closed at nominal friction, so all points sit at positive
		\(\eps_\text{nom}\), yet the shaded population falls below the closure boundary
		in the adverse tail. The nominal margin overrates these
		grasps.}\label{fig:divergence}
\end{figure}

\subsection{Aggregate Metric Divergence}\label{ssec:divergence}
Under a tight Gaussian friction prior our risk metric \(\epsb\) tracks the nominal
Ferrari-Canny margin, with a combined Pearson correlation of 0.95
(\Cref{tab:aggregate}). The risk adjustment at the nominal prior is mild, so the two
metrics rank grasps almost identically. Under the adverse prior the picture
changes. The correlation of \(\epsb\) with \(\eps_\text{nom}\) falls to 0.32 and
with \(\ell^*\) to 0.52, so our risk metric separates from both friction-unaware
baselines (\Cref{fig:divergence}). \Cref{fig:corrdegrade} traces the resulting degradation
as the confidence level \(\beta\) rises, the nominal-prior correlation holding near
0.95 while the adverse-prior correlation settles near 0.30.

Specifically, the divergence is concentrated in a friction-sensitive majority. Of the 1{,}599
grasps that are force-closed at nominal friction, 849 (53\%) lose closure in
the adverse tail, where \(\epsb\) becomes nonpositive (57\% on LEAP, 49\%
on Allegro). Object geometry, rather than a synthesis artifact, explains the pattern.
Rounded objects remain robust: the sphere at 5\% and the tennis ball at 13\%, while
irregular objects are the most friction-sensitive, the mustard bottle at 86\%
and the potted meat can at 80\%. The nominal metric \(\eps_\text{nom}\)
assigns high scores to these friction-sensitive grasps, whereas \(\epsb\)
identifies them. The risk adjustment therefore measures a dimension of
robustness that fixed-friction metrics do not encode.

\begin{figure}[t]
	\centering
	\includegraphics[width=\linewidth]{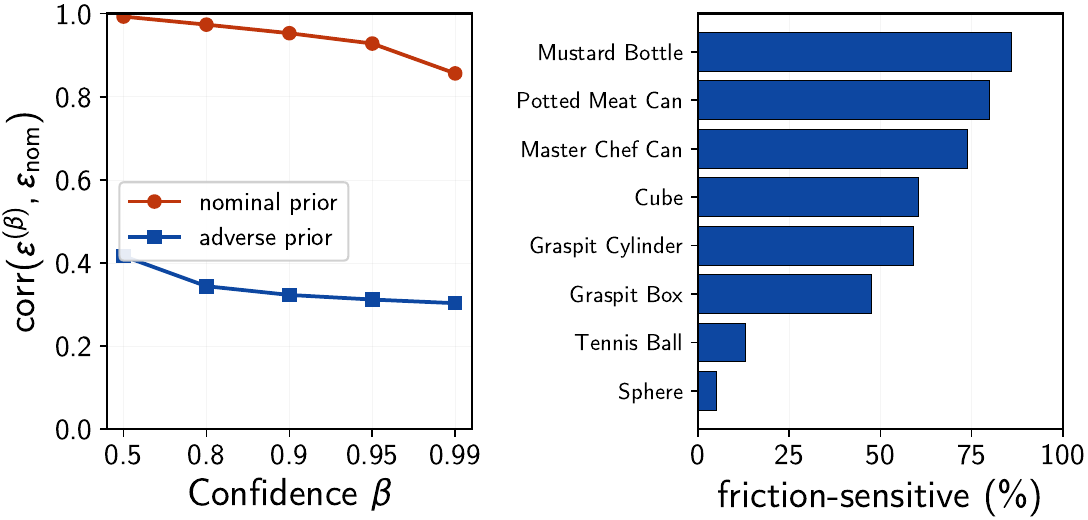}
	\caption{\textbf{Correlation Degradation and Object Stratification.} \emph{Left}: the
		Pearson correlation of \(\epsb\) with \(\eps_\text{nom}\) as the confidence
		\(\beta\) rises, holding near 0.95 under the nominal prior and falling toward
		0.30 under the adverse prior. \emph{Right}: the per-object fraction of nominally
		force-closed grasps that lose closure in the adverse tail, from the rounded
		sphere to the irregular mustard bottle.}\label{fig:corrdegrade}
\end{figure}

\subsection{Synthesis Under the Risk Objective}\label{ssec:synth_objective}
\Cref{tab:synthesis} and \Cref{fig:synthesis} demonstrate the risk adjustment's use
in a synthesis objective, beyond evaluation. We run FRoGGeR synthesis under the
differentiable risk objective, the min-weight program of \eqref{eq:minweight}
solved on the risk-adjusted body \(\mathcal{W}^{(\beta)}\) (the \(\ell^{*(\beta)}\)
member of \Cref{tab:qualfuncs}), and under the nominal min-weight objective at
matched budget over the eight core objects on both hands, then score every
synthesized grasp with our metric \(\epsb\) under the adverse prior. At matched synthesis time the
risk objective lowers the friction-sensitive fraction from 67.7 to 57.6\%
combined, and on LEAP alone from 87.3 to 75.4\%, while raising the median
adverse-tail margin from \(-0.00322\) to \(-0.00119\). The differentiable objective
we expose through the FRoGGeR pipeline uses a Gaussian friction model, so it
optimizes a mild Gaussian-prior risk while the evaluation applies the harsher
adverse prior. The improvement is therefore best understood as a transfer result. Grasps tuned
for a mild risk withstand the harsher adverse tail more often than the min-weight
grasps the FRoGGeR baseline produces.

\begin{table}[t]
	\centering
	\caption{Synthesis Under the Risk Objective Versus the Min-Weight Objective (Combined LEAP and Allegro, Adverse Prior, \(\beta = 0.9\)). The shaded row is the risk objective.}\label{tab:synthesis}
	\small
	\setlength{\tabcolsep}{3pt}
	\begin{tabular}{@{}l c c c@{}}
		\toprule
		Synth. Obj. & Friction-Sens.\ \(\downarrow\) & Median \(\epsb\) \(\uparrow\) & Solve (s) \(\downarrow\) \\
		\midrule
		\rowcolor{oursmint}
		Risk \(\ell^{*(\beta)}\)  & \textbf{72/125 (57.6\%)} & \(\mathbf{-0.00119}\) & 19.3 \\
		Min-weight \(\ell^*\) & 86/127 (67.7\%) & \(-0.00322\) & \textbf{14.9} \\
		\bottomrule
	\end{tabular}
\end{table}

\begin{figure}[t]
	\centering
	\includegraphics[width=\linewidth]{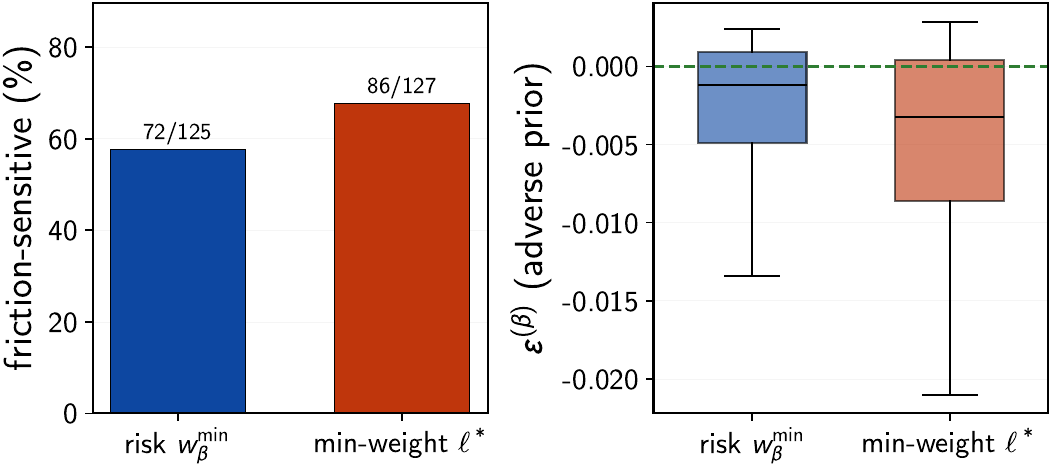}
	\caption{\textbf{Synthesis Under the Risk Objective.} \emph{Left}: the fraction of
		synthesized grasps that are friction-sensitive under the adverse prior, lower
		for the risk objective than for the min-weight objective at matched budget.
		\emph{Right}: the distribution of \(\epsb\) under the adverse prior, with the risk
		objective shifted toward the closure boundary.}\label{fig:synthesis}
\end{figure}

Next, \Cref{tab:main_results} reports a representative RealHand L6 grasp for each of seven
objects under the analytic protocol of \Cref{ssec:reeval}, the force-closed grasp that
maximizes \(\epsb\) under the nominal prior at \(\beta = 0.9\). For each grasp it lists
the nominal Ferrari-Canny margin \(\eps_\text{nom}\), the FRoGGeR min-weight
\(\ell^*\), our risk metric \(\epsb\) under the nominal prior across confidence levels
and under the adverse prior at \(\beta = 0.9\), and the adverse-prior closure certificate
\(\beta_c\).

\begin{table*}[t]
	\centering
	\caption{Per-Object RealHand L6 Analysis (Analytic, Representative Grasp per Object). \(n_c\): contact count; \(\eps_\text{nom}\): Ferrari-Canny at \(\mu{=}0.7\); \(\ell^*\): FRoGGeR min-weight; \(\eps^{(\beta)}_{\mathcal{N}}\): risk metric under the nominal prior \(\mathcal{N}(0.7, 0.1^2)\) at confidence \(\beta\); \(\eps^{(0.9)}_{\text{adv}}\): risk metric under the adverse prior at \(\beta{=}0.9\); \(\beta_c\): largest adverse-prior confidence with \(\epsb > 0\). Salmon cells lose closure in the adverse tail (\(\eps^{(0.9)}_{\text{adv}} \le 0\)), mint cells retain it, and the mustard bottle pairs the weakest certificate with the deepest loss in the adverse tail.}\label{tab:main_results}
	\small
	\setlength{\tabcolsep}{4.5pt}
	\begin{tabular}{@{}l r c c c c c c c c@{}}
		\toprule
		Object & \(n_c\) & \(\eps_\text{nom}\) & \(\ell^*\) & \(\eps^{(0.5)}_{\mathcal{N}}\) & \(\eps^{(0.8)}_{\mathcal{N}}\) & \(\eps^{(0.9)}_{\mathcal{N}}\) & \(\eps^{(0.95)}_{\mathcal{N}}\) & \(\eps^{(0.9)}_{\text{adv}}\) & \(\beta_c\)\\
		\midrule
		Cube            & 5 & 0.0075 & 0.0211 & 0.0083 & 0.0076 & 0.0072 & 0.0067 & \cellcolor{losssalmon}\(-0.0003\) & 0.85 \\
		Box             & 5 & 0.0099 & 0.0268 & 0.0103 & 0.0090 & 0.0082 & 0.0075 & \cellcolor{losssalmon}\(-0.0001\) & 0.87 \\
		Cylinder        & 5 & 0.0078 & 0.0266 & 0.0085 & 0.0078 & 0.0073 & 0.0069 & \cellcolor{oursmint}0.0010      & 1.00 \\
		Potted Meat Can & 5 & 0.0089 & 0.0307 & 0.0100 & 0.0092 & 0.0085 & 0.0079 & \cellcolor{oursmint}0.0003      & 0.99 \\
		Mustard Bottle  & 5 & 0.0081 & 0.0259 & 0.0086 & 0.0077 & 0.0071 & 0.0065 & \cellcolor{losssalmon}\(-0.0036\) & \cellcolor{losssalmon}0.46 \\
		Tomato Soup Can & 5 & 0.0091 & 0.0328 & 0.0099 & 0.0087 & 0.0080 & 0.0075 & \cellcolor{oursmint}0.0003      & 0.97 \\
		Tennis Ball     & 5 & 0.0070 & 0.0294 & 0.0075 & 0.0066 & 0.0061 & 0.0057 & \cellcolor{oursmint}0.0010      & 1.00 \\
		\bottomrule
	\end{tabular}
\end{table*}

\begin{figure}[t]
	\centering
	\includegraphics[width=\linewidth]{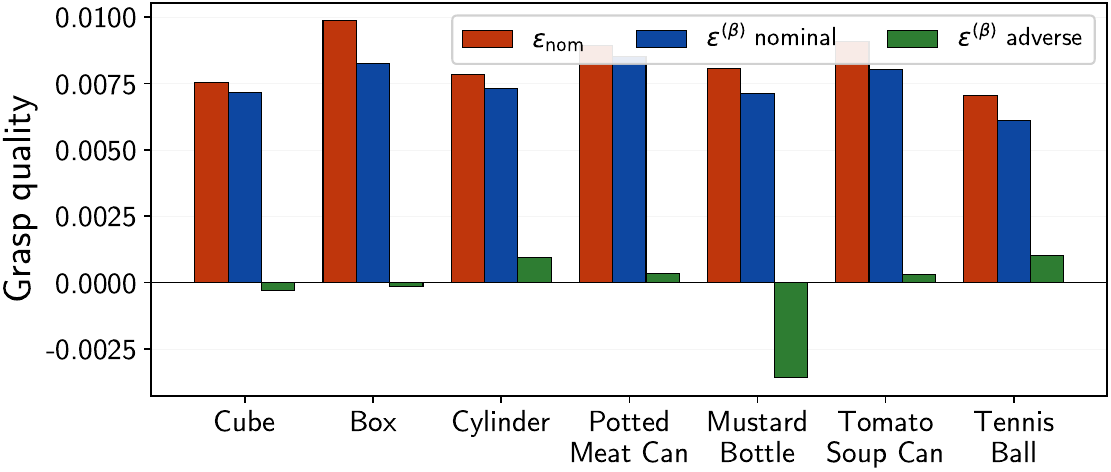}
	\caption{\textbf{Per-Object Quality Comparison.} Nominal Ferrari-Canny
		$\eps_\text{nom}$, our risk metric $\epsb$ under the nominal prior, and $\epsb$
		under the adverse prior at $\beta = 0.9$, for each RealHand L6 grasp. The nominal
		prior tracks the friction-unaware margin, while the adverse prior collapses toward
		and below the closure boundary, sharply for the mustard
		bottle.}\label{fig:quality_comparison}
\end{figure}

\subsection{Per-Object Analysis}\label{ssec:cvar_sensitivity}
The per-object data reproduce the aggregate divergence of \Cref{ssec:divergence} at
the level of single grasps and confirm the analytic properties of \Cref{sec:theory}.
Under the nominal prior our risk metric \(\epsb\) tracks the nominal Ferrari-Canny
margin, so both rate the seven grasps as comparable (\Cref{fig:quality_comparison}).
Under the adverse prior, \(\epsb\) separates them, falling to a small positive margin
for the rounded tennis ball and cylinder and below the closure boundary for the cube,
box, and mustard bottle, exposing the per-object friction sensitivity that the
friction-unaware margin cannot capture. \Cref{fig:grasp_gallery} shows twelve such
RealHand L6 grasps spanning household items, primitive shapes, and rounded objects,
each force-closed (\(\ell^* > 0\)) and satisfying the friction-tail requirement
\(\eps^{(0.99)} > 0\) (\(\epsb\) at \(\beta = 0.99\) under the nominal prior, \Cref{def:rsepsmet}).

\begin{figure}[t]
	\centering
	\includegraphics[width=.98\linewidth]{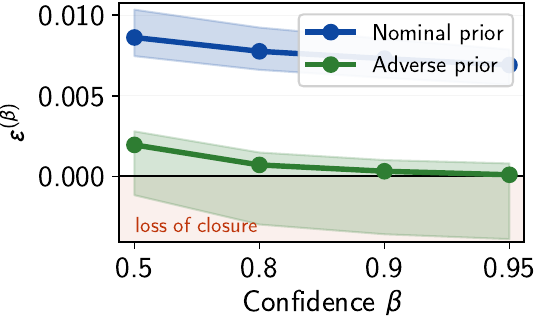}
	\caption{\textbf{Risk Sensitivity to the Confidence Level.} Our risk metric
		\(\epsb\) against confidence \(\beta\), summarized over the seven representative
		RealHand L6 grasps of \Cref{tab:main_results} as a median line with a
		min-to-max band. Under the nominal prior \(\epsb\) stays above the closure
		boundary. Under the adverse prior the band crosses into the loss-of-closure
		region as \(\beta\) rises. \(\epsb\) is non-increasing in \(\beta\)
		(the \(\epsb\) analog of Theorem~\ref{thm:monotone}, via
		\Cref{rem:epsmu_curve}).}\label{fig:beta_sensitivity}
\end{figure}

Empirically, the data confirm the two analytic guarantees of \Cref{sec:theory}. The margin
\(\epsb\) is non-increasing in the confidence \(\beta\), since it evaluates the
non-decreasing piecewise-linear quality curve of \Cref{rem:epsmu_curve} at a
risk-adjusted friction \(v_\beta\) that falls as \(\beta\) rises (the
\(\epsb\) analog of Theorem~\ref{thm:monotone}), and this holds for every grasp
across both priors in \Cref{tab:main_results} and
\Cref{fig:beta_sensitivity}. By Corollary~\ref{cor:epsb_certificate}, the sign of \(\epsb\)
is a probabilistic closure certificate. Under the nominal prior every selected grasp
is certified at all confidence levels, while under the adverse prior our certificate
\(\beta_c\) tightens with friction sensitivity, from \(\beta_c = 1.0\) for the cylinder
and tennis ball down to \(\beta_c = 0.46\) for the mustard bottle. The FRoGGeR
min-weight \(\ell^*\) varies little across the seven grasps, consistent with its role
as a friction-unaware force-closure margin that our risk metric complements rather than
replaces.

\begin{remark}[Piecewise-Linear Quality Curve]\label{rem:epsmu_curve}
Because Lemma~\ref{lem:linearscaling} establishes that each primitive wrench is affine in \(\mu\), the support function \(h_{\mathcal{W}(\mu)}(d)\) is piecewise linear and convex in \(\mu\) for every fixed direction \(d\), a maximum of affine vertex supports.  The Ferrari-Canny margin \(\eps(g,\mu) = \min_{\|d\|=1} h_{\mathcal{W}(\mu)}(d)\) is therefore piecewise linear in \(\mu\).  For any grasp there exists the critical friction coefficient \(\mu^{*} = \inf\{\mu : \eps(g,\mu) > 0\}\) of the Necessity remark below which force closure is lost. Plotting \(\eps(g,\mu)\) against \(\mu\) directly reveals \(\mu^{*}\) and shows what fraction of the prior \(p_\mu\) falls below it (\Cref{fig:epsmu}).  The resulting curve is the most direct expression of the analytic reduction, a single sweep of \(\mu\) recovers the full failure boundary without Monte Carlo sampling, and \(\epsb\) evaluates the curve at the risk-adjusted friction \(v_\beta\), the mean of the adverse tail that the confidence level selects.
\end{remark}

\begin{figure}[t]
	\centering
	\includegraphics[width=.94\linewidth]{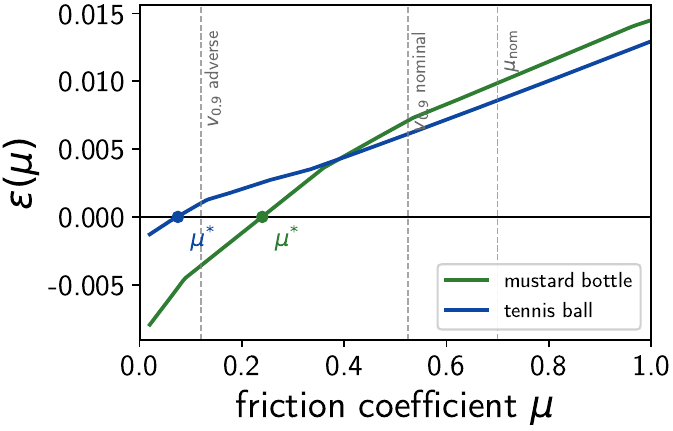}
	\caption{\textbf{Per-Grasp Quality Curve.} The Ferrari-Canny margin
		\(\eps(g,\mu)\) of the mustard-bottle and tennis-ball grasps of
		\Cref{tab:main_results}, swept over the friction coefficient. The curve is
		piecewise linear (\Cref{rem:epsmu_curve}) and comes from the stored grasp
		maps in the unit-normal-force convention of \(\epsb\), with no sampling.
		The zero crossing marks the critical friction \(\mu^{*}\) below which
		closure is lost. The mustard bottle's \(\mu^{*}\) sits above the
		adverse-prior risk-adjusted friction \(v_{0.9}\), so its adverse margin
		\(\eps^{(0.9)}_{\text{adv}}\) is negative, while the tennis ball's \(\mu^{*}\) sits
		below \(v_{0.9}\) and its certificate holds across the
		sweep.}\label{fig:epsmu}
\end{figure}
\begin{figure*}
    \centering
    \subfloat[FRoGGeR-LEAP (\(N{=}800\)).]{%
        \includegraphics[width=0.32\textwidth]{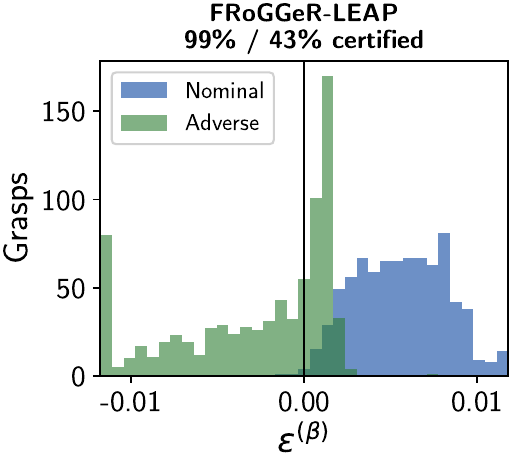}%
        \label{fig:cc_leap}}\hfill
    \subfloat[FRoGGeR-Allegro (\(N{=}799\)).]{%
        \includegraphics[width=0.32\textwidth]{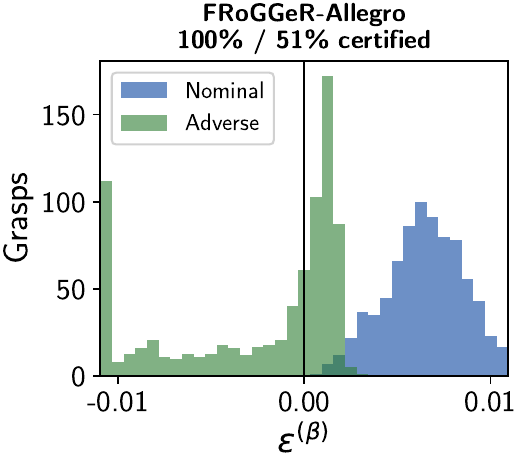}%
        \label{fig:cc_allegro}}\hfill
    \subfloat[DGN-Shadow (\(N{=}400\)).]{%
        \includegraphics[width=0.32\textwidth]{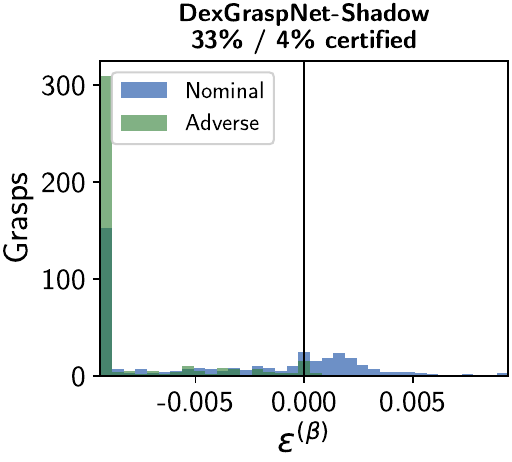}%
        \label{fig:cc_shadow}}
    \caption{\textbf{Cross-Dataset Distribution of \(\epsb\).}
        Per-grasp histograms of our risk metric \(\epsb\) at \(\beta{=}0.9\) under the
        nominal prior \(\mathcal{N}(0.7, 0.1^2)\) (blue) and the adverse prior
        \(0.5\,\mathcal{U}(0.7, 1.0) + 0.5\,\mathcal{U}(0.1, 0.3)\) (green), for each
        dataset. The nominal distribution sits above the closure boundary while the
        adverse distribution shifts across it. The title reports the certified-robust
        fraction (\(\epsb > 0\)) under the nominal and adverse priors, and the adverse
        mass at or below zero is the certification gap of
        \Cref{tab:crosscorpus}.}
    \label{fig:crosscorpus_hist}
\end{figure*}

\begin{figure*}
	\centering
	\newlength{\gallerywidth}%
	\setlength{\gallerywidth}{\dimexpr.16\textwidth-2\fboxrule\relax}%
	{\setlength{\fboxsep}{0pt}%
	\newcommand{\gcell}[1]{\fbox{\includegraphics[width=\gallerywidth, trim=0pt 150pt 0pt 0pt, clip]{#1}}}%
	\gcell{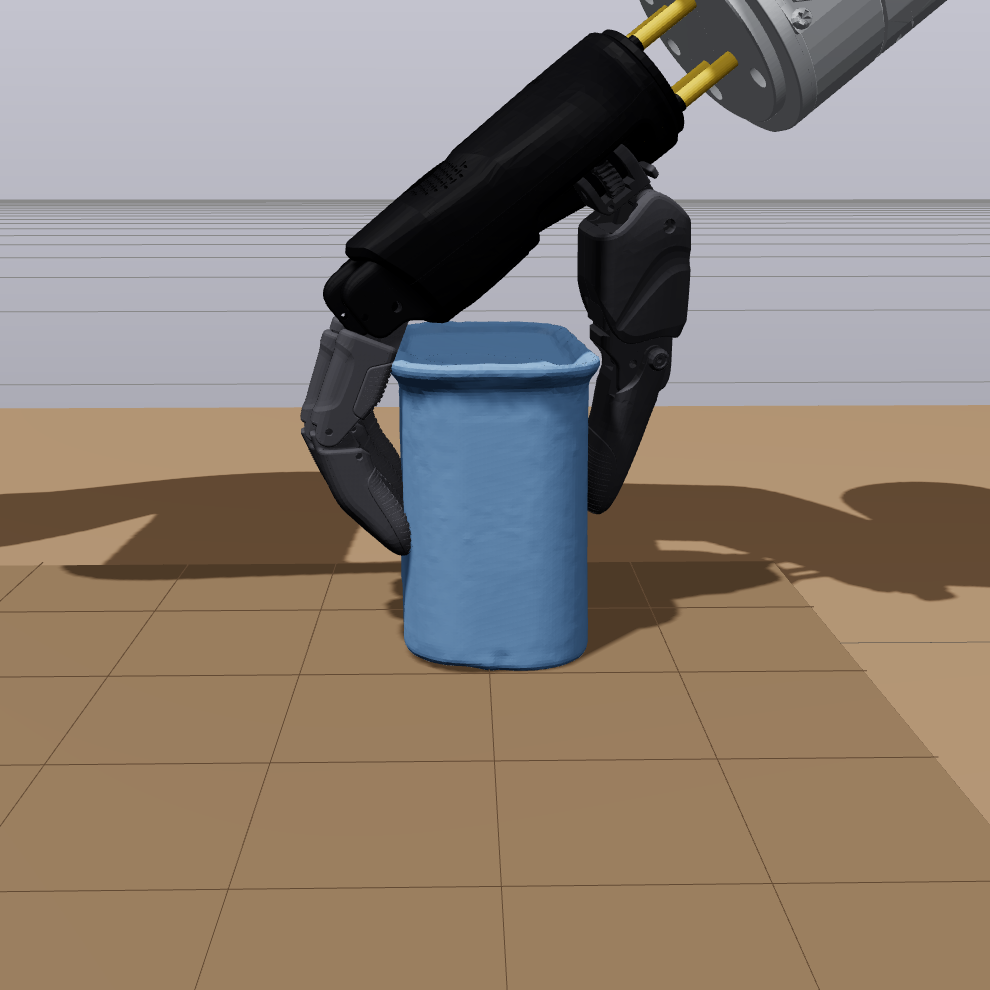}%
	\gcell{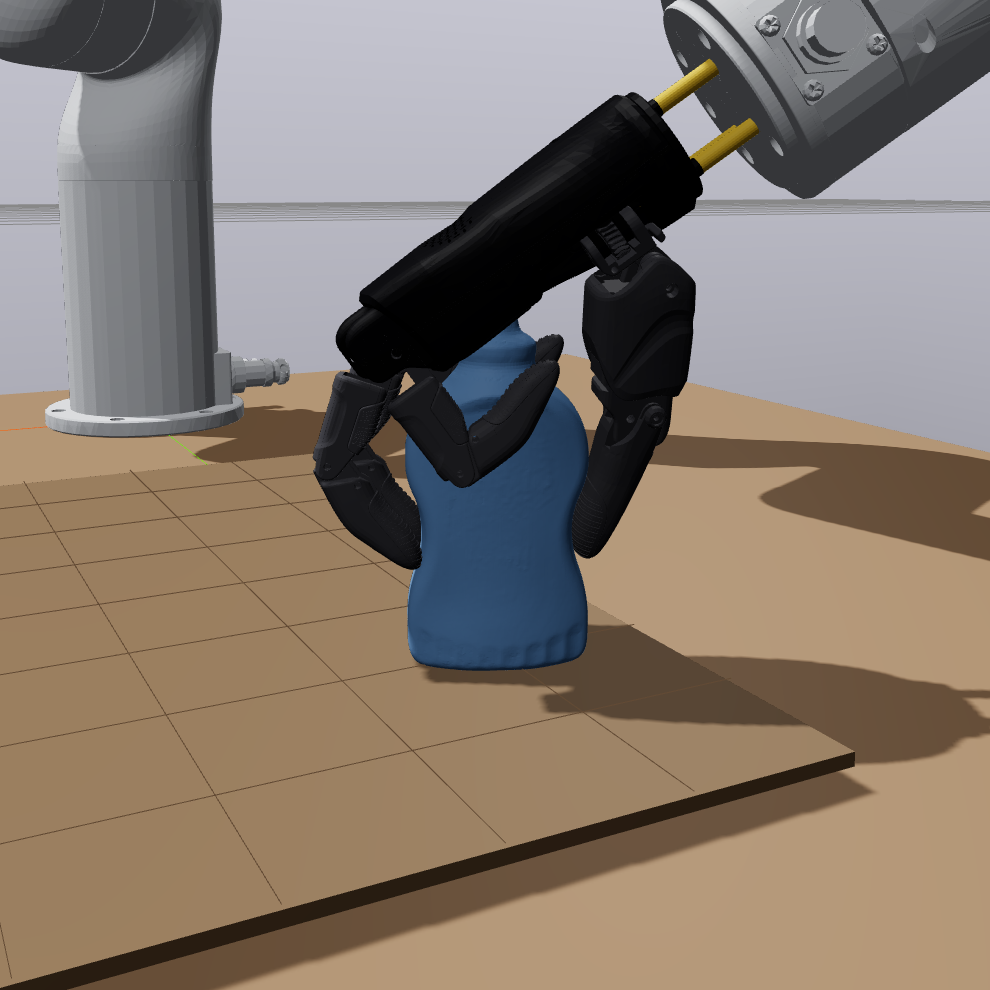}%
	\gcell{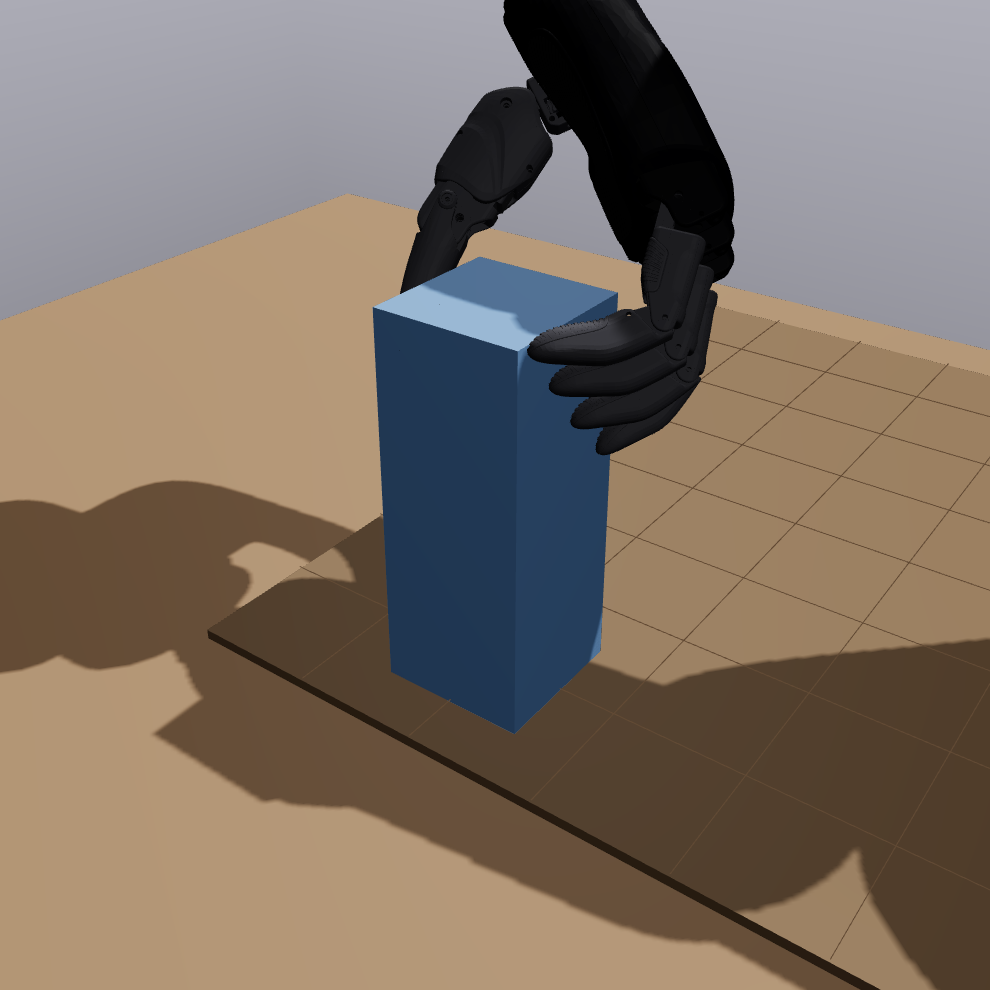}%
	\gcell{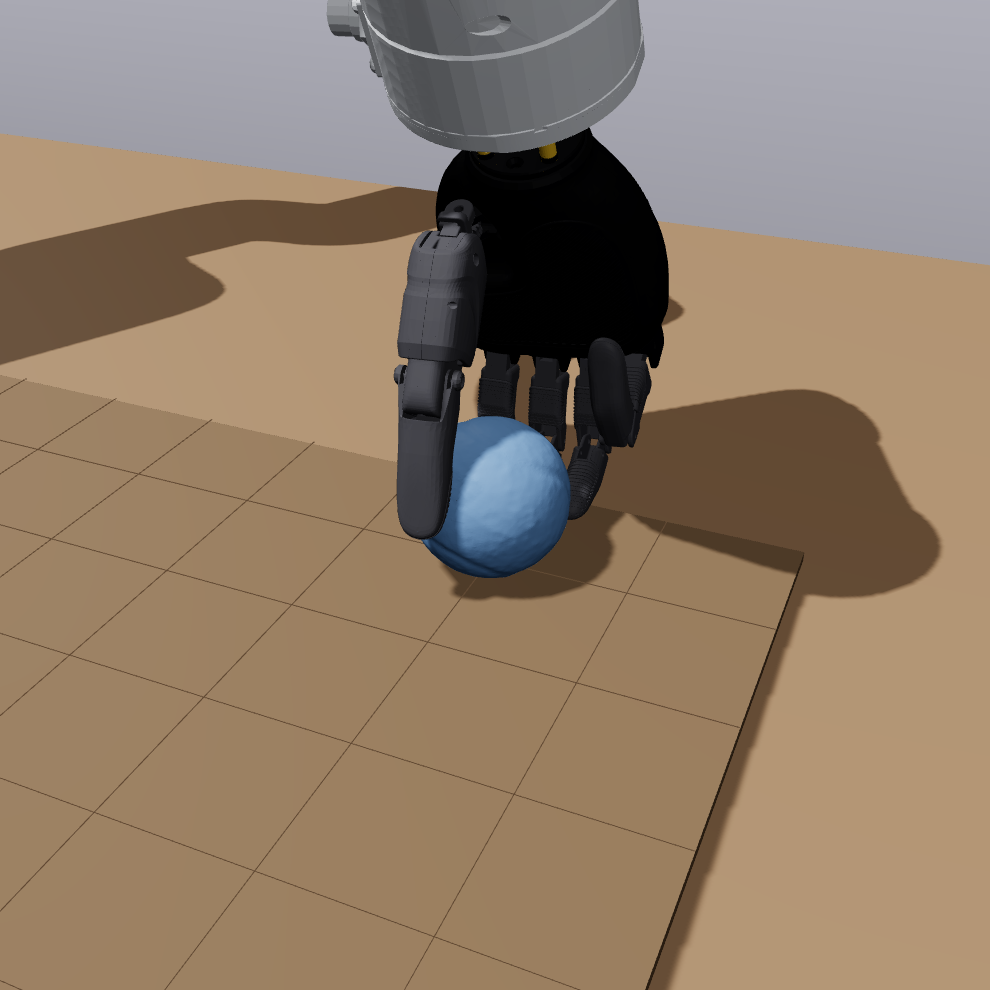}%
	\gcell{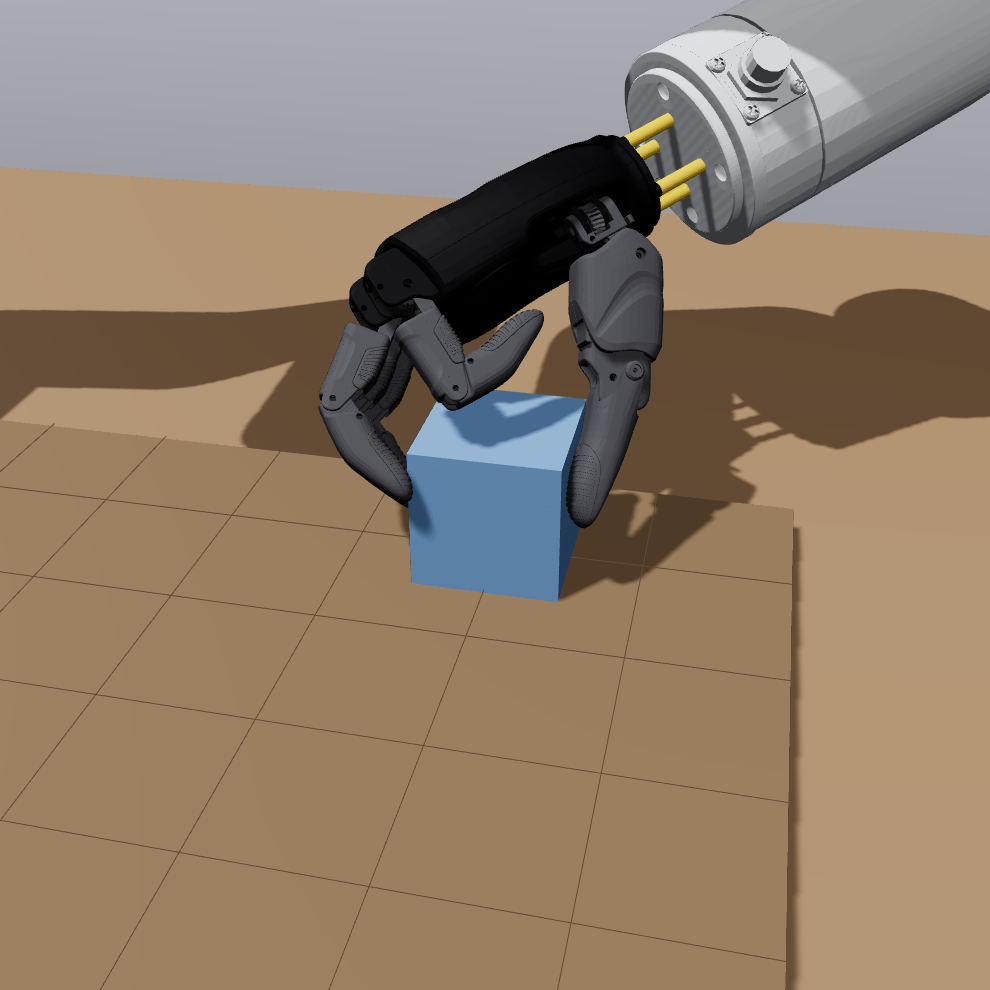}%
	\gcell{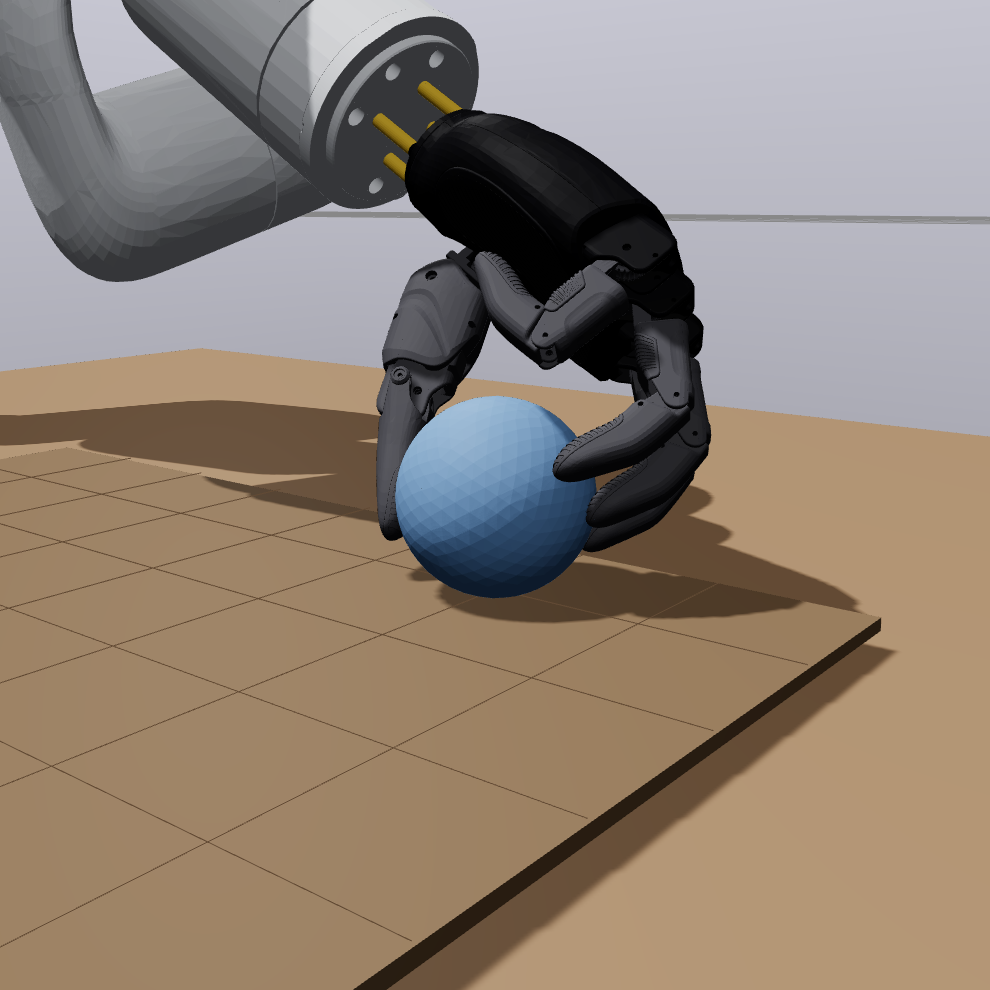}%
	\\[0pt]
	\gcell{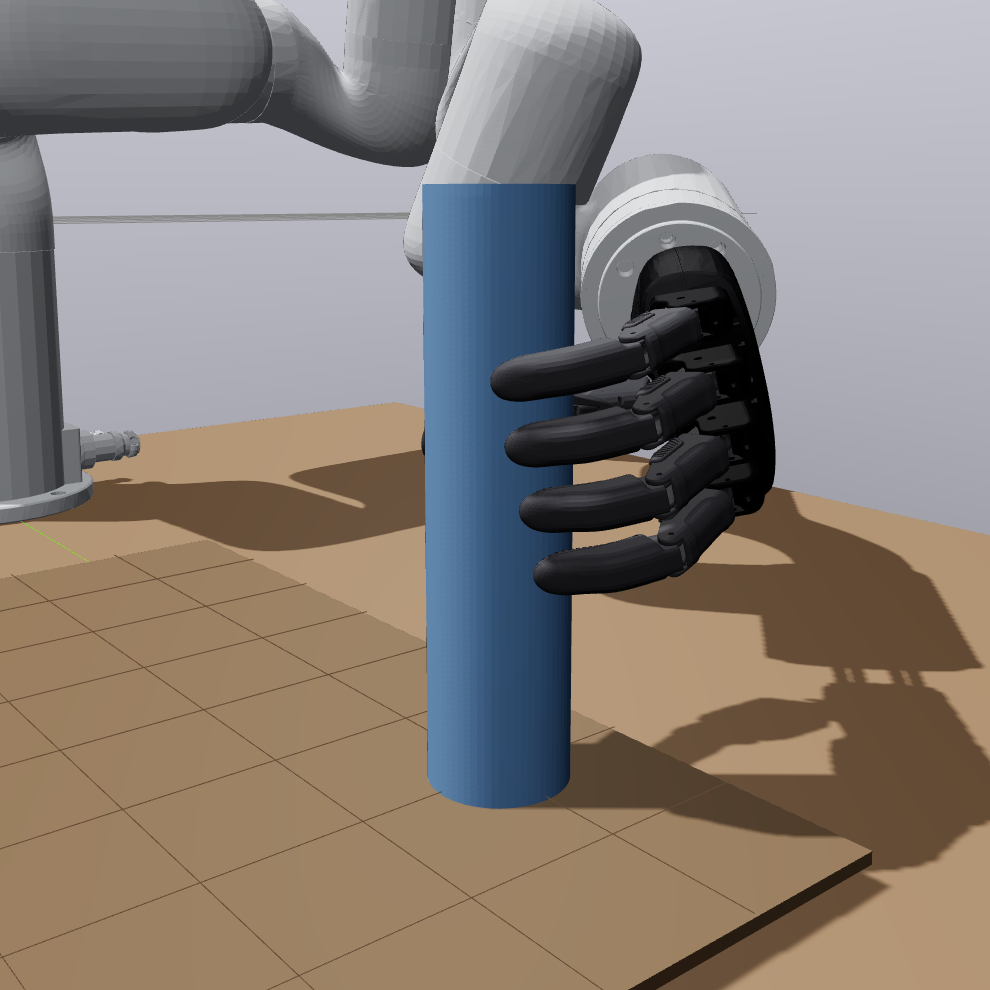}%
	\gcell{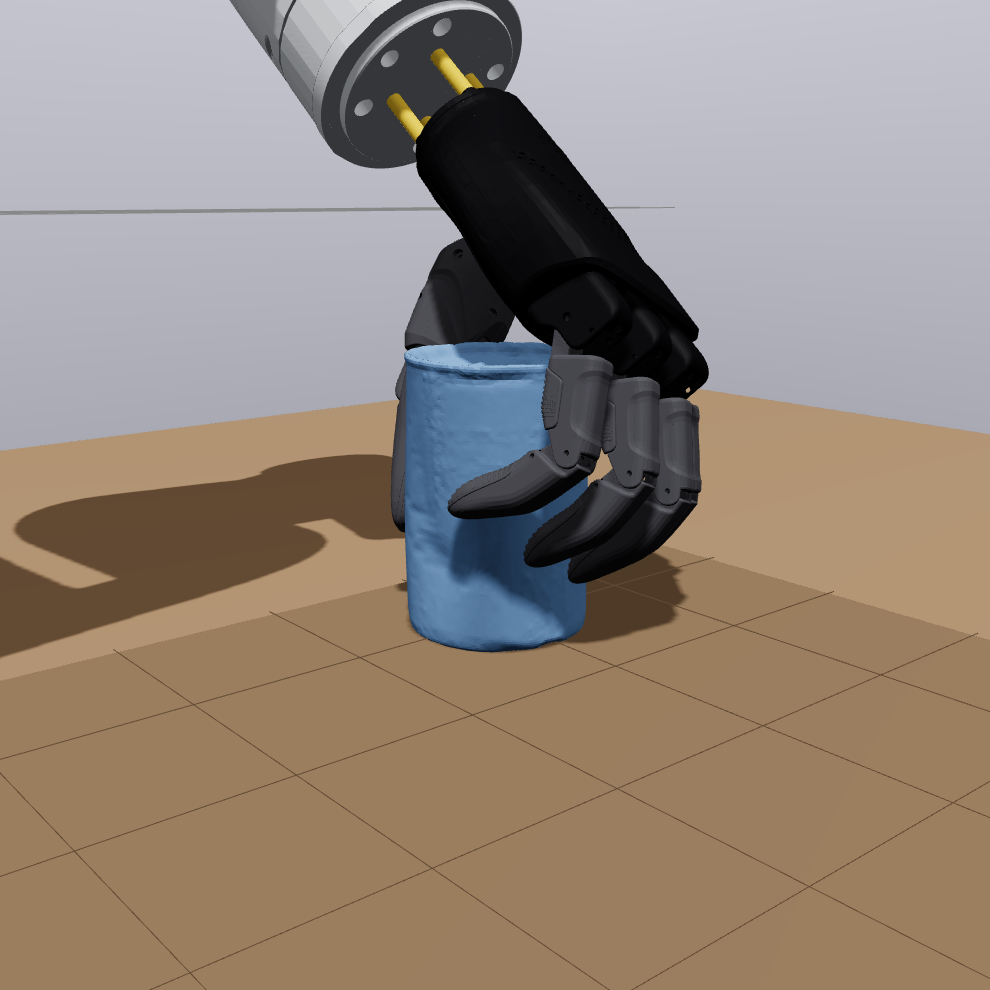}%
	\gcell{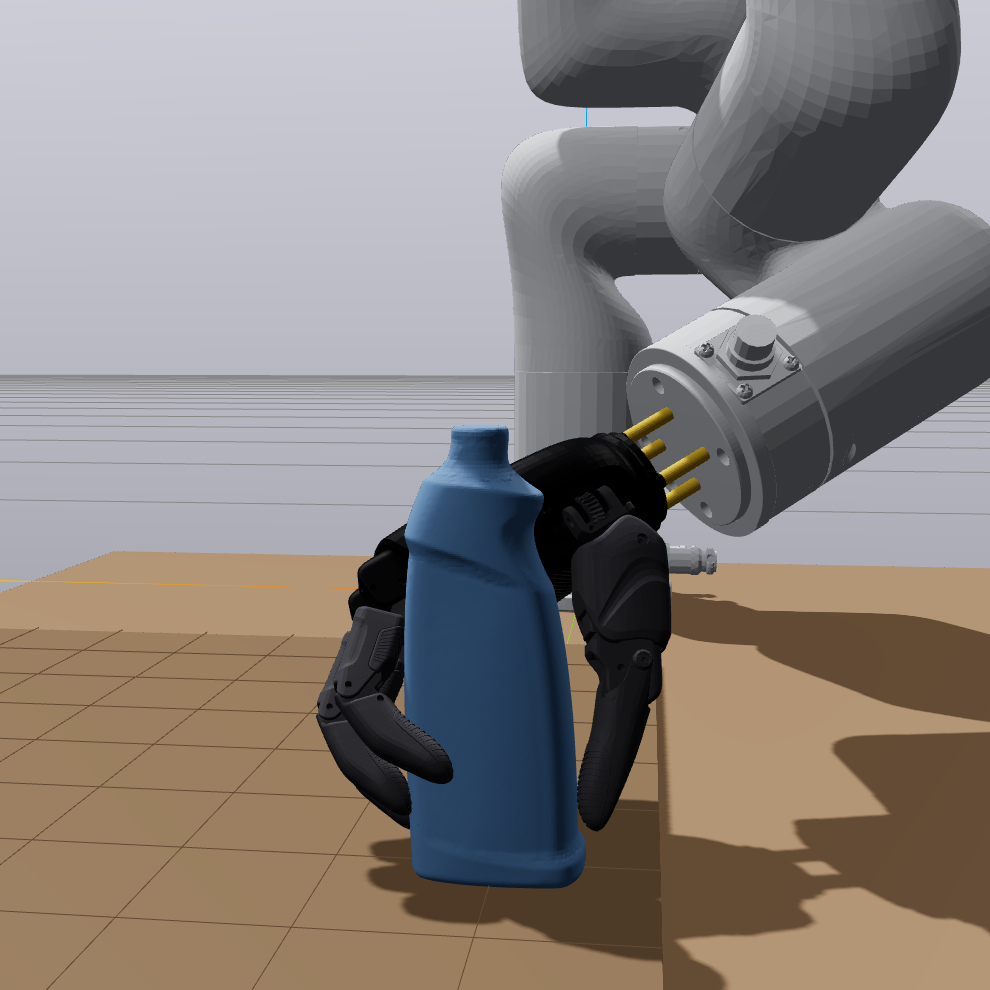}%
	\gcell{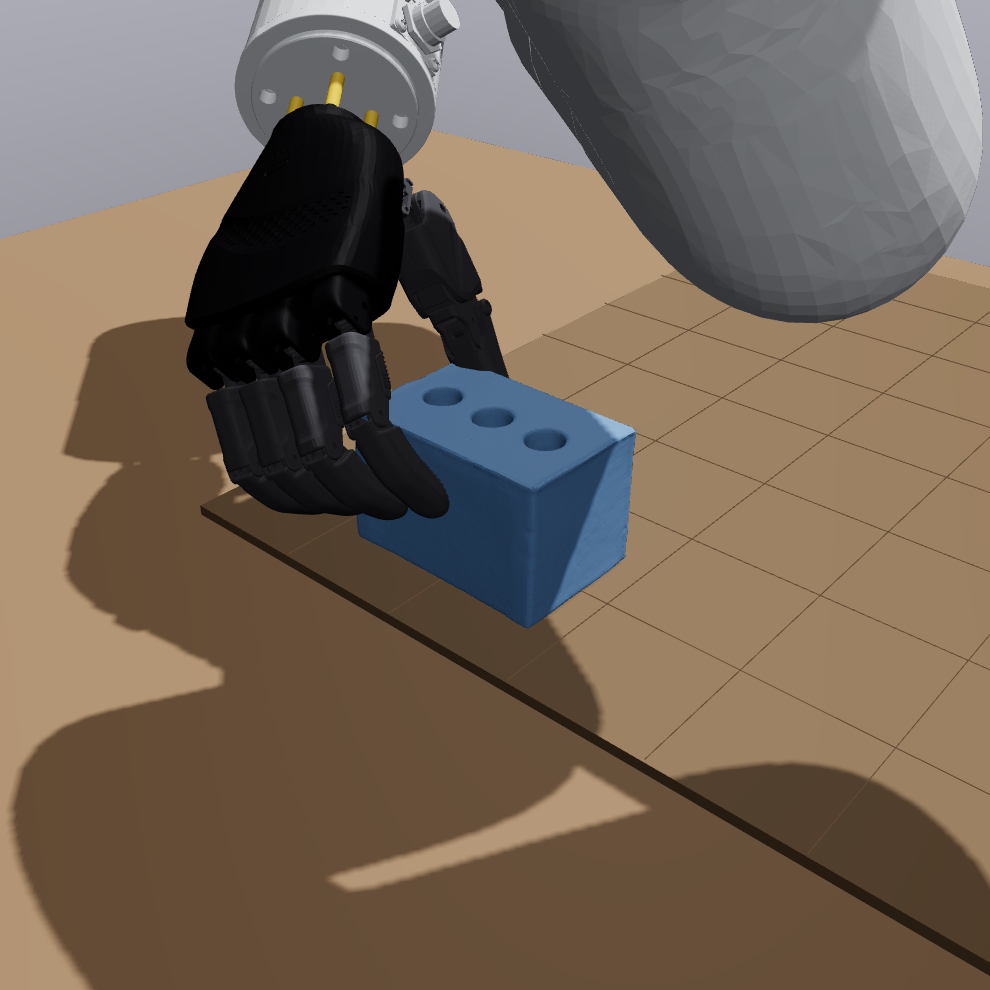}%
	\gcell{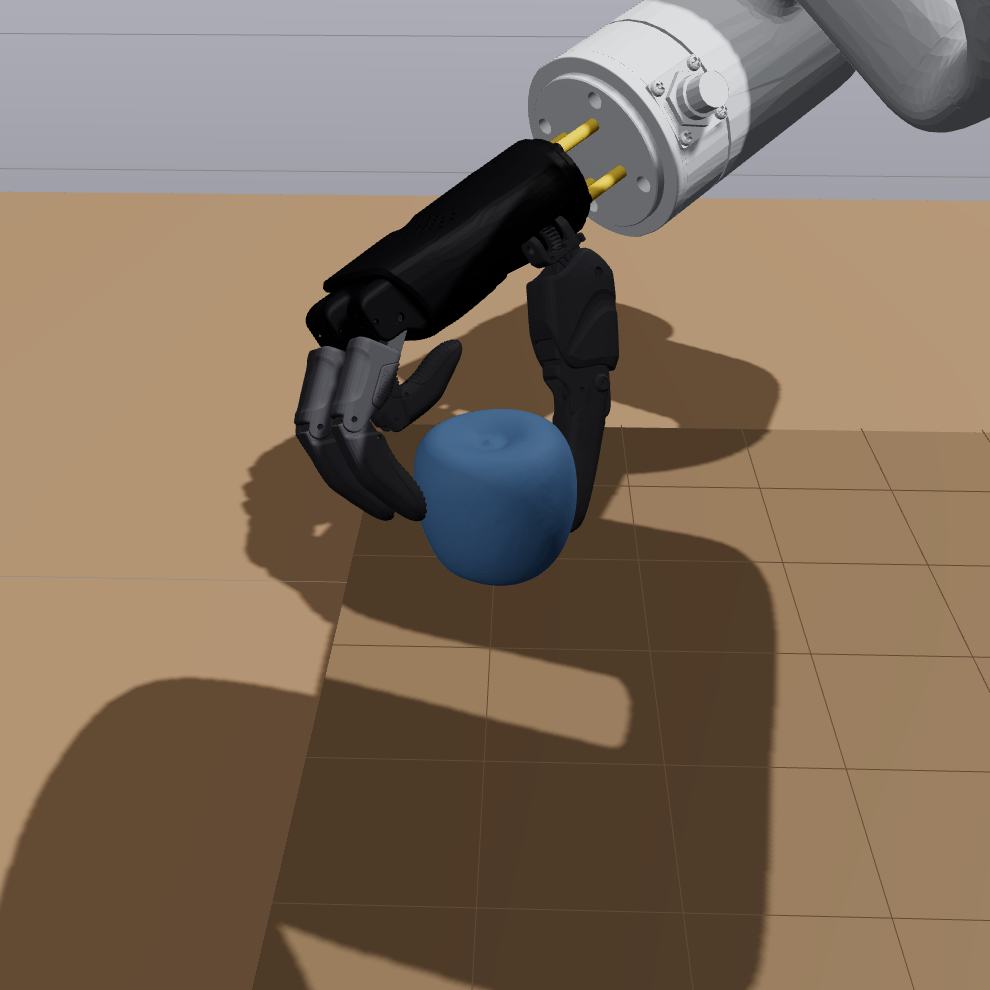}%
	\gcell{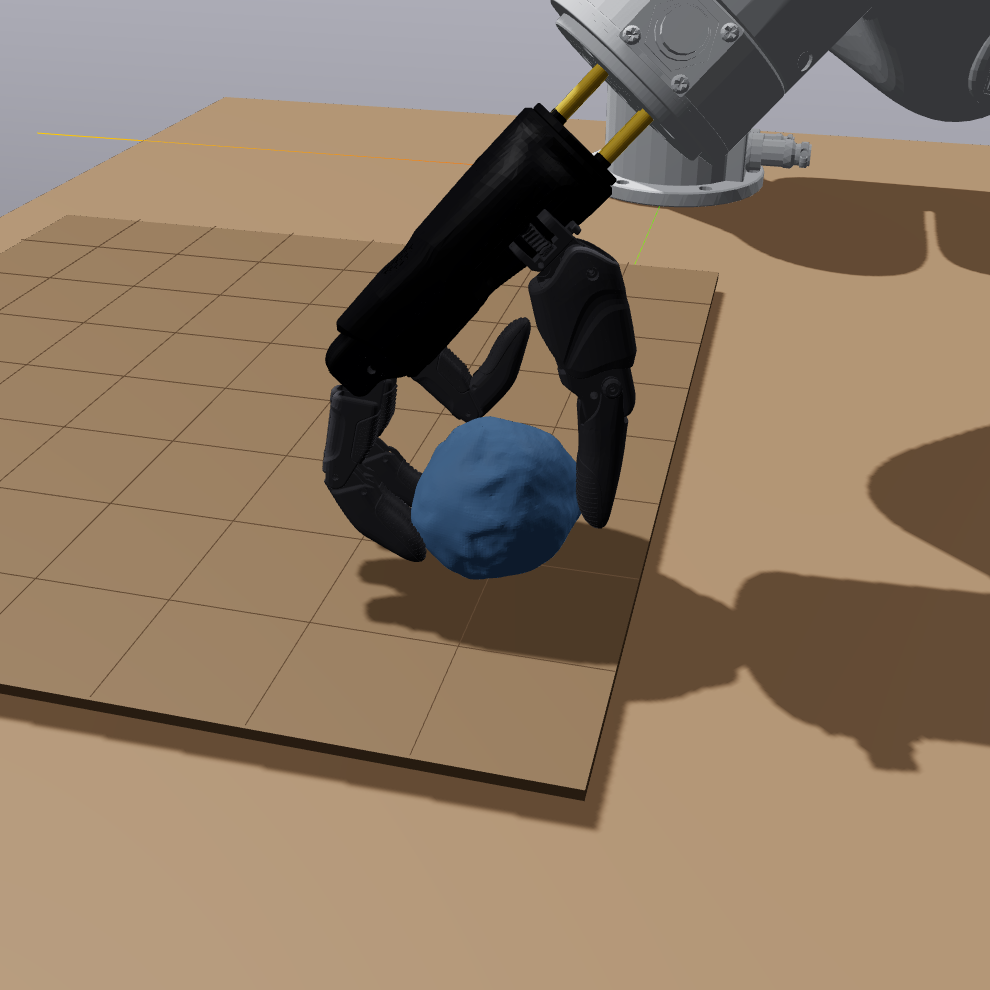}}%
	\caption{\textbf{Friction-Error Robust Grasps on the RealHand L6.} Twelve penetration-free, friction-uncertainty-aware grasps of objects spanning household items (meat and soup cans, mustard and bleach bottles, foam brick),
		primitive shapes (box, cube, cylinder, sphere), and rounded objects (tennis ball,
		apple, strawberry), and exercise enveloping, tripod, and pinch closures on the
		RealHand L6 (rendered in Drake). Every rendered grasp is force-closed (\(\ell^* > 0\)) and satisfies the friction-tail requirement (\(\eps^{(0.99)} > 0\)).}\label{fig:grasp_gallery}
\end{figure*}

\subsection{Cross-Dataset Validation on FRoGGeR and DexGraspNet}\label{ssec:crosscorpus}
To test whether the divergence of \Cref{ssec:divergence} generalizes beyond the LEAP
and Allegro datasets, we evaluate the same analytic signed \(\epsb\) across two public
grasp datasets and a third hand. We compute \(\epsb\) on 800 LEAP and 799 Allegro
grasps synthesized with a FRoGGeR-derived pipeline~\cite{li_frogger_2023}, and on 400
grasps from the DexGraspNet release~\cite{wang_dexgraspnet_2023} on the Shadow Hand,
sampled over 50 core-bottle objects with eight grasps each. The FRoGGeR datasets expose
contact points and inward normals directly. For DexGraspNet we obtain fingertip
positions by forward kinematics on the Shadow Hand model and project them onto the
scaled object mesh.

Specifically, every fraction uses the analytic signed \(\epsb\) at confidence \(\beta = 0.9\) under
the two priors of \Cref{ssec:friction}, with a grasp certified robust when
\(\epsb > 0\), the same certificate the rest of the paper uses. On both FRoGGeR
datasets our \(\eps_\text{nom}\) correlates strongly with FRoGGeR's own Ferrari-Canny
field, \(r = 0.96\) on LEAP and \(r = 0.94\) on Allegro, an independent cross-check of
the two pipelines. \Cref{tab:crosscorpus} reports the certified-robust fractions.
Under the nominal prior both FRoGGeR datasets certify at least 99\% of grasps and
the DexGraspNet subset 33.2\%. Under the adverse prior the LEAP fraction falls to
42.9\%, Allegro to 50.9\%, and DexGraspNet to 3.8\%, with the
per-hand distributions in \Cref{fig:cc_leap,fig:cc_allegro,fig:cc_shadow}. The pattern
matches the aggregate study, classical force closure is overconfident under realistic
friction, and the certified fraction shrinks as the prior shifts mass into the slippery
tail.

\begin{table}[htb]
    \centering
    \caption{Fraction of Grasps Certified Robust (\(\epsb > 0\) at \(\beta{=}0.9\)) Across Three Hand-Dataset Combinations and Two Friction Priors. Salmon-colored rows show the adverse-prior collapse of the certified fraction.}\label{tab:crosscorpus}
    \small
    \begin{tabular}{@{}l l r r r@{}}
        \toprule
        Dataset & Hand & \(N\) & Gaussian Prior & Adverse Prior \\
        \midrule
        FRoGGeR     & LEAP    & 800 & 99.4\%        & \cellcolor{losssalmon}42.9\% \\
        FRoGGeR     & Allegro & 799 & 100.0\%       & \cellcolor{losssalmon}50.9\% \\
        DexGraspNet & Shadow  & 400 & 33.2\%        & \cellcolor{losssalmon}~3.8\% \\
        \bottomrule
    \end{tabular}
\end{table}
In particular, the cross-dataset pattern reflects a property of our metric family
rather than a defect of any synthesizer. Both datasets were produced for purposes
other than risk-sensitive evaluation, and the authors that released them did so on
the basis of metrics (Ferrari-Canny \(\eps\), \(\ell^*\), simulator success
rate) that do not encode the friction prior used here. The presence of the
gap across two independent synthesizers indicates that an explicit-belief
margin like \(\epsb\) adds a real new axis to the family of force-closure
margins rather than restating an existing one.

\subsection{Predictive Validity Under Dynamic Perturbation}\label{ssec:predictive}
The analytic divergence of \Cref{ssec:divergence} predicts that a grasp the nominal
metric rates highly can fail once the contact friction departs from its assumed value.
We test this prediction in dynamic simulation with the six-axis shake protocol of GenDexGrasp~\cite{li2023gendexgrasp, zhong_gagrasp_2025}. In this test, the risk margin ranks realized robustness more accurately than both fixed-friction baselines.

\subsubsection{Shake-Based Simulation Stress Test}\label{sssec:gagraspsimtest}
We stress each certified grasp in the Drake physics-validation system. The FRoGGeR
collision proxy recesses behind the visual fingertip, so we weld a small contact pad
at every certified contact point and confirm that the grasp engages the object at the
synthesized configuration. We then apply a gravity-off six-axis shake at a fixed
contact friction and record how many of the six world-axis directions retain the
object. We repeat the shake over the adverse friction sweep
\(\mu \in \{0.2, 0.3, 0.4\}\) at the nominal friction \(\mu_\text{nom}=0.7\). The
object body takes on the mesh-and-density mass and moment of inertia, so the
disturbance force scales per object and the rotation criterion uses the true
inertia rather than a unit surrogate. Of the 641 certified grasps across the
seven objects, 302 retain the object in at least one direction at the nominal
friction, and we restrict the friction analysis to these, since a grasp that
never engages cannot test friction sensitivity.

\subsubsection{Realized Robustness Versus the Risk Margin}\label{sssec:realizedrobust}
\begin{figure}[htb]
	\centering
	\includegraphics[width=\linewidth]{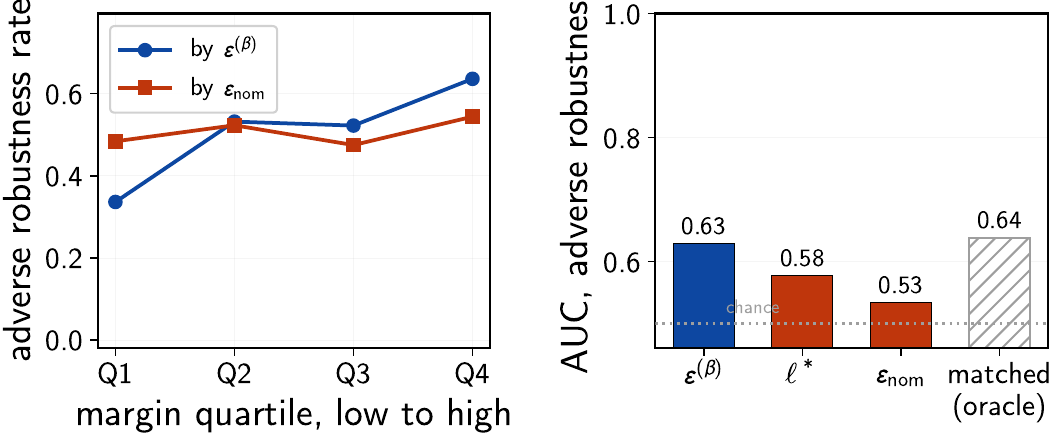}
	\caption{\textbf{Predictive Validity Under Dynamic Perturbation.} Realized
		robustness of 302 established-contact grasps under a gravity-off six-axis
		shake over a friction sweep \(\mu \in \{0.2, 0.3, 0.4\}\), pooled across the
		seven objects, at each object's mesh-and-density mass and moment of
		inertia. \emph{Left}: the mean fraction of shake directions that retain the
		object, binned by predictor quartile, rises with the risk margin \(\epsb\)
		but saturates under the nominal Ferrari-Canny margin \(\eps_\text{nom}\).
		\emph{Right}: the area under the receiver operating characteristic curve (AUC) for
		the adverse-friction outcome, where \(\epsb\), computed only from the friction
		prior and never from the realized friction, ranks realized robustness above
		\(\eps_\text{nom}\) and the FRoGGeR min-weight \(\ell^*\), neither of which
		models friction volatility. The hatched bar marks a matched-friction oracle
		with access to the realized friction, an upper bound, and the dotted line marks the
		chance level, an AUC of \(0.5\), where a predictor ranks robustness no better
		than a coin flip.}\label{fig:predictive_validity}
\end{figure}
\Cref{fig:predictive_validity} ranks realized robustness against three
predictors, none of which has access to the realized friction. The risk margin \(\epsb\) (\Cref{def:rsepsmet})
uses the adverse prior of \Cref{ssec:friction} through its CVaR friction
\(v_\beta \approx 0.12\), so it scores each grasp once without knowledge of the
operating friction. The nominal Ferrari-Canny margin
\(\eps_\text{nom}=\eps(g,\mu_\text{nom})\) evaluates the same functional at the assumed
friction and does not model the friction uncertainty, and the FRoGGeR min-weight
\(\ell^*\) is the third predictor. By the area under the receiver operating
characteristic curve (AUC) for the adverse-friction outcome, \(\epsb\) reaches
\(0.629\) against \(0.578\) for \(\ell^*\) and \(0.534\) for \(\eps_\text{nom}\), and it
ranks above the min-weight in \(97.9\%\) and above the nominal margin in \(100\%\) of
grasp-level paired bootstrap resamples, each gap with a 95\% interval clear of zero.
The per-grasp Spearman correlation against the graded directions-retained rate orders
the predictors the same way, \(0.25\) for \(\epsb\) against \(0.14\) for \(\ell^*\) and
\(0.05\) for \(\eps_\text{nom}\). Realized robustness climbs across the \(\epsb\)
quartiles, from \(0.34\) to \(0.64\), whereas under \(\eps_\text{nom}\) it stays near
\(0.50\) across the four bins, so the nominal margin separates neither the least nor
the most robust grasps well. A matched-friction oracle with access to the realized
friction reaches AUC \(0.639\), and \(\epsb\) recovers most of the corresponding separation without
the friction knowledge.

\subsubsection{Weight-Bearing Retention Under a Gravity-On Lift}\label{sssec:pickcompanion}
\begin{figure}[htb]
	\centering
	\includegraphics[width=\linewidth]{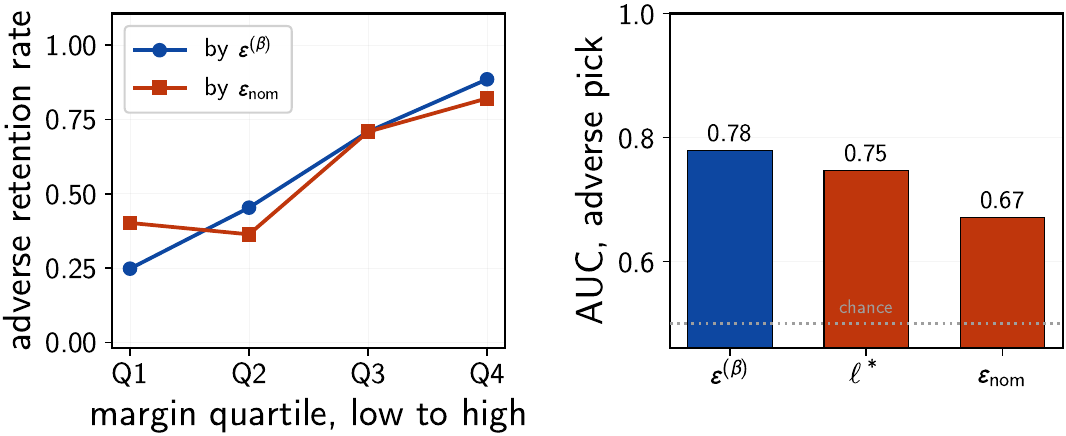}
	\caption{\textbf{Pick Companion Under a Gravity-On Lift.} Realized
		weight-bearing retention of 51 established-contact grasps under a
		gravity-on prismatic lift with a graded lateral pull at \(1\), \(3\),
		\(6\) times object weight, over the friction sweep
		\(\mu \in \{0.2, 0.3, 0.4\}\). \emph{Left}: the success rate binned by
		predictor quartile climbs with \(\epsb\) and saturates under the
		nominal Ferrari-Canny margin \(\eps_\text{nom}\). \emph{Right}: the pooled
		area under the receiver operating characteristic curve (AUC) for
		the binary adverse pick, where \(\epsb\) ranks
		above the min-weight \(\ell^*\) and the nominal margin
		\(\eps_\text{nom}\) without access to the realized friction, and the dotted
		line marks the chance level, an AUC of \(0.5\).}\label{fig:pick_companion}
\end{figure}
The shake measures wrench-space quality once contact is established. To probe
weight-bearing retention under the same friction sweep, we complement the shake with
a gravity-on prismatic lift on the same seven objects. The hand raises the object
along a vertical mount, and the pick then applies a lateral pull at graded
multiples of object weight while the grip holds under Coulomb friction. The pick
grades realized retention by the highest weight multiple the grasp holds before slip,
and we rank the resulting outcome against the same three predictors used for the shake. Of
the 105 lifts attempted, 51 establish contact under gravity across the seven objects,
and we restrict the retention analysis to these. \Cref{fig:pick_companion} ranks the
three predictors against the adverse pick. The risk margin \(\epsb\) reaches AUC
\(0.78\) against \(0.75\) for the min-weight \(\ell^*\) and \(0.67\) for the nominal
margin \(\eps_\text{nom}\). The success rate climbs from \(0.25\) to \(0.88\)
across the \(\epsb\) quartiles, while under \(\eps_\text{nom}\) the low quartile
already sits at \(0.40\) and \(\eps_\text{nom}\) fails to separate the least robust picks.
The lead of \(\epsb\) over \(\eps_\text{nom}\) is decisive, and the ranking against
\(\ell^*\) is directional only, not significant, at the seven-object cluster count. The two dynamic
axes rank the grasps consistently against the nominal margin, and the pick extends
the evidence into a weight-bearing regime the free-space shake does not model.
\Cref{tab:predictive} places the two axes side by side.

\begin{table}[htb]
	\centering
	\caption{Predictive Validity Across Two Dynamic Axes. Pooled ROC
		AUC and per-grasp Spearman \(\rho\) for each predictor against the
		adverse-friction outcome, on the gravity-off shake (\(n=302\)
		established-contact grasps) and the gravity-on pick (\(n=51\)
		established-contact grasps), over the friction sweep \(\mu \in \{0.2, 0.3, 0.4\}\). The shaded row is our risk margin, which
		leads both dynamic axes on both statistics all without access to the realized
		friction.}\label{tab:predictive}
	\small
	\setlength{\tabcolsep}{6pt}
	\begin{tabular}{@{}l c c c c@{}}
		\toprule
		 & \multicolumn{2}{c}{Shake, \(n=302\)} & \multicolumn{2}{c}{Pick, \(n=51\)} \\
		\cmidrule(lr){2-3} \cmidrule(l){4-5}
		Predictor & AUC \(\uparrow\) & \(\rho\) \(\uparrow\) & AUC \(\uparrow\) & \(\rho\) \(\uparrow\) \\
		\midrule
		\rowcolor{oursmint}
		\(\epsb\)                    & \textbf{0.63} & \textbf{0.25} & \textbf{0.78} & \textbf{0.55} \\
		\(\ell^*\)                   & 0.58 & 0.14 & 0.75 & 0.50 \\
		\(\eps_\text{nom}\)          & 0.53 & 0.05 & 0.67 & 0.33 \\
		\midrule
		Matched oracle               & 0.64 & n/a & n/a & n/a \\
		\bottomrule
	\end{tabular}
\end{table}

\ifpreprint
In addition, a larger re-run reproduces the ordering on both dynamic axes. The
gravity-off shake over 329 established-contact grasps scores \(\epsb\) at AUC
\(0.635\) against \(0.590\) for \(\ell^*\) and \(0.561\) for \(\eps_\text{nom}\),
and the gravity-on pick over 321 established-contact grasps scores \(0.744\)
against \(0.714\) and \(0.641\), so \(\epsb\) leads at the larger sample as well.
\fi

\subsection{Discussion}\label{ssec:discus}
The aggregate results expose a mismatch between nominal and distributional grasp
assessment. Across the 1{,}599 LEAP and Allegro grasps the nominal Ferrari-Canny
margin \(\eps_\text{nom}\) correlates with the adverse-prior \(\epsb\) at only 0.32,
down from 0.95 under the tight nominal prior (\Cref{ssec:divergence}), so the two
metrics agree when friction is well characterized but differ when friction is uncertain.
The divergence is geometric rather than a synthesizer artifact, since it tracks object
shape. Rounded objects such as the cylinder and tennis ball stay certified across the
adverse tail (\(\beta_c = 1.0\)), while irregular objects concentrate the friction
sensitivity, the mustard bottle losing closure at \(\beta_c = 0.46\)
(\Cref{tab:main_results}). A grasp that \(\eps_\text{nom}\) rates as adequate can therefore hold most of its wrench
capacity in the normal directions and collapse along the friction-dependent ones, a
failure mode the scalar nominal margin cannot express but the risk-adjusted margin
\(\epsb\) captures.

In addition, our risk objective extends the same distinction into synthesis. \Cref{tab:synthesis} shows
that FRoGGeR run under the differentiable risk objective lowers the friction-sensitive
fraction from 67.7\% to 57.6\% at matched budget, even though our objective targets the
mild Gaussian prior while the evaluation applies the harsher adverse prior
(\Cref{ssec:synth_objective}). The selected candidates favor contact configurations
whose margin degrades gracefully as friction drops, so \(\epsb\) serves as an
evaluation metric and a synthesis signal from one linear program.

\subsection{Limitations}\label{ssec:limits}
The risk margin \(\epsb\) orders realized dynamic robustness above both friction-unaware
baselines and reproduces the friction-sensitivity pattern across three hands and a
public dataset. The present evaluation has several limitations that future work can address. Specifically, the predictive-validity study of \Cref{ssec:predictive} spans seven objects, which limits an object-level grasp failure prediction analysis. Furthermore, the pooled AUC lead of \(\epsb\) stays stable at 0.629 against 0.578 and 0.534, and a grasp-level paired bootstrap places its 95\% interval clear of zero, yet a resampling that groups grasps by object has less power at seven object clusters, so we take the per-object generalization to be directional rather than absolute.

Because a static wrench-space margin only partly predicts a dynamic shake outcome, the absolute discrimination remains moderate at AUC 0.629. We therefore interpret this result as a comparison among metrics rather than as a calibrated closure probability. The shake test welds a contact pad at each certified contact and runs without gravity, so it measures wrench-space quality after the hand establishes contact. It does not measure whether the robot can achieve the grasp or support the object's weight. The pick study of \Cref{sssec:pickcompanion} adds a weight-bearing evaluation on the same seven objects. Because both studies contain only seven object clusters, we treat the consistency across the two evaluations as directional evidence rather than as proof of per-object generalization.

\section{Conclusion}\label{sec:conc}
We derived a family of risk-sensitive grasp quality metrics grounded in CVaR and
established three theoretical properties, monotonicity in the confidence level
(Theorem~\ref{thm:monotone}), differentiability with respect to grasp parameters
(Theorem~\ref{thm:diff}), and a probabilistic closure certificate
(Theorem~\ref{thm:certificate}). Across the aggregate, cross-dataset, dynamic shake,
and synthesis studies, $\epsb$ consistently identifies friction-sensitive grasps
that the nominal Ferrari-Canny margin overrates, reproduces the pattern across a
third hand and a public dataset, orders realized adverse robustness above both
baselines, and lowers the friction-sensitive fraction of synthesized grasps at
matched budget
(\Cref{ssec:divergence,ssec:crosscorpus,ssec:predictive,ssec:synth_objective}), so
These results support the use of $\epsb$ both as an evaluation metric and as a synthesis objective.

Beyond this, several directions for future work remain open. The differentiability result of
Theorem~\ref{thm:diff} provides a foundation for gradient-based synthesis via
backpropagation through a differentiable contact simulator. The present sample-based
approach does not exploit the resulting capability. Integration with learned contact models
could extend the uncertainty distribution beyond parametric friction variation.
Finally, hardware validation with real tactile feedback under natural contact
uncertainty would confirm whether the risk-sensitive certificates transfer from
simulation to physical grasping, a step we leave to future work.

  \bibliographystyle{ieeetr}
\bibliography{references}

\begin{figure*}
	\centering
	\newlength{\leapcellw}%
	\setlength{\leapcellw}{\dimexpr.1225\textwidth-2\fboxrule\relax}%
	{\setlength{\fboxsep}{0pt}%
	\newcommand{\lcell}[1]{\fbox{\includegraphics[width=\leapcellw]{figures/leap_zarm/#1}}}%
	\lcell{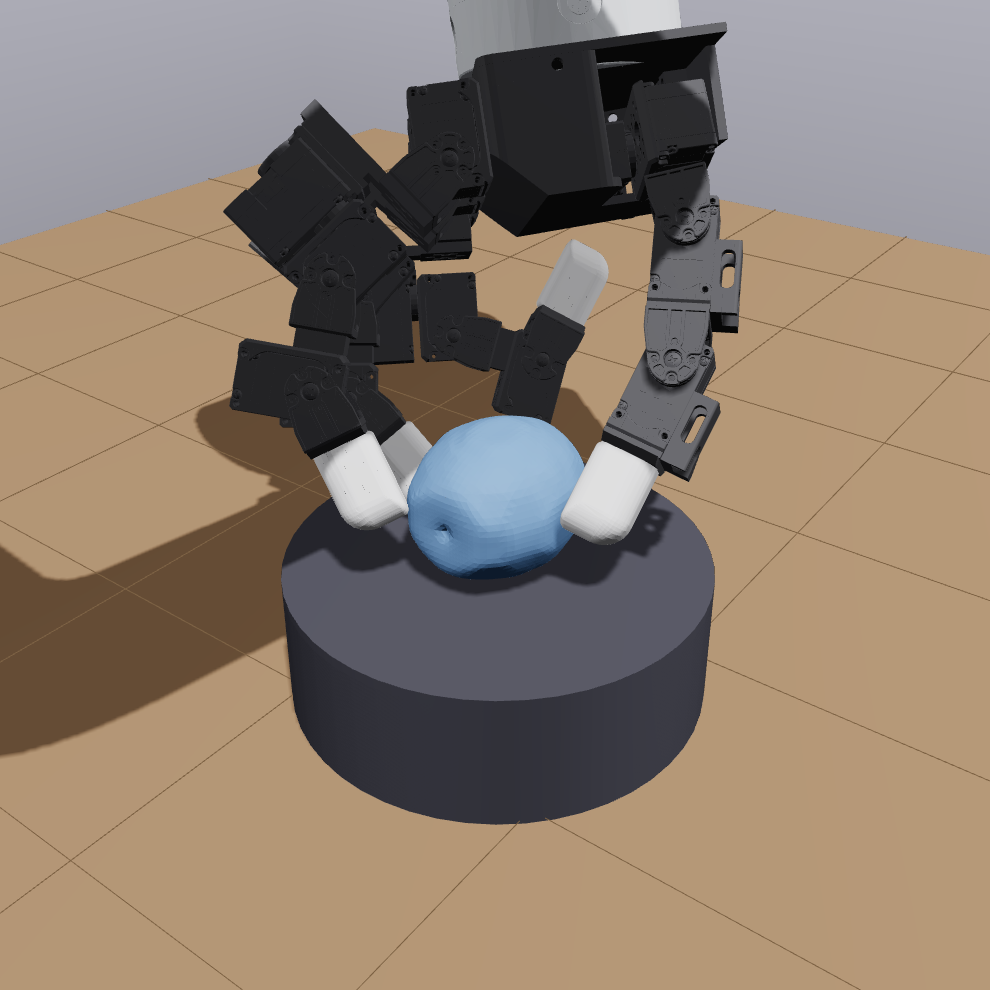}%
	\lcell{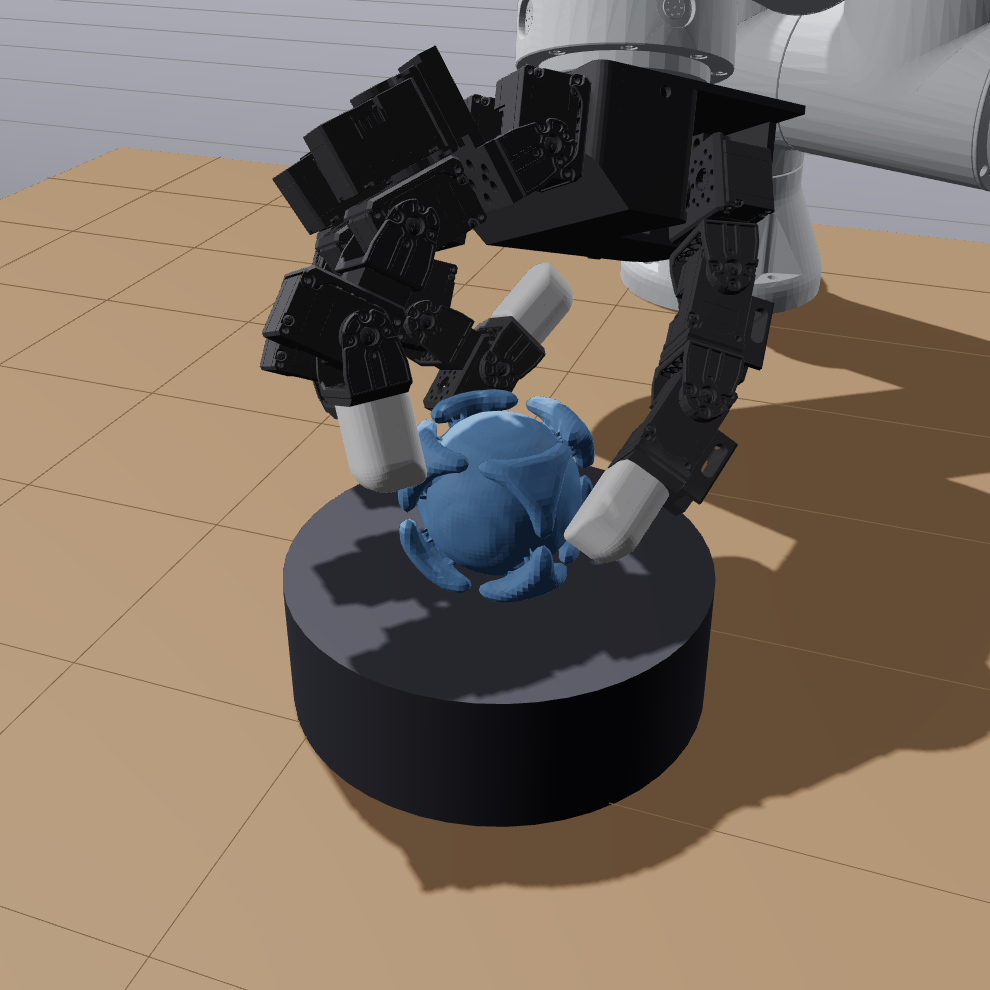}%
	\lcell{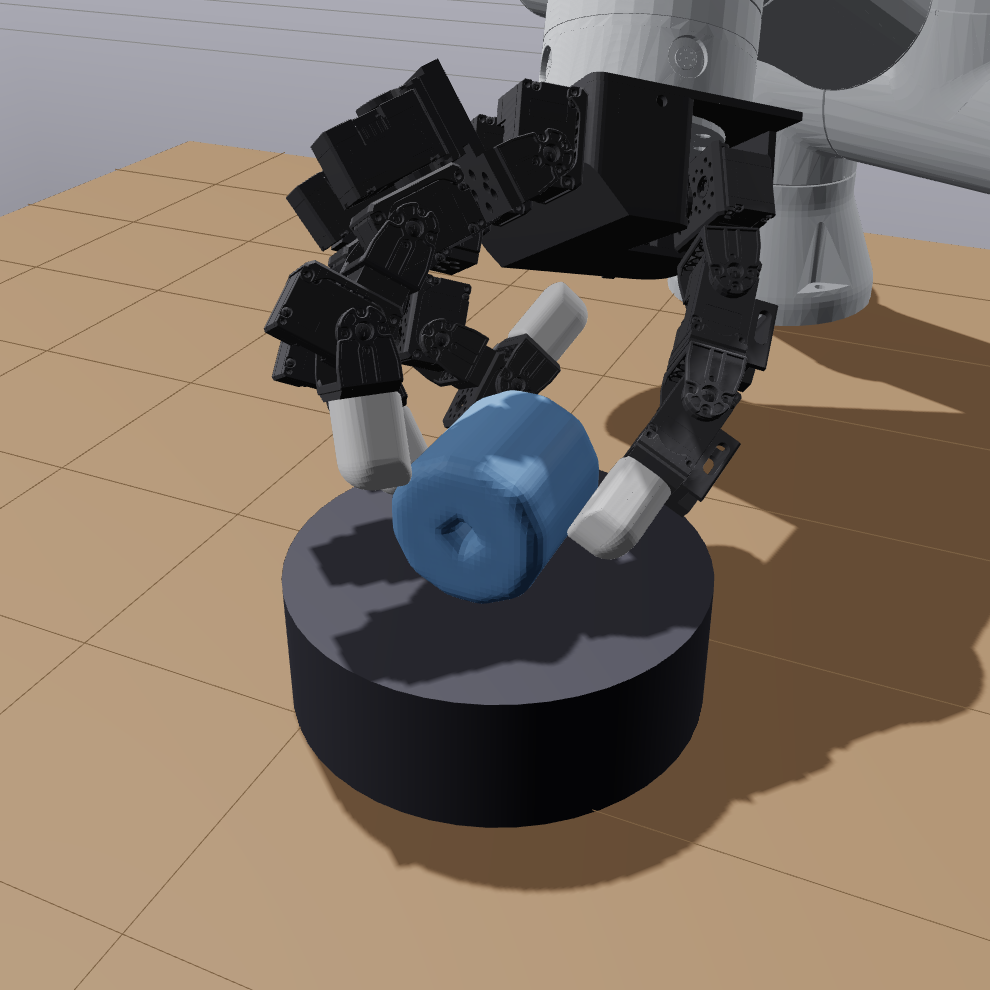}%
	\lcell{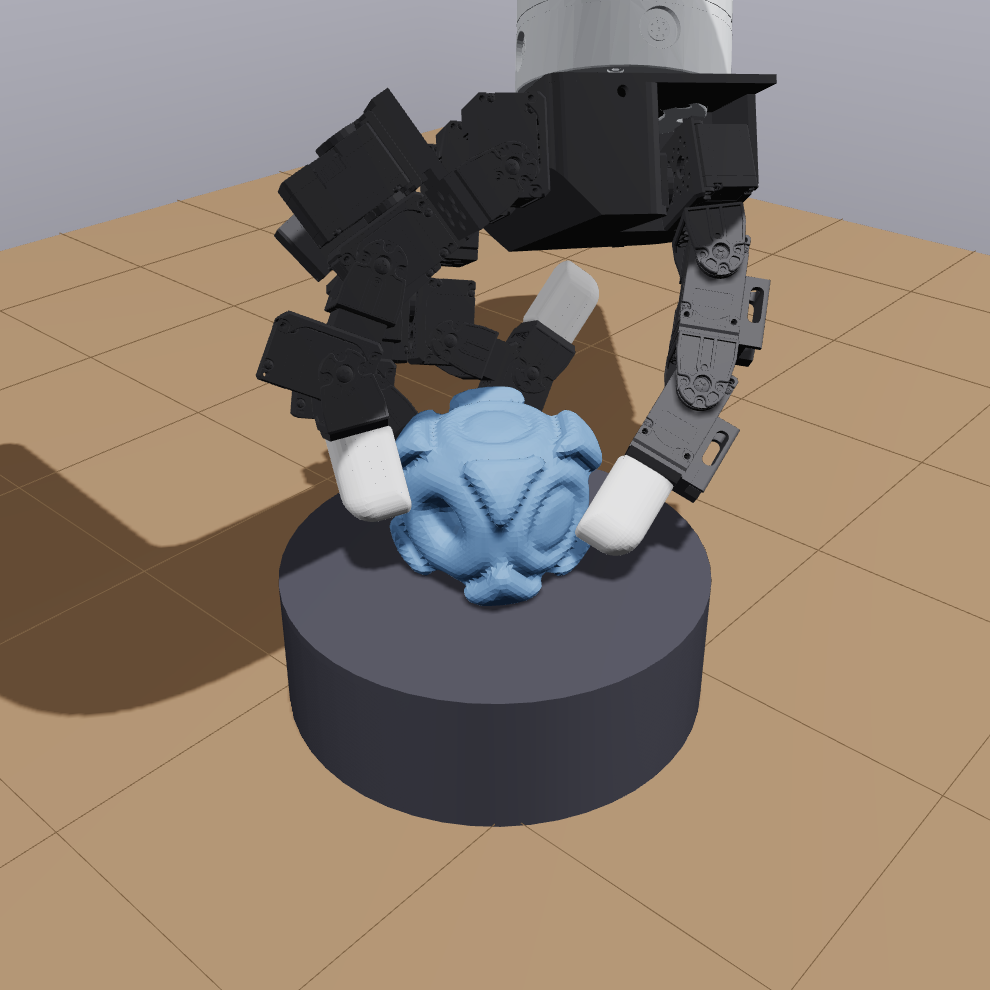}%
	\lcell{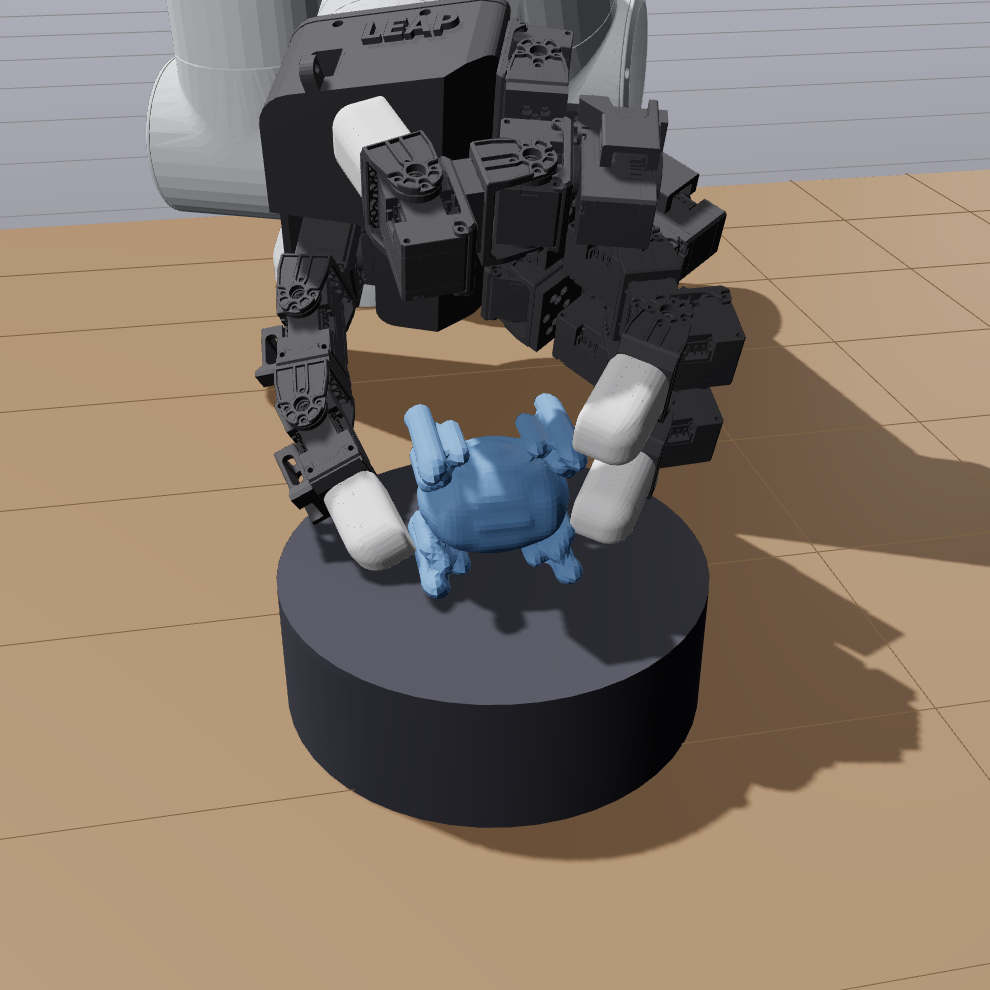}%
	\lcell{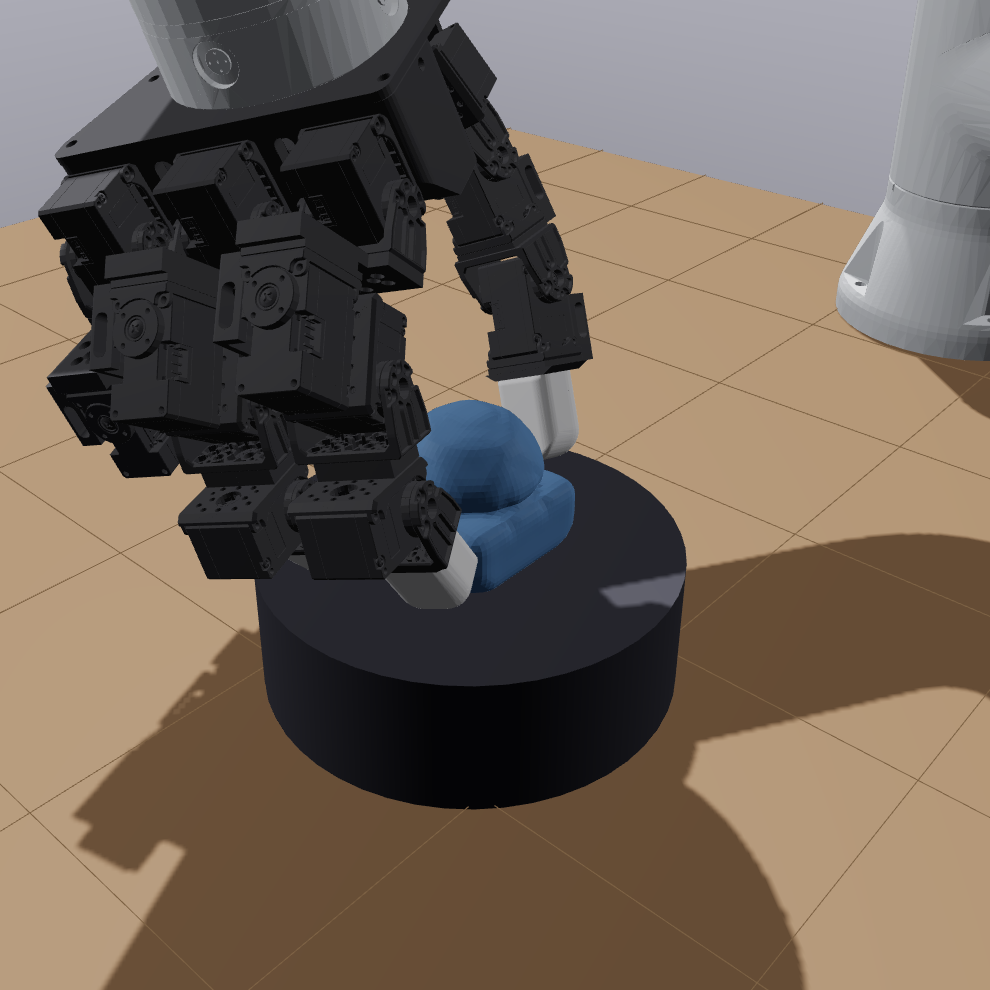}%
	\lcell{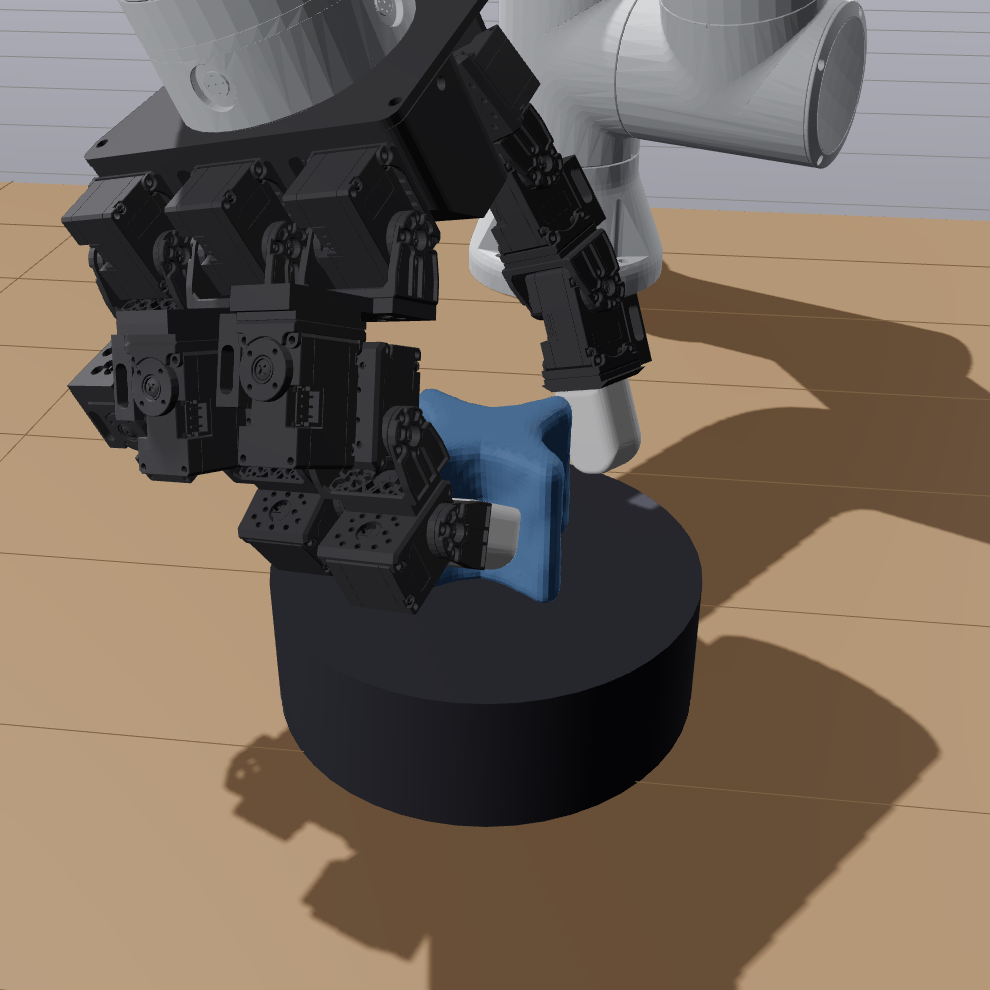}%
	\lcell{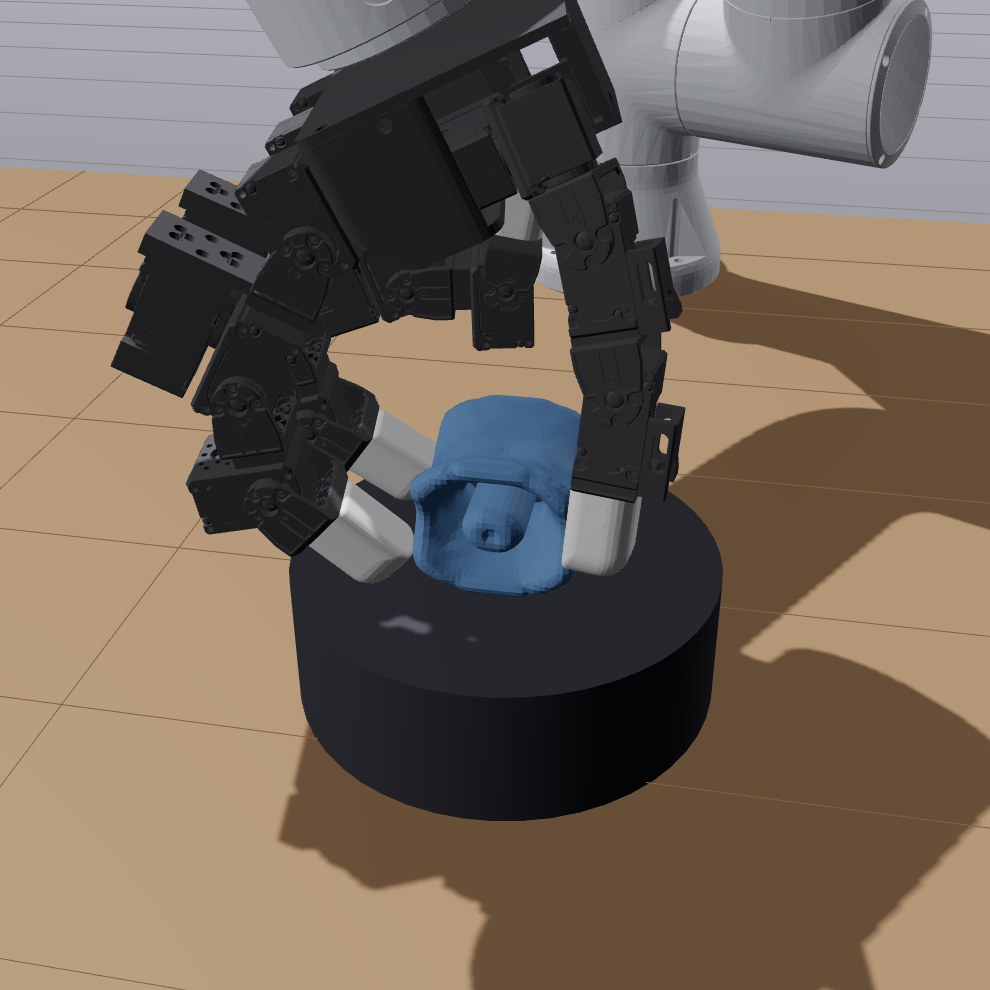}%
	\\[0pt]
	\lcell{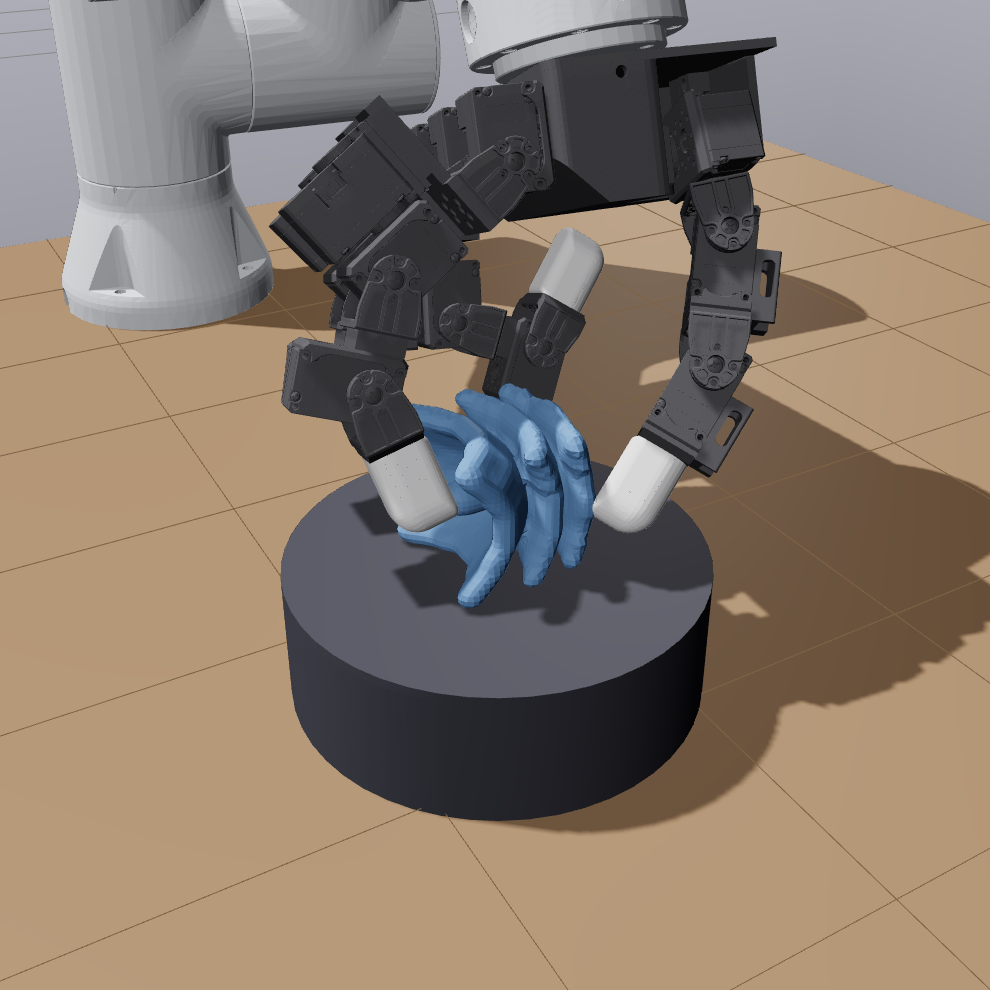}%
	\lcell{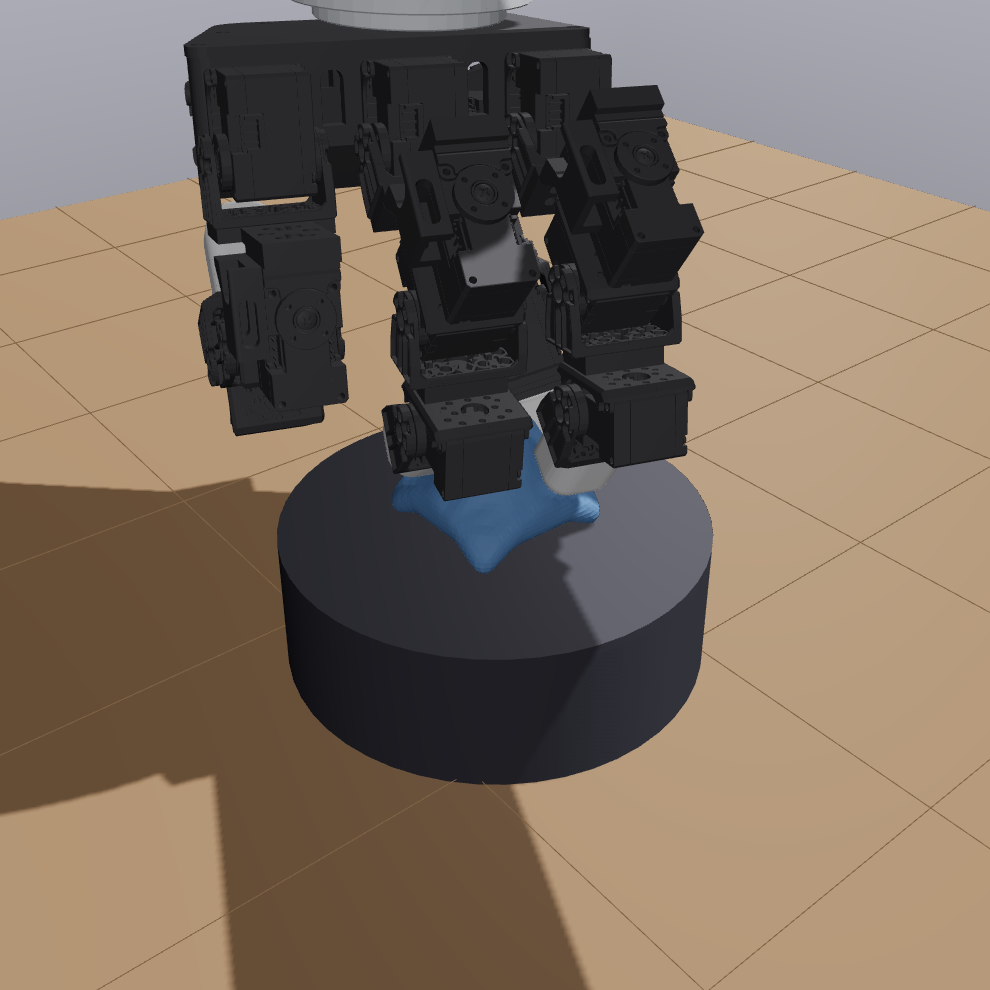}%
	\lcell{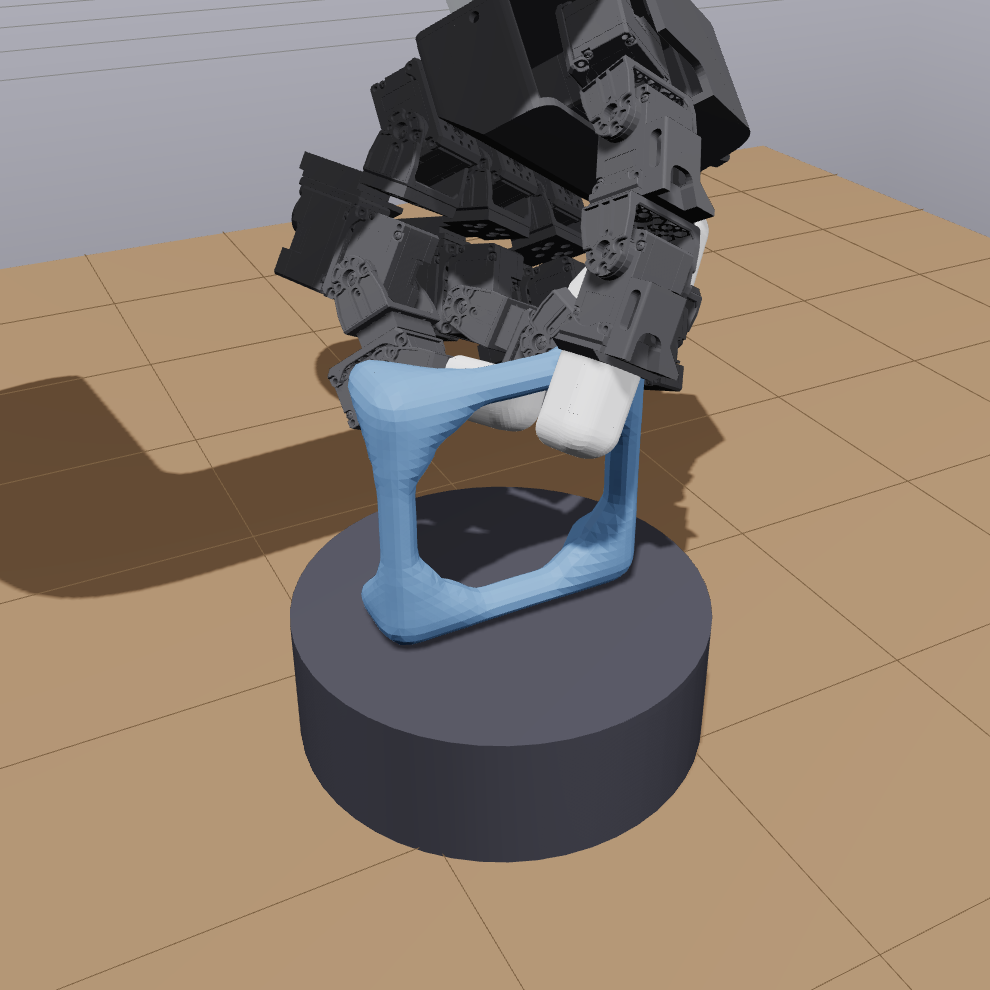}%
	\lcell{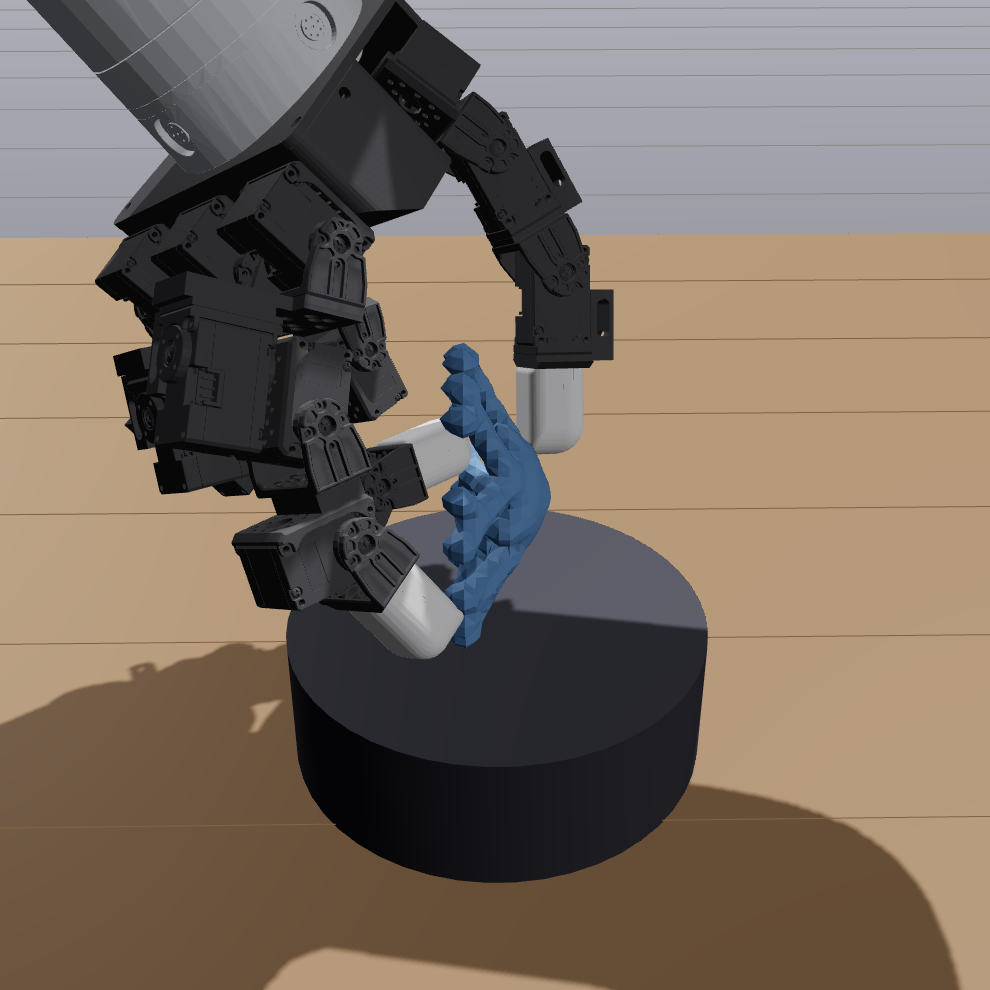}%
	\lcell{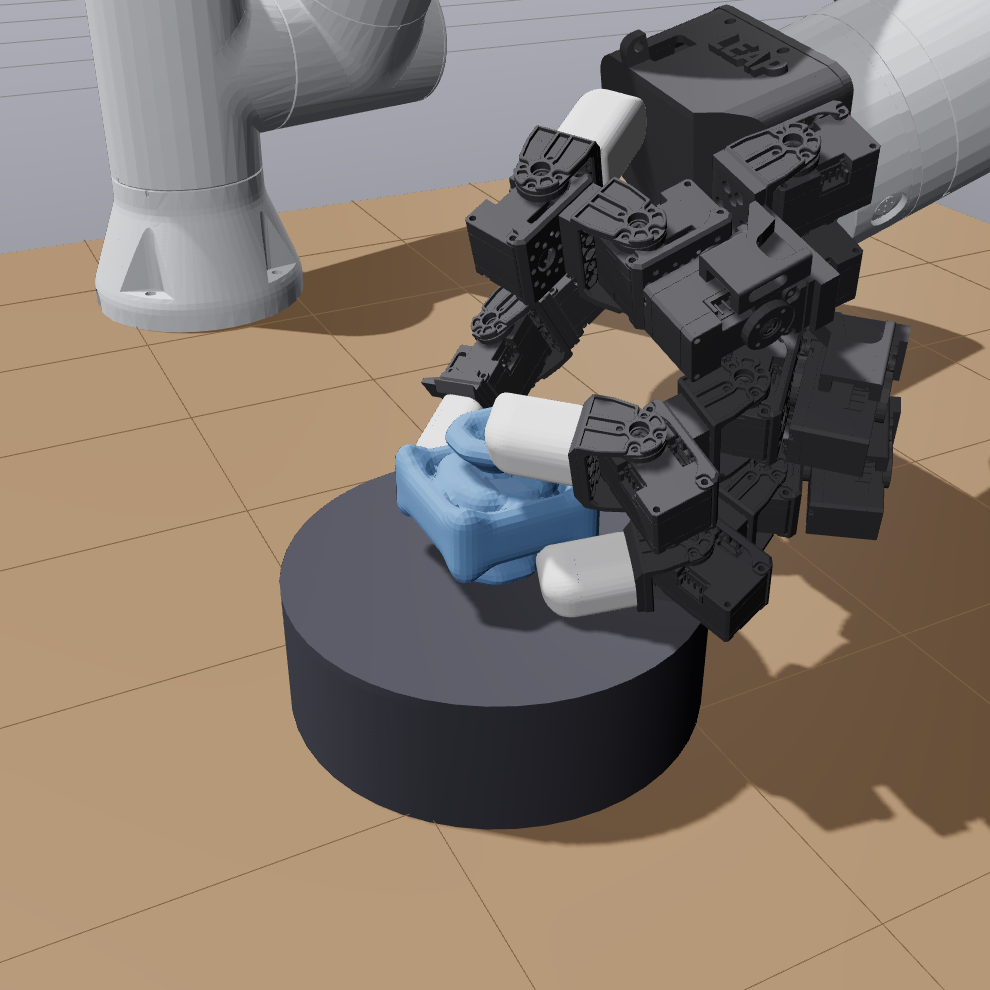}%
	\lcell{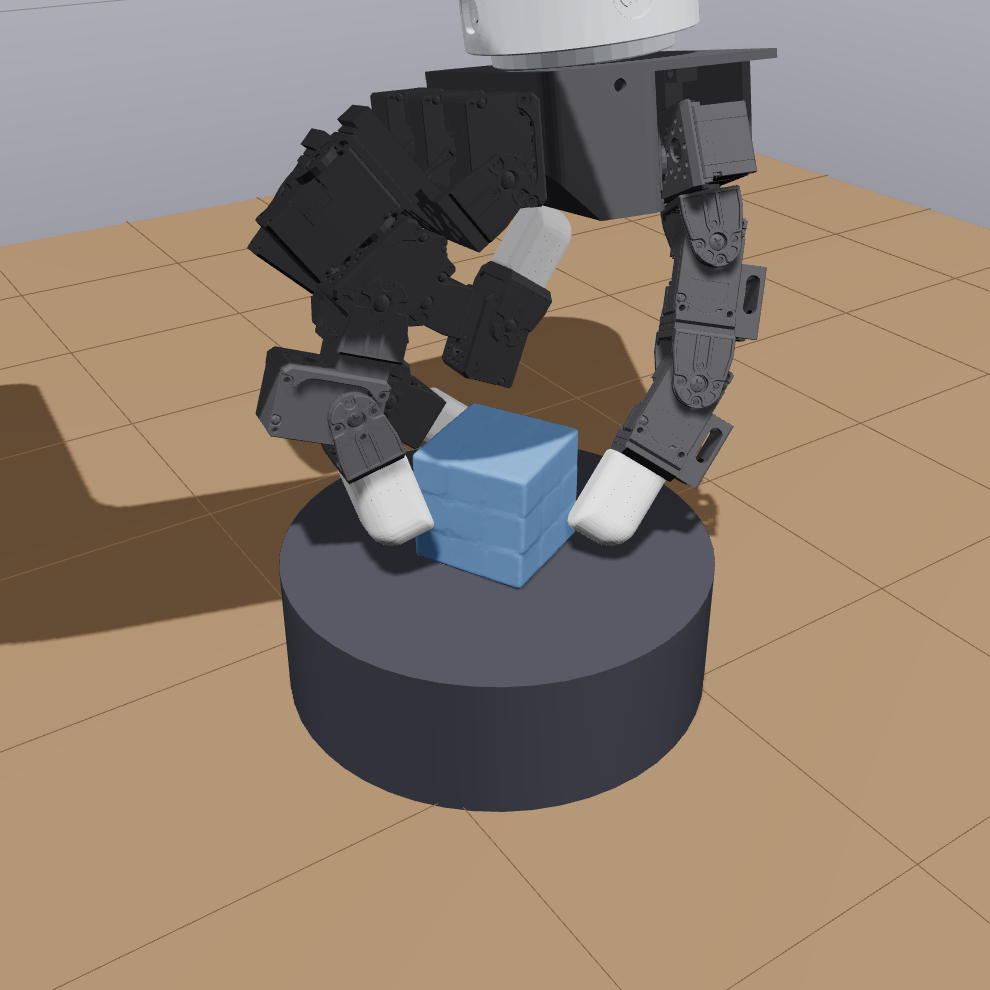}%
	\lcell{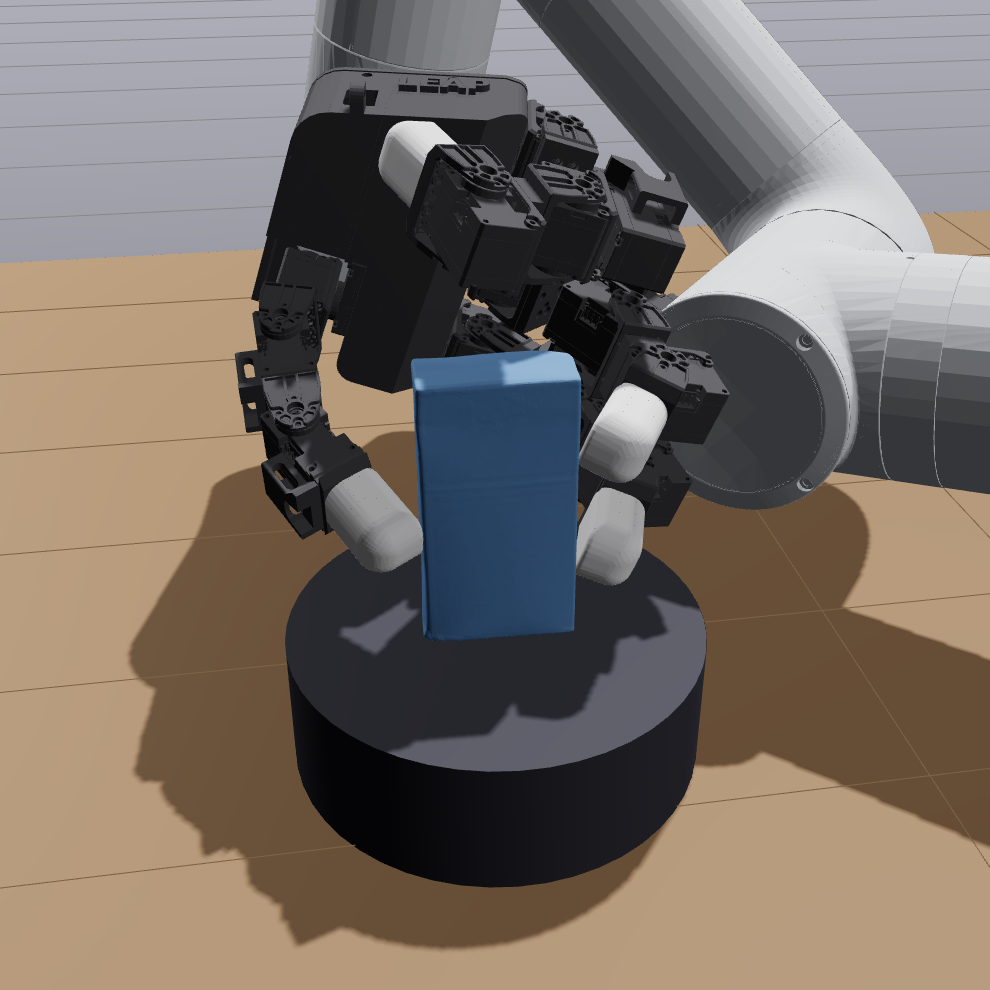}%
	\lcell{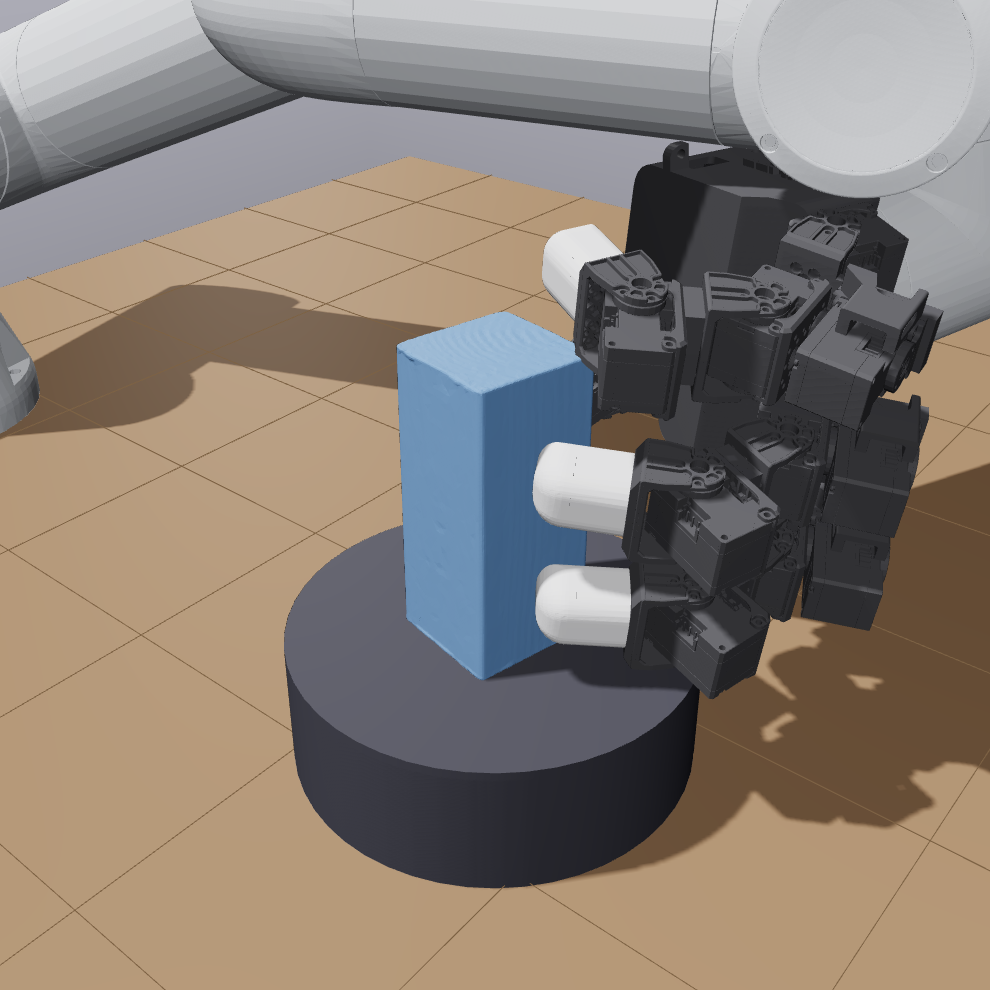}%
	\\[0pt]
	\lcell{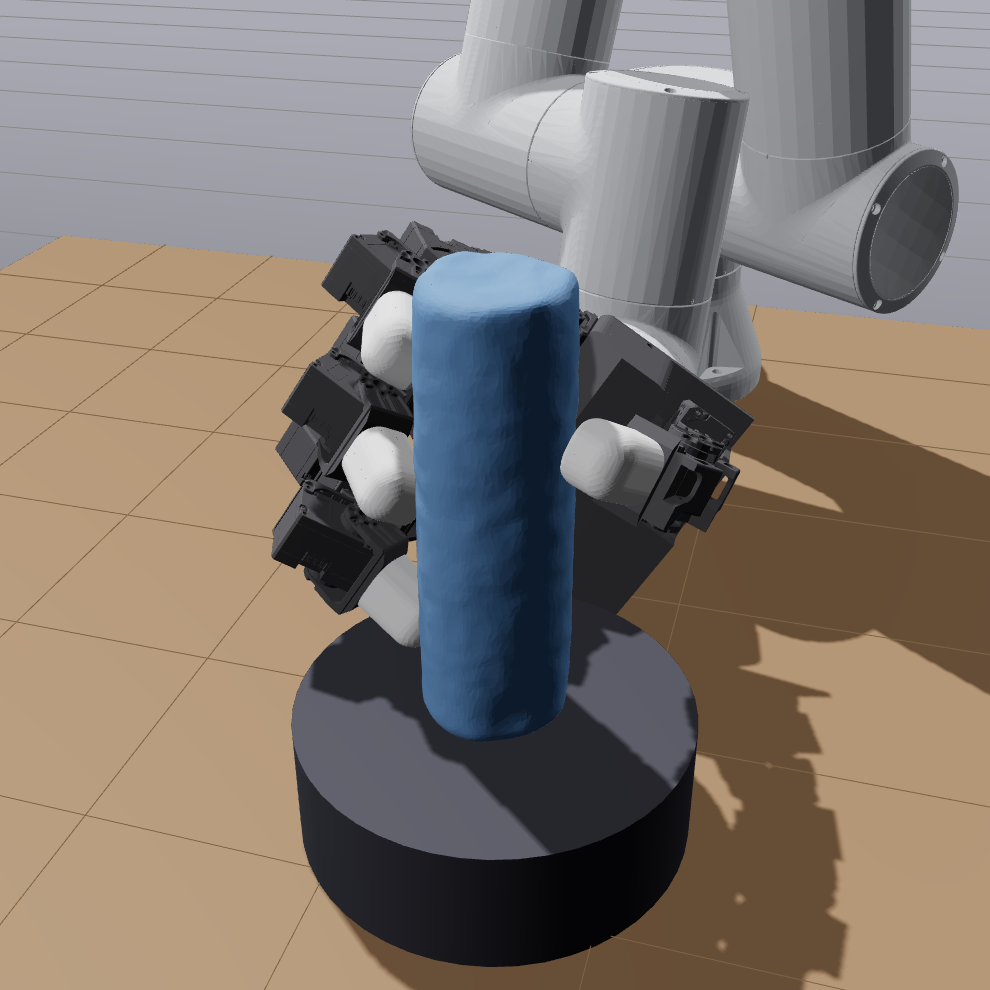}%
	\lcell{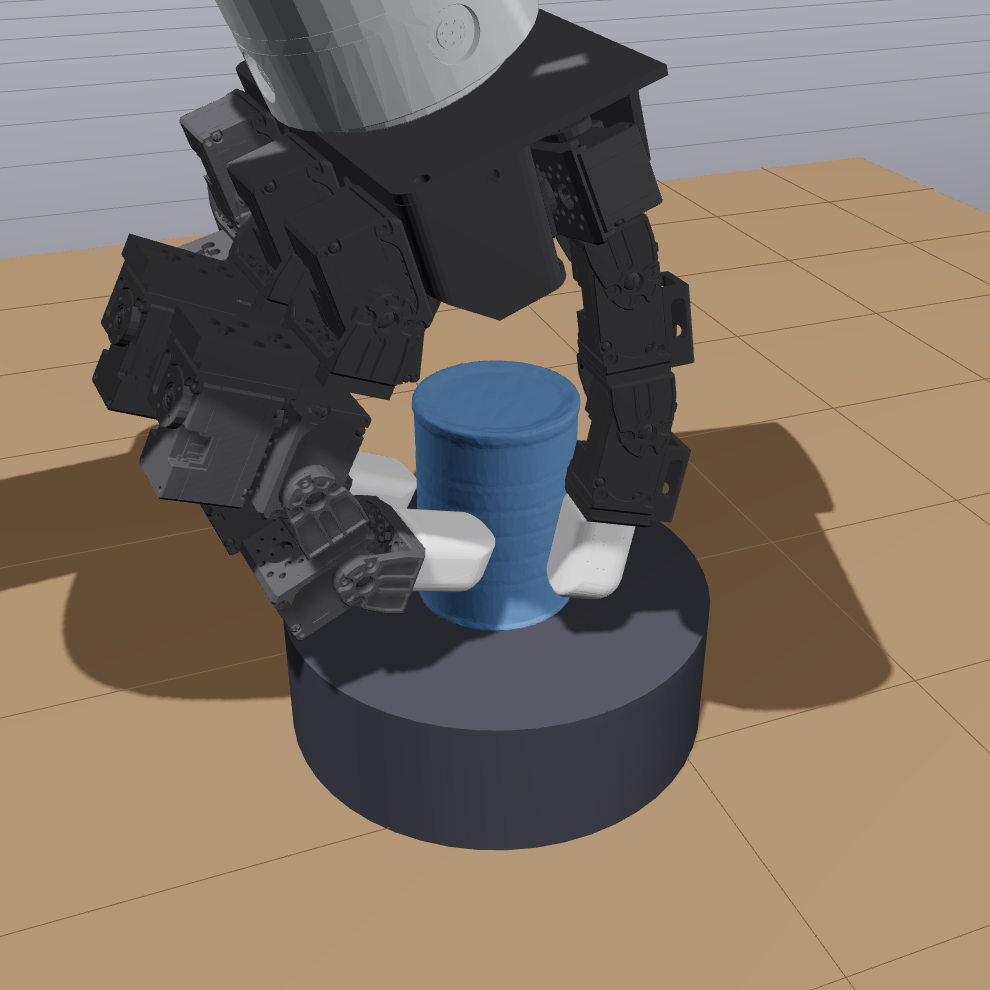}%
	\lcell{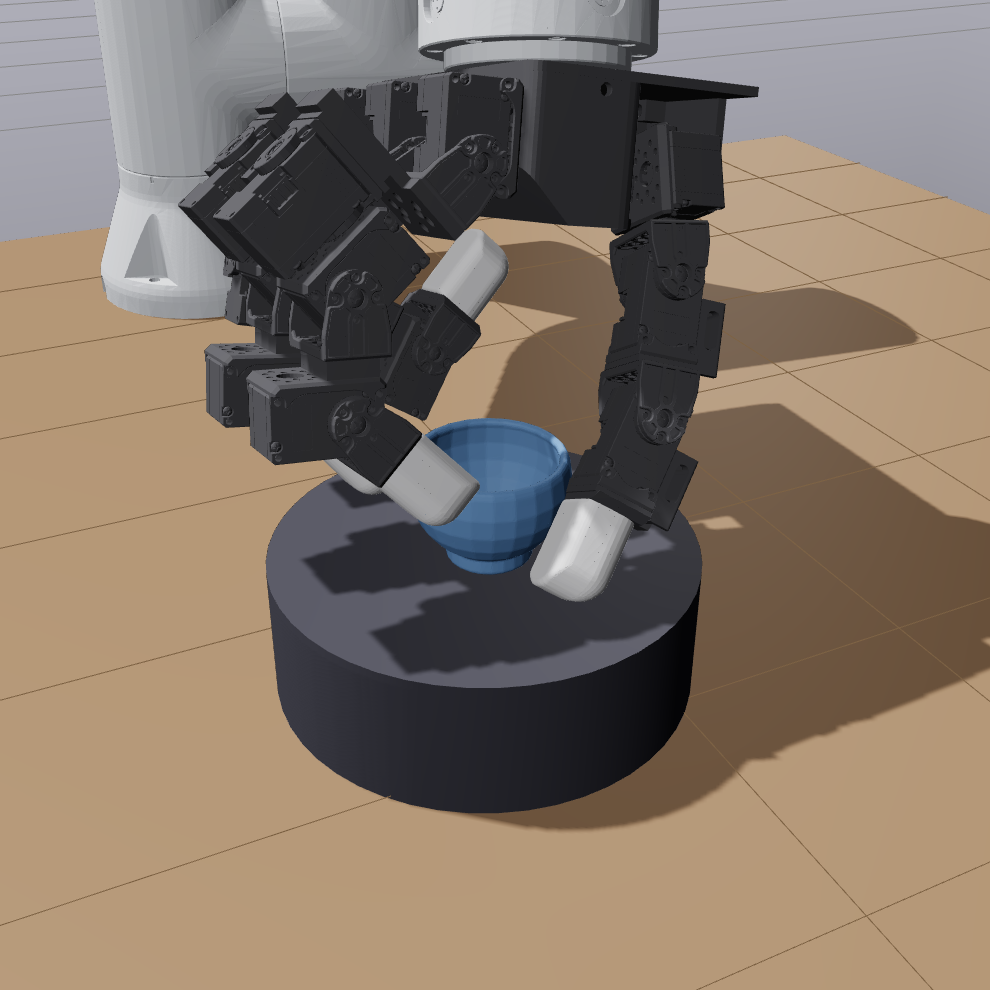}%
	\lcell{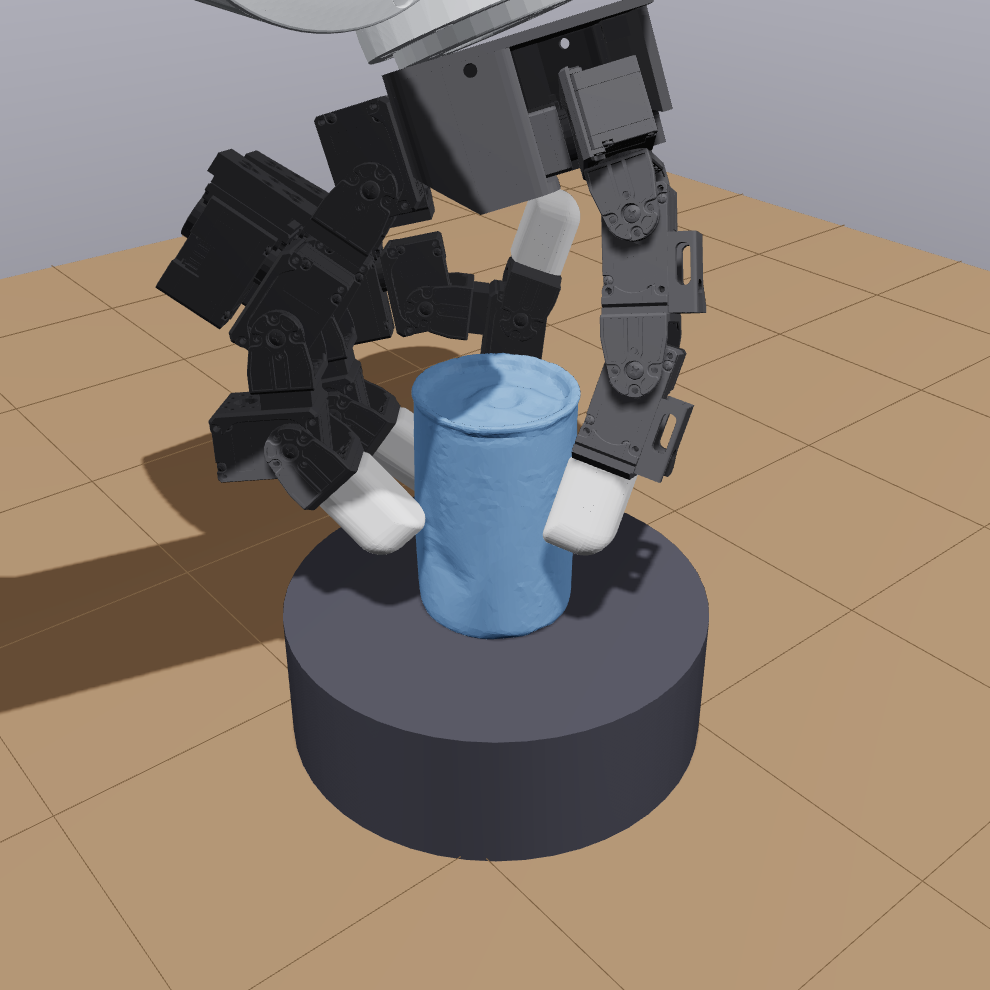}%
	\lcell{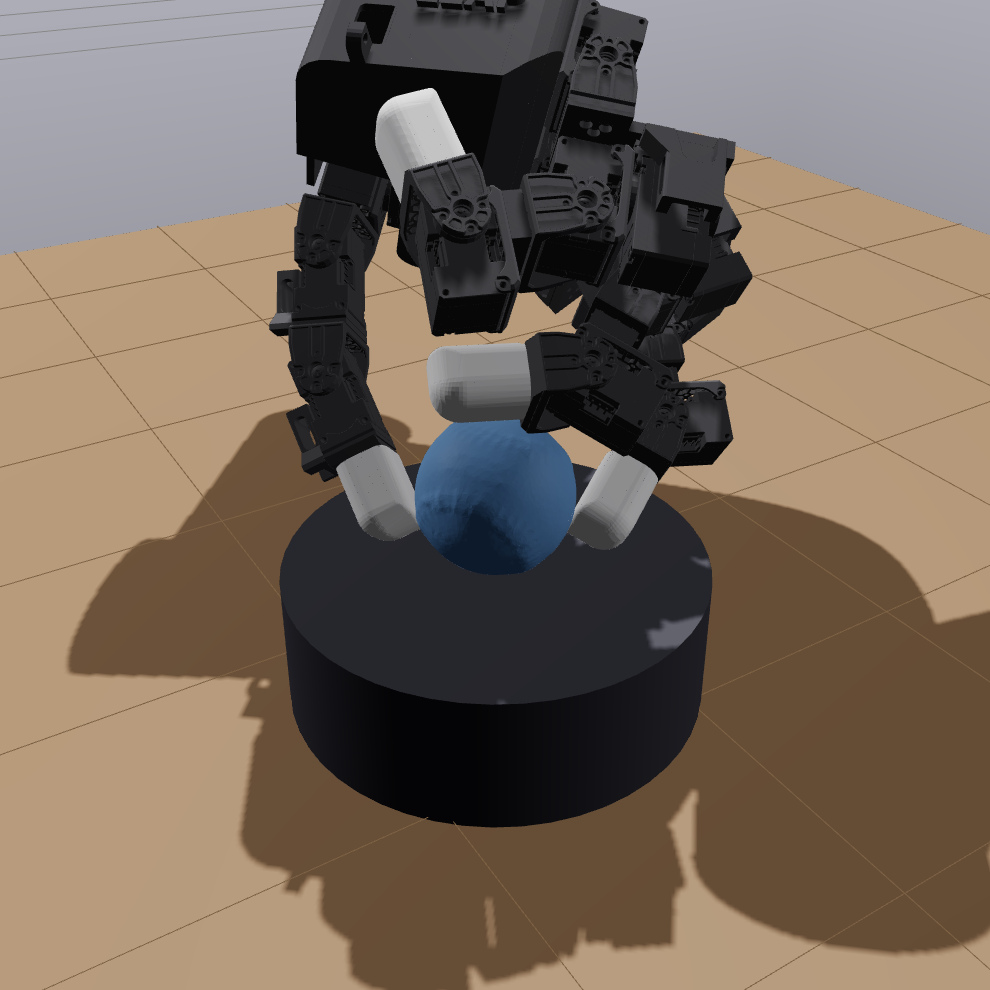}%
	\lcell{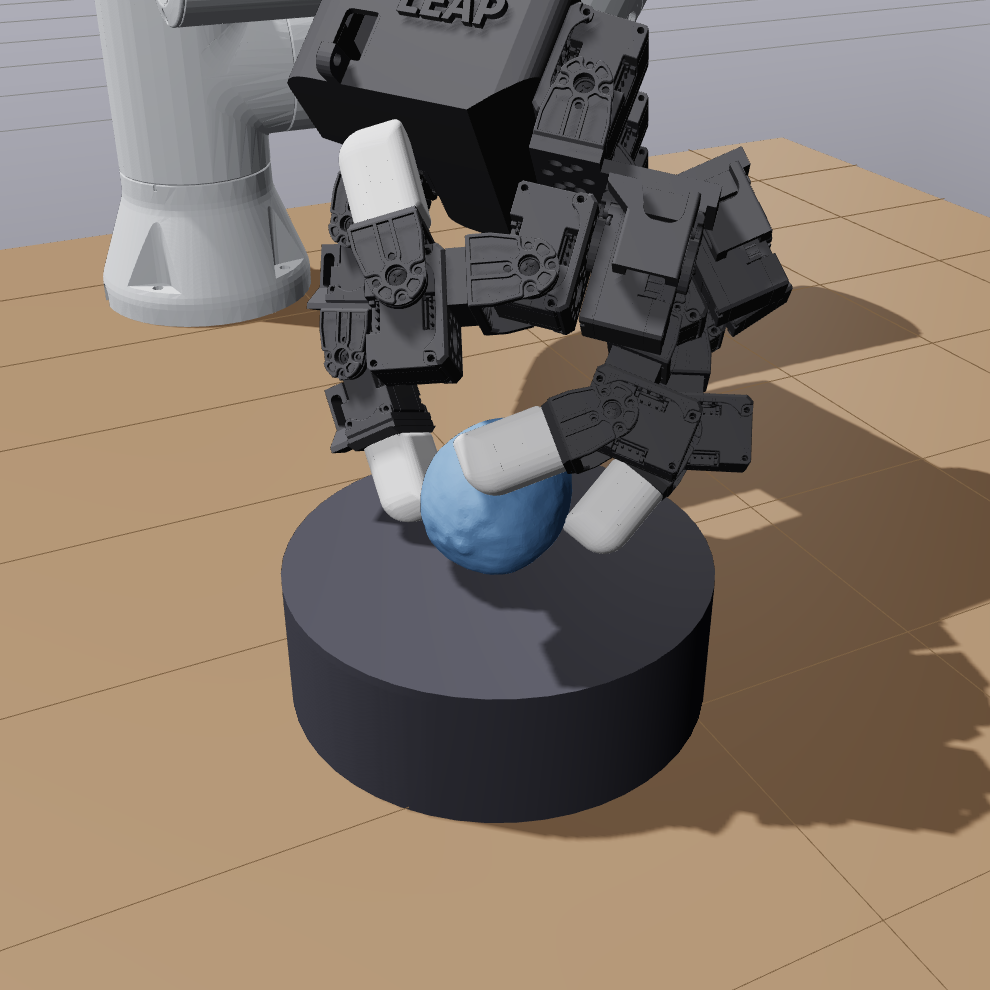}%
	\lcell{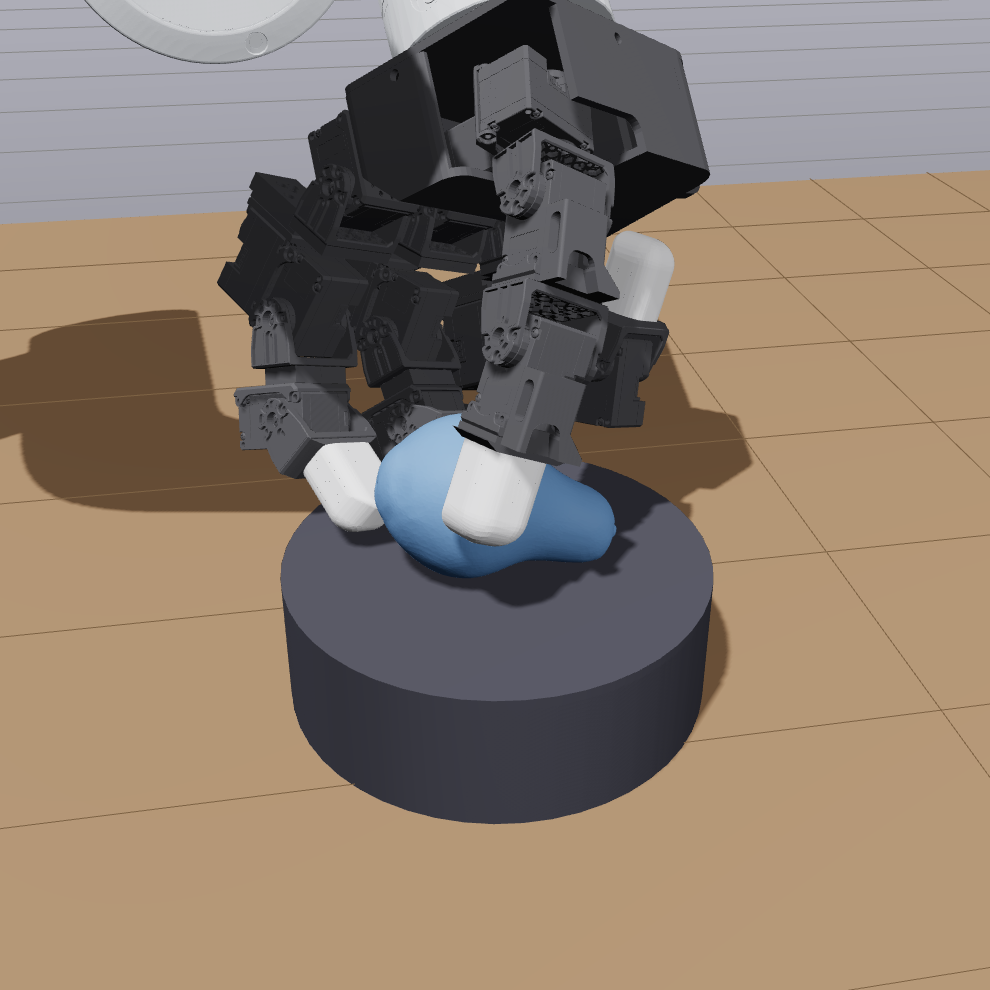}%
	\lcell{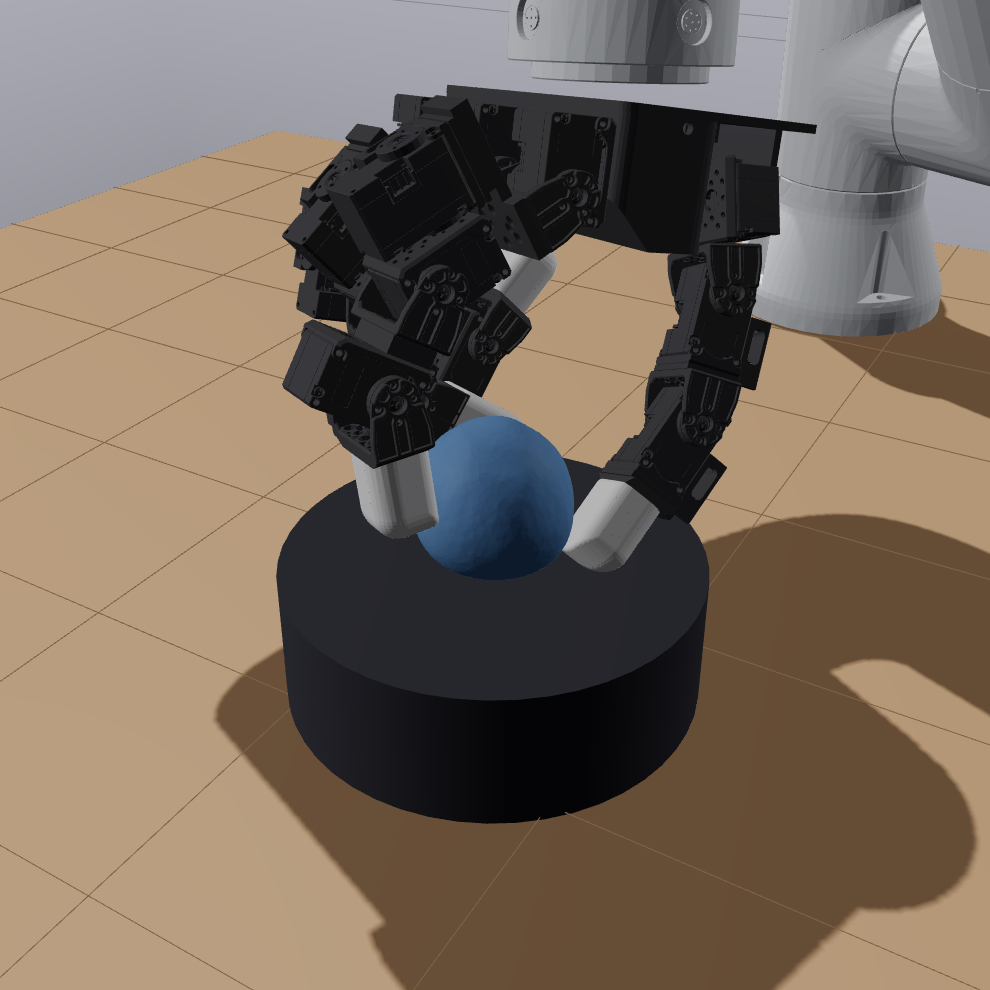}}%
	\caption{\textbf{Friction-Uncertainty-Aware Grasp Synthesis on the LEAP Hand.} Twenty-four friction-uncertainty-aware
		grasps synthesized by the retargeted pipeline for the LEAP hand attached to a 6-DoF FAIR Innovation FR3 cobot,
		spanning EGAD! shapes (top row and left half of the middle row), household
		objects, and fruit-scale items (bottom row), rendered in Drake at the stored
		arm configuration. Every rendered grasp is force-closed (\(\ell^* > 0\)),
		satisfies the friction-tail requirement (\(\eps^{(0.99)} > 0\)), and reaches
		its grasp pose without table or pedestal
		collision. At synthesis time, we build a Drake \texttt{MultibodyPlant} with the object set on a \(150\) mm wide by \(40\) mm high cylindrical pedestal to provide additional support-surface clearance and improve arm reachability. We further confirm grasp-and-lift success for the LEAP hand in Drake dynamic simulation, reporting the one-shot lift outcomes in \Cref{apx:drkexec} and the per-object executability study in \Cref{tab:sim_exec}.}\label{fig:leap_gallery}
\end{figure*}

\appendix

\subsection{One-Shot Grasp-and-Lift Success in Dynamic Simulation}\label{apx:drkexec}
We confirm grasp-and-lift success in the Drake physics-validation system with
gravity on, the protocol of \Cref{sssec:pickcompanion}. Each certified grasp
mounts on a vertical prismatic joint at the synthesized palm pose, welds a contact
pad at every certified contact, and closes with a force-controlled feedforward
grip along the inward contact normals. Once the grip holds contact, the mount
raises the object \(12\,\text{cm}\), after which a lateral pull grows through
\(1\), \(3\), and \(6\) times the object weight while Coulomb friction alone
retains the grasp, over the adverse friction sweep \(\mu \in \{0.2, 0.3, 0.4\}\).
A run counts as a success when the object stays within a \(3\,\text{cm}\) slip and
a \(30^\circ\) rotation relative to the hand with at least two contacts through the
raise. Of the \(105\) lifts attempted across the seven objects, \(51\) establish
contact and enter the success grading, and the risk margin \(\epsb\) orders
their realized success rate above the nominal Ferrari-Canny margin
(\Cref{fig:pick_companion}). \Cref{tab:drkcert} isolates the same comparison. Among the
\(51\) established lifts, all of which the nominal margin certifies
(\(\eps_\text{nom} > 0\)), the \(27\) that the risk margin also certifies
(\(\epsb > 0\)) achieve a \(0.70\) success rate under the most adverse friction
\(\mu = 0.2\), against \(0.25\) for the \(24\) grasps the nominal margin certifies
but the risk margin rejects (\(\epsb \le 0\)), though the two sets hold identically
at \(\mu = 0.7\). \Cref{tab:drkexec} breaks the same lifts down by object.

\begin{table}[htb]
	\centering
	\caption{Lift Success Rate by Metric Certification in Drake. Success rate of the
		\(51\) established-contact lifts at each friction, split by the margin that
		certifies the grasp. All \(51\) are nominal-certified (\(\eps_\text{nom} > 0\)).
		The certified-robust subset (\(\epsb > 0\)) holds under adverse friction where the
		nominal-only grasps (\(\epsb \le 0\)) fail, though the nominal margin rates the
		two sets identically at \(\mu = 0.7\).}\label{tab:drkcert}
	\small
	\setlength{\tabcolsep}{5pt}
	\begin{tabular}{@{}l c c c c c@{}}
		\toprule
		 & & \multicolumn{4}{c}{Success Rate at Friction \(\mu\)} \\
		\cmidrule(l){3-6}
		Grasp Set & \(n\) & \(0.7\) & \(0.4\) & \(0.3\) & \(0.2\) \\
		\midrule
		
		Certified-robust \((\epsb > 0)\)  & \cellcolor{oursmint}27 & \cellcolor{oursmint}1.00 & \cellcolor{oursmint}0.89 & \cellcolor{oursmint}0.81 & \cellcolor{oursmint}0.70 \\
		Nominal only \((\epsb \le 0)\)  & \cellcolor{losssalmon}24 & \cellcolor{losssalmon}1.00 & \cellcolor{losssalmon}0.50 & \cellcolor{losssalmon}0.42 & \cellcolor{losssalmon}0.25 \\
		\bottomrule
	\end{tabular}
\end{table}

\begin{table}[htb]
	\centering
	\caption{Per-Object Grasp-and-Lift Success in Drake. For each object, the number
		of certified grasps evaluated, the number that establish contact and complete
		the gravity-on lift, and the success rate of those lifts under the graded
		lateral pull at each friction. Green marks success at or above \(0.80\) and
		salmon at or below \(0.40\), and the bold row totals the sweep.}\label{tab:drkexec}
	\small
	\setlength{\tabcolsep}{4.5pt}
	\begin{tabular}{@{}l c c c c c c@{}}
		\toprule
		 & & & \multicolumn{4}{c}{Success Rate at Friction \(\mu\)} \\
		\cmidrule(l){4-7}
		Object & Grasps & Lifted & \(0.7\) & \(0.4\) & \(0.3\) & \(0.2\) \\
		\midrule
		Cube            & 15  & 4  & \cellcolor{oursmint}1.00 & \cellcolor{oursmint}1.00 & \cellcolor{oursmint}1.00 & \cellcolor{oursmint}1.00 \\
		Box             & 15  & 7  & \cellcolor{oursmint}1.00 & 0.71 & 0.57 & 0.57 \\
		Cylinder        & 15  & 10 & \cellcolor{oursmint}1.00 & 0.60 & \cellcolor{losssalmon}0.40 & \cellcolor{losssalmon}0.20 \\
		Master Chef Can & 15  & 7  & \cellcolor{oursmint}1.00 & 0.43 & \cellcolor{losssalmon}0.29 & \cellcolor{losssalmon}0.29 \\
		Mustard Bottle  & 15  & 5  & \cellcolor{oursmint}1.00 & \cellcolor{losssalmon}0.40 & \cellcolor{losssalmon}0.40 & \cellcolor{losssalmon}0.00 \\
		Potted Meat Can & 15  & 4  & \cellcolor{oursmint}1.00 & \cellcolor{oursmint}1.00 & \cellcolor{oursmint}1.00 & 0.50 \\
		Tennis Ball     & 15  & 14 & \cellcolor{oursmint}1.00 & \cellcolor{oursmint}0.86 & \cellcolor{oursmint}0.86 & 0.79 \\
		\midrule
		\textbf{Total}  & \textbf{105} & \textbf{51} & \cellcolor{oursmint}\textbf{1.00} & \textbf{0.71} & \textbf{0.63} & \textbf{0.49} \\
		\bottomrule
	\end{tabular}
\end{table}

\subsection{Extending the Risk-Adjusted Synthesis Pipeline to Other Hand Morphologies}\label{app:leapext}
Our risk-adjusted metrics use only the stored grasp map and contact set, so
retargeting the synthesis pipeline of \Cref{sec:synthesis} to a new hand amounts to
swapping the hand model, its fingertip geometry, and the mounting arm, while the
metric layer stays unchanged. We demonstrate the resulting portability on a LEAP right hand
mounted on a 6-DoF FAIR Innovation FR3 robot arm. Our retargeted generator seeds fingertip-opposition candidates on the object surface, gates every candidate through an arm-reachability check (an inverse-kinematics solve on the arm-hand assembly subject to an arm
manipulability threshold and table and pedestal collision constraints), and
certifies the accepted grasps with the same force-closure and friction-tail
requirements as the RealHand L6 set of \Cref{fig:grasp_gallery}.

In particular, \Cref{fig:leap_gallery} shows twenty-four grasps selected from an 81-object set: thirteen EGAD!~\cite{morrisonEGADEvolvedGrasping2020b} shapes and eleven household and fruit-scale objects. We scale meshes that exceed the hand's working grasp aperture to a median extent of 55\,mm. Without the arm-reachability stage, the hand-retargeted synthesis procedure generates all 37{,}440 requested LEAP grasps for 234 objects and 160 seeds per object. We use this larger set in the extended evaluation noted in \Cref{ssec:limits}.

\subsection{Acknowledgments}
The cross-dataset benchmark in Section~\ref{ssec:crosscorpus} uses 1{,}599
grasps for the LEAP and Allegro hands that we synthesized with an
extension of the FRoGGeR library~\cite{li_frogger_2023}. We thank the
FRoGGeR authors for releasing the underlying nonlinear-programming
synthesis tooling as open-source software, on which our extension builds. The
400 Shadow-hand grasps come from the DexGraspNet
release~\cite{wang_dexgraspnet_2023}, and we thank the DexGraspNet authors for making their grasp
dataset publicly available.

\end{document}